\documentclass[twoside,11pt]{article}

\usepackage{jmlr2e}

\usepackage{amssymb}
\usepackage{amsmath}
\usepackage{subfig}
\usepackage{framed}
\usepackage{paralist}
\usepackage{enumitem}
\numberwithin{figure}{section}
\numberwithin{table}{section}

\jmlrheading{XX}{2013}{1-61}{2/13}{XX/XX}{Marc Maier and Katerina Marazopoulou and David Jensen}

\ShortHeadings{Independence in Models of Relational Data}{Maier, Marazopoulou, and Jensen}
\firstpageno{1}

\newcommand{\proda}{Case}
\newcommand{\prodb}{Adapter}
\newcommand{\prodc}{Laptop}
\newcommand{\prodd}{Tablet}
\newcommand{\prode}{Smartphone}
\newcommand{\businessa}{Accessories}
\newcommand{\businessb}{Devices}

\newtheorem{mytheorem}{Theorem}[section]
\newtheorem{mylemma}{Lemma}[section]
\newtheorem{mycorollary}{Corollary}[section]

\newenvironment{repeattheorem}[1][Theorem]{\begin{trivlist}
\item[\hskip \labelsep {\bfseries Theorem~#1}] \itshape}{\end{trivlist}}

\newenvironment{repeatlemma}[1][Lemma]{\begin{trivlist}
\item[\hskip \labelsep {\bfseries Lemma~#1}] \itshape}{\end{trivlist}}

\newenvironment{repeatdefinition}[2][Definition]{\begin{trivlist}
\item[\hskip \labelsep {\bfseries Definition~#1 (#2)}]}{\end{trivlist}}

\renewenvironment{proof}[1][Proof]{\begin{trivlist}
\item[\hskip \labelsep {\bfseries #1.}]}{\end{trivlist}}

\newcounter{definition}[section]
\renewcommand{\thedefinition}{\thesection.\arabic{definition}}
\renewenvironment{definition}[1][Definition]{\refstepcounter{definition}\begin{trivlist}
\item[\hskip \labelsep {\bfseries Definition~\thedefinition\ (#1)}]}{\end{trivlist}}

\newcounter{examplectr}[section]
\renewcommand{\theexamplectr}{\thesection.\arabic{examplectr}}
\renewenvironment{example}[1][Example]{\refstepcounter{examplectr}\begin{trivlist}
\item[\hskip \labelsep {\bfseries Example~\theexamplectr\ }]}{\end{trivlist}}

\newcounter{remarkctr}[section]
\renewcommand{\theremarkctr}{\thesection.\arabic{remarkctr}}

\newcommand{\Perp}{\perp \! \! \! \perp}
\newcommand{\Perpn}{\perp \! \! \! \perp \! \! \! \! \! \!/ \ \ }

\setitemize{itemsep=2pt, topsep=4pt, parsep=2pt}
\setenumerate{itemsep=2pt, topsep=4pt, parsep=2pt}

\begin{document}

\title{Reasoning about Independence\\in Probabilistic Models of Relational Data}

\author{\name Marc Maier \email maier@cs.umass.edu \\
		\AND
		\name Katerina Marazopoulou \email kmarazo@cs.umass.edu \\
       \AND
       \name David Jensen \email jensen@cs.umass.edu \\
       \addr School of Computer Science\\
       University of Massachusetts Amherst\\
       Amherst, MA 01003, USA}

\editor{}

\maketitle

\begin{abstract}%   <- trailing '%' for backward compatibility of .sty file
%Bayesian networks leverage conditional independence to compactly encode joint probability distributions.
%Many learning algorithms exploit the constraints implied by observed conditional independencies to learn the structure of Bayesian networks.
%The rules of \textit{d}-separation provide a theoretical and algorithmic framework for deriving conditional independence facts from model structure.
%However, this theory only applies to Bayesian networks.
%Many real-world systems, such as social or economic systems, are characterized by interacting heterogeneous entities and probabilistic dependencies among the variables on those entities.
%Consequently, researchers have developed extensions to Bayesian networks that can represent these relational dependencies.
We extend the theory of \textit{d}-separation to cases in which data instances are not independent and identically distributed.
We show that applying the rules of \textit{d}-separation directly to the structure of probabilistic models of relational data inaccurately infers conditional independence.
We introduce \textit{relational d-separation}, a theory for deriving conditional independence facts from relational models.
We provide a new representation, the \textit{abstract ground graph}, that enables a sound, complete, and computationally efficient method for answering \textit{d}-separation queries about relational models, and we present empirical results that demonstrate effectiveness.
\end{abstract}


\section{Introduction}
\label{sec:intro}

The rules of \textit{d}-separation can algorithmically derive all conditional independence facts that hold in distributions represented by a Bayesian network.
In this paper, we show that \textit{d}-separation may not correctly infer conditional independence when applied directly to the graphical structure of a relational model.
We introduce the notion of \textit{relational d-separation}---a graphical criterion for deriving conditional independence facts from relational models---and define its semantics to be consistent with traditional \textit{d}-separation (i.e., it claims independence only when it is guaranteed to hold for all model instantiations).
We present an alternative, lifted representation---the \textit{abstract ground graph}---that enables an algorithm for deriving conditional independence facts from relational models.
We show that this approach is sound, complete, and computationally efficient, and we provide an empirical demonstration of effectiveness across synthetic causal structures of relational domains.

The main contributions of this work are:
\begin{itemize}
\item A precise formalization of fundamental concepts of relational data and relational models necessary to reason about conditional independence (Section~\ref{sec:rel_concepts})
\item A formal definition of \textit{d}-separation for relational models analogous to \textit{d}-separation for Bayesian networks (Section~\ref{sec:rds})
\item The abstract ground graph---a lifted representation that \textit{abstracts} all possible ground graphs of a given relational model structure, as well as proofs of the soundness and completeness of this abstraction (Section~\ref{sec:aggs})
\item Proofs of soundness and completeness for a method that answers relational \textit{d}-separation queries (Section~\ref{sec:proof-of-correctness})
\end{itemize}

We also provide an empirical comparison of relational \textit{d}-separation to traditional \textit{d}-separation applied directly to relational model structure, showing that, not only would most queries be undefined, but those that can be represented yield an incorrect judgment of conditional independence up to 50\% of the time (Section~\ref{sec:naive_rds}).
Finally, we offer additional empirical results on synthetic data that demonstrate the effectiveness of relational \textit{d}-separation with respect to complexity and consistency (Section~\ref{sec:experiments}).
The remainder of this introductory section first gives a brief overview of Bayesian networks and their generalization to relational models and then describes why \textit{d}-separation is a useful theory.

\subsection{From Bayesian Networks to Relational Models}
\label{sec:intro-prop-rel}

Bayesian networks are a widely used class of graphical models that are capable of compactly representing a joint probability distribution over a set of variables.
The joint distribution can be factored into a product of conditional distributions by assuming that variables are independent of their non-descendants given their parents (the Markov condition).
The Markov condition ties the structure of the model to the set of conditional independencies that hold over all probability distributions the model can represent.
Accurate reasoning about such conditional independence facts is the basis for constraint-based algorithms, such as PC and FCI \citep{spirtes-etal-book00}, and hybrid approaches, such as MMHC \citep{tsamardinos-etal-ml06}, that are commonly used to learn the structure of Bayesian networks.
Under a small number of assumptions and with knowledge of the conditional independencies, these algorithms can identify causal structure \citep{pearl-causality00, spirtes-etal-book00}.

Deriving the full set of conditional independencies implied by the Markov condition is complex, requiring manipulation of the joint distribution and various probability axioms.
Fortunately, the exact same set of conditional independencies entailed by the Markov condition are also entailed by \textit{d}-separation, a set of graphical rules that algorithmically derive conditional independence facts directly from the graphical structure of the model.
That is, the Markov condition and \textit{d}-separation are equivalent approaches for producing conditional independence from Bayesian networks \citep{verma1988cnse, geiger1988lcm, neapolitan2004learning}.
When interpreting a Bayesian network causally, the causal Markov condition (variables are independent of their non-effects given their direct causes) and \textit{d}-separation have been shown to provide the correct connection between causal structure and conditional independence \citep{scheines-intro97}.

Bayesian networks assume that data instances are independent and identically distributed, but many real-world systems are characterized by interacting heterogeneous entities.
For example, social network data consist of individuals, groups, and their relationships; citation data involve researchers collaborating and authoring scholarly papers that cite prior work; and sports data include players, coaches, teams, referees, and their competitive interactions.
Over the past 15 years, researchers in statistics and computer science have devised more expressive classes of directed graphical models, such as probabilistic relational models (PRMs), which remove the assumptions of independent and identically distributed instances to more accurately describe these types of domains \citep{getoor2007srlbook}.
Relational models generalize other classes of models that incorporate interference, spillover effects, or violations of the stable unit treatment value assumption (SUTVA) \citep{hudgens2008tciwi, tchetgen2012interference} and multilevel or hierarchical models \citep{gelman2007data}.

Many practical applications have also benefited from learning and reasoning with relational models.
Examples include analysis of gene regulatory interactions \citep{segal2001rich}, scholarly citations \citep{taskar2001probabilistic}, ecosystems \citep{dambrosio2003ecosystem}, biological cellular networks \citep{friedman2004inferring}, epidemiology \citep{getoor2004tuberculosis}, and security in information systems \citep{sommestad2010prm-security}.
The structure and parameters of these models can be learned from a relational data set.
The model is typically used either to predict values of certain attributes (e.g., topics of papers) or the structure is examined directly (e.g., to determine predictors of disease spread).
A major goal in many of these applications is to promote understanding of a domain or to determine causes of various outcomes.
However, as with Bayesian networks, to effectively interpret and reason about relational models causally, it is necessary to understand their conditional independence implications.

\subsection{Why \textit{d}-Separation Is Useful}
\label{sec:utility}

A Bayesian network, as a model of a joint probability distribution, enables a wide array of useful tasks by supporting inference over an entire set of variables.
Bayesian networks have been successfully applied to model many domains, ranging from bioinformatics and medicine to computer vision and information retrieval.
Na\"{i}vely specifying a joint distribution by hand requires an exponential number of states; however, Bayesian networks leverage the Markov condition to factor a joint probability distribution into a compact product of conditional probability distributions.

The theory of \textit{d}-separation is an alternative to the Markov condition that provides equivalent implications.
It provides an algorithmic framework for deriving the conditional independencies encoded by the model.
These conditional independence facts are guaranteed to hold in every joint distribution the model represents and, consequently, in any data instance sampled from those distributions.
The semantics of holding across all distributions is the main reason why \textit{d}-separation is useful, enabling two large classes of applications:

(1) \textit{Identification of causal effects}:
The theory of \textit{d}-separation connects the causal structure encoded by a Bayesian network to the set of probability distributions it can represent.
On this basis, many researchers have developed accompanying theory that describes the conditions under which certain causal effects are identifiable (uniquely known) and algorithms for deriving those quantities from the joint distribution.
This work enables sound and complete identification of causal effects, not only with respect to conditioning, but also under counterfactuals and interventions---via the \textit{do}-calculus introduced by \citet{pearl-causality00}---and in the presence of latent variables \citep{tian2002identification, huang2006identifiability, shpitser2008hierarchy}.

(2) \textit{Constraint-based causal discovery algorithms}:
Causal discovery, the task of learning generative models of observational data, superficially appears to be a futile endeavor.
Yet learning and reasoning about the causal structure that underlies real domains is a primary goal for many researchers.
Fortunately, \textit{d}-separation offers a connection between causal structure and conditional independence.
The theory of \textit{d}-separation can be leveraged to constrain the hypothesis space by eliminating models that are inconsistent with observed conditional independence facts.
While many distributions do not lead to uniquely identifiable models, this approach (under simple assumptions) frequently discovers useful causal knowledge for domains that can be represented as a Bayesian network.
This approach to learning causal structure is referred to as the \textit{constraint-based} paradigm, and many algorithms that follow this approach have been developed over the past 20 years, including Inductive Causation (IC) \citep{pearl1991theory}, PC \citep{spirtes-etal-book00} and its variants, Three Phase Dependency Analysis (TPDA) \citep{cheng1997tpda}, Grow-Shrink \citep{margaritis1999growshrink}, Total Conditioning (TC) \citep{pellet2008tc}, Recursive Autonomy Identification (RAI) \citep{yehezkel2009rai}, and hybrid methods that partially employ this approach, including Max-Min Hill Climbing (MMHC) \citep{tsamardinos-etal-ml06} and Hybrid HPC (H2PC) \citep{gasse2012h2pc}.

As described above, relational models more closely represent the real-world domains that many social scientists and other researchers investigate.
To successfully learn causal models from observational data of relational domains, we need a theory for deriving conditional independence from relational models.
In this paper, we formalize the theory of \textit{relational \textit{d}-separation} and provide a method for deriving conditional independence facts from the structure of a relational model.
In another paper, we have used these results to provide a theoretical framework for a sound and complete constraint-based algorithm---the Relational Causal Discovery (RCD) algorithm \citep{maier2013rcd}---that learns causal models of relational domains.

\section{Example}
\label{sec:example}

	\begin{figure}[t]
	\centering
	\subfloat[Example relational schema for an organization consisting of employees working on products, which are funded by specific business units within a corporation.]{
	\includegraphics[width=150mm]{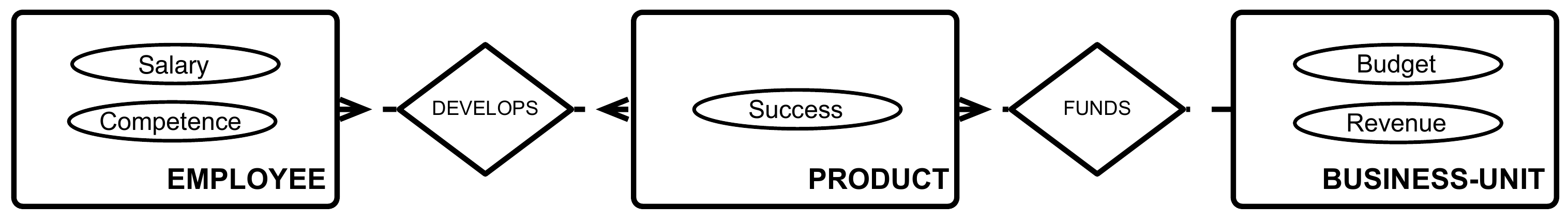}
	\label{fig:org_schema}}\\
	\subfloat[Example fragment of a relational skeleton.  Roger and Sally are employees, both of whom develop the {\prodc} product, but, of the two, only Sally works on product {\prodd}. Both products {\prodc} and {\prodd} are funded by business unit {\businessb}. For convenience, we depict attribute placeholders on each entity instance.]{
	\includegraphics[width=150mm]{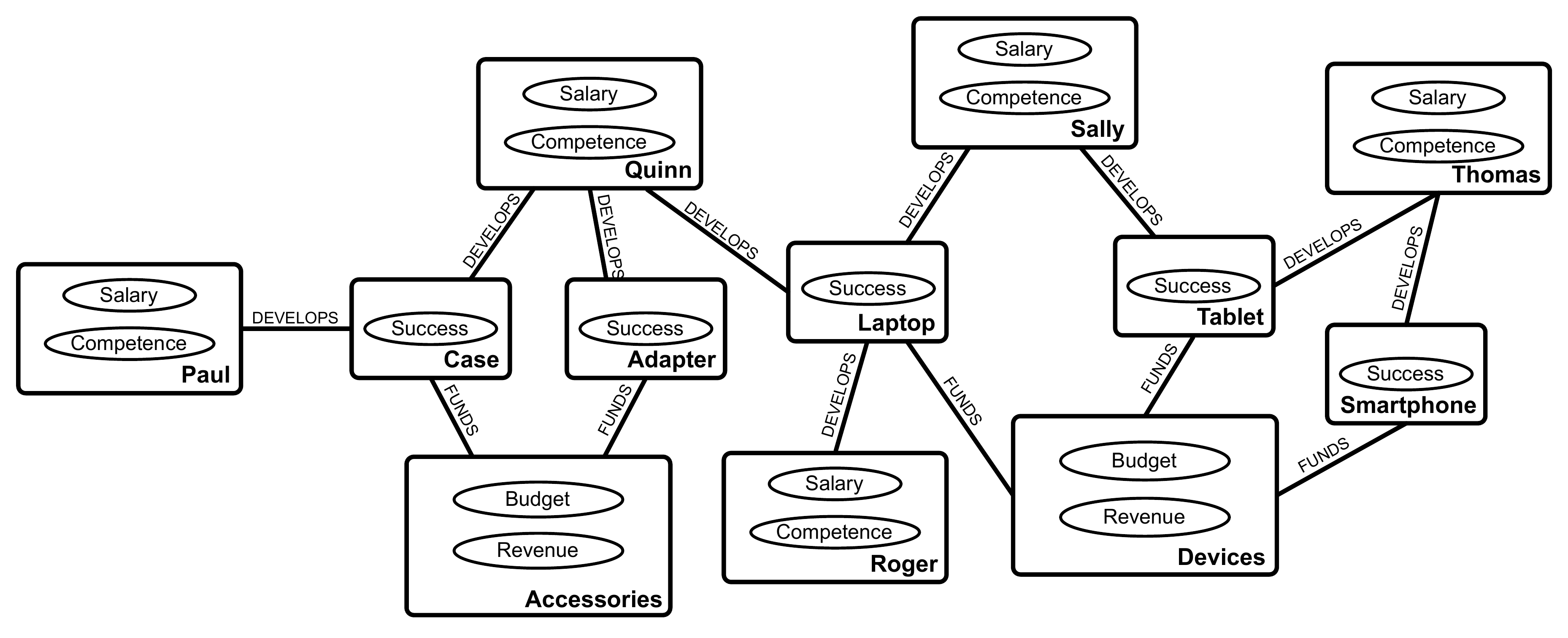}
	\label{fig:org_skeleton}}
	\caption{An example relational schema and skeleton for the organization domain.}
	\label{fig:org_example_data}
	\end{figure}

	\begin{figure}[t]
	\centering
	\subfloat[Example relational model.  Competence of employees cause the success of products they develop, which in turn influences the revenue received by the business unit funding the product.  Additional dependencies involve the budget of business units and employee salaries.  The dependencies are specified by relational paths, listed below the graphical model.]{
	\includegraphics[width=150mm]{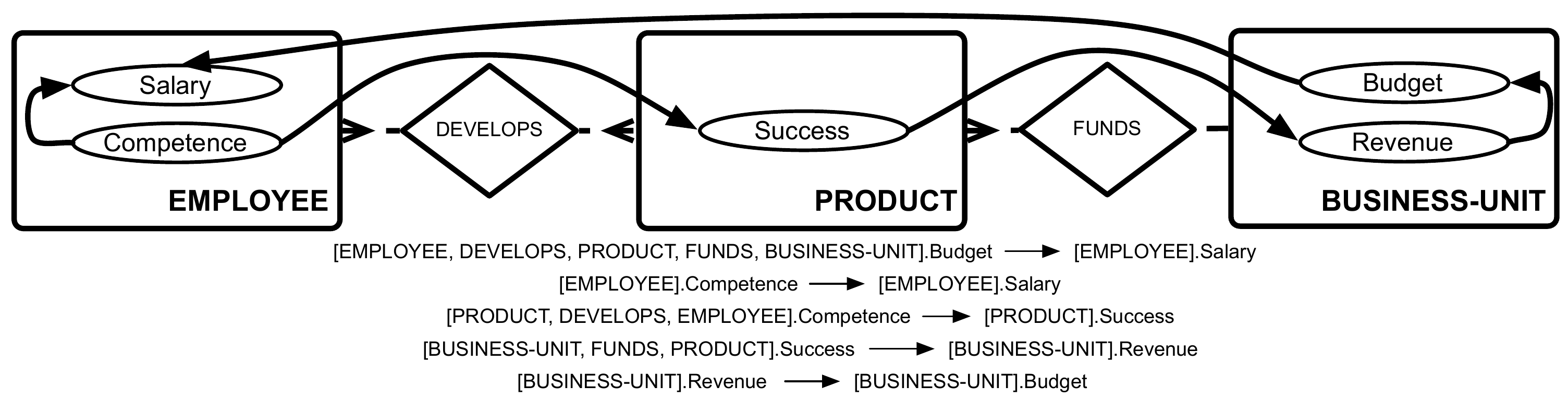}
	\label{fig:org_model}}\\
	\subfloat[Example fragment of a ground graph. The success of product {\prodc} is influenced by the competence of both Roger and Sally. The revenue of business unit {\businessb} is caused by the success of all its funded products---{\prodc}, {\prodd}, and {\prode}.]{
	\includegraphics[width=150mm]{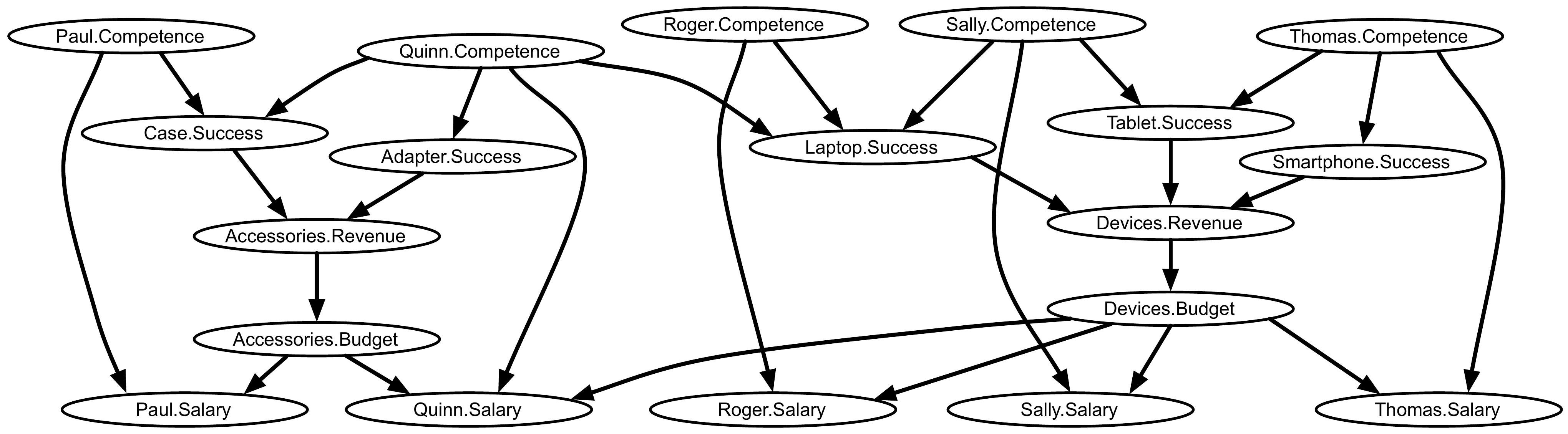}
	\label{fig:org_ground_graph}}
	\caption{An example relational model and ground graph for the organization domain.}
	\label{fig:org_example_model}
	\end{figure}

Consider a corporate analyst who was hired to identify which employees are effective and productive for some organization.
If the company is structured as a pure project-based organization (for which company personnel are structured around projects, not departments), the analyst may collect data as described by the relational schema in Figure~\ref{fig:org_example_data}\subref{fig:org_schema}.
The schema denotes that employees can collaborate and work on multiple products, each of which is funded by a specific business unit.
The analyst has also obtained attributes on each entity---salary and competence of employees, the success of each product, and the budget and revenue of business units.
In this example, the organization consists of five employees, five products, and two business units, which are shown in the relational skeleton in Figure~\ref{fig:org_example_data}\subref{fig:org_skeleton}.

Assume that the organization operates under the model depicted in Figure~\ref{fig:org_example_model}\subref{fig:org_model}.
For example, the success of a product depends on the competence of employees that develop it, and the revenue of a business unit is influenced by the success of products that it funds.
If this were known by the analyst (who happens to have experience in graphical models), then it would be conceivable to spot-check the model and test whether some of the conditional independencies encoded by the model are reflected in the data.
The analyst then na\"{i}vely applies \textit{d}-separation to the model structure in an attempt to derive conditional independencies to test.
However, applying \textit{d}-separation directly to the structure of relational models may not correctly derive conditional independencies, violating the Markov condition.
If the analyst were to discover significant and substantive effects, he may believe the model structure is incorrect and needlessly search for alternative dependencies.

%The analyst may believe that the organization operates under the model depicted in Figure~\ref{fig:org_example_model}\subref{fig:org_model}.
%For example, the success of a product depends on the competence of employees that develop it, and the revenue of a business unit is influenced by the success of products that it funds.
%The analyst then wishes to verify the model structure in order to accurately advise executive decisions, such as determining which business units should have increased funding, which employees are being fairly or unfairly compensated, etc.
%Perhaps the analyst has experience in graphical models and decides to check that the conditional independencies encoded by the model are reflected in the data.
%The analyst then na\"{i}vely applies \textit{d}-separation to the model structure in an attempt to derive these conditional independencies.
%However, applying \textit{d}-separation directly to relational models may not correctly derive the set of conditional independencies, violating the Markov condition.

Na\"{i}vely applying \textit{d}-separation to the model in Figure~\ref{fig:org_example_model}\subref{fig:org_model} suggests that employee competence is conditionally independent of the revenue of business units given the success of products:
\vspace{-2mm}
\begin{center}
\textsc{Employee}.$Competence \Perp$ \textsc{Business-Unit}$.Revenue\ \vert$ \textsc{Product}.Success
\end{center}
\vspace{-1mm}
\noindent To see why this approach is flawed, we must consider the \textit{ground graph}.
A necessary precondition for inference is to apply a model to a data instantiation, which yields a ground graph to which \textit{d}-separation can be applied.
For a Bayesian network, a ground graph consists of replicates of the model structure for each data instance.
In contrast, a relational model defines a template for how dependencies apply to a data instantiation, resulting in a ground graph with varying structure. See Section~\ref{sec:rel_concepts} for more details on ground graphs.

Figure~\ref{fig:org_example_model}\subref{fig:org_ground_graph} shows the ground graph for the relational model in Figure~\ref{fig:org_example_model}\subref{fig:org_model} applied to the relational skeleton in Figure~\ref{fig:org_example_data}\subref{fig:org_skeleton}.
This ground graph illustrates that, for a single employee, simply conditioning on the success of developed products can activate a path through the competence of other employees who develop the same products---we call this a \textit{relationally d-connecting path}.\footnote{The indirect effect attributed to a relationally \textit{d}-connecting path is often referred to as interference, a spillover effect, or a violation of the stable unit treatment value assumption (SUTVA) because the treatment of one instance (employee competence) affects the outcome of another (the revenue of another employee's business unit).}
Checking \textit{d}-separation on the ground graph indicates that to \textit{d}-separate an employee's competence from the revenue of funding business units, we should not only condition on the success of developed products, but also on the competence of other employees who work on those products (e.g., Roger.$Competence \Perp$ {\businessb}$.Revenue\ \vert\ \{${\prodc}$.Success$, Sally.$Competence\}$).

This example also demonstrates that the Markov condition can be violated when directly applied to the structure of a relational model.
In this case, the Markov condition according to the model structure in Figure~\ref{fig:org_example_model}\subref{fig:org_model} implies that $P(Competence, Revenue\ \vert\ Success) = P(Competence\ \vert\ Success)P(Revenue\ \vert\ Success)$, that revenue is independent of its non-descendants (competence) given its parents (success).
However, the ground graph shows the opposite, for example, $P($Roger$.Competence,$ {\businessb}$.Revenue\ \vert\ ${\prodc}$.Success)$ $\neq$\\ $P($Roger$.Competence\ \vert${\prodc}$.Success)$ $P(${\businessb}$.Revenue\ \vert$ {\prodc}$.Success)$.
%This is not surprising since the conditional independencies entailed by \textit{d}-separation are known to be equivalent to those entailed by the Markov condition \citep{neapolitan2004learning}.
In fact, for this model, \textit{d}-separation produces many other incorrect judgments of conditional independence.
Through simulation, we found that only 25\% of the pairs of variables can even be described by direct inspection of this model structure, and of those (such as the above example), 75\% yield an incorrect conclusion.
This is a single data point of a larger empirical evaluation presented in Section~\ref{sec:naive_rds}.
Those results provide quantitative details of how often to expect traditional \textit{d}-separation to fail when applied to the structure of relational models.

\section{Semantics and Alternatives}
\label{sec:semantics-alternatives}

The example in Section~\ref{sec:example} provides a useful basis to describe the semantics imposed by relational \textit{d}-separation and the characteristics of our approach.  
There are two primary concepts:

(1) \textit{All-ground-graphs semantics}:
It might appear that, since the standard rules of \textit{d}-separation apply to Bayesian networks and the ground graphs of relational models are also Bayesian networks, that applying \textit{d}-separation to relational models is a non-issue.
However, applying \textit{d}-separation to a single ground graph may result in potentially unbounded runtime if the instantiation is large (i.e., since relational databases can be arbitrarily large).
Further, and more importantly, the semantics of \textit{d}-separation require that conditional independencies hold across \textit{all possible} model instantiations.
Although \textit{d}-separation can apply directly to a ground graph, these semantics prohibit reasoning about a single ground graph.

The conditional independence facts derived from \textit{d}-separation hold for all distributions represented by a Bayesian network.
Analogously, the implications of \textit{relational} \textit{d}-separation should hold for all distributions represented by a relational model.
It is simple to show that these implications hold for all ground graphs of a Bayesian network---every ground graph consists of a set of disconnected subgraphs, each of which has a structure that is identical to that of the model.
However, the set of distributions represented by a relational model depends on both the relational skeleton (constrained by the schema) and the model parameters.
That is, the ground graphs of relational models vary with the structure of the underlying relational skeleton (e.g., different products are developed by varying numbers of employees).
As a result, answering relational \textit{d}-separation queries requires reasoning without respect to ground graphs.

(2) \textit{Perspective-based analysis}:
Relational models make explicit one implicit choice underlying nearly any form of data analysis.
This choice---what we refer to here as a \textit{perspective}---concerns the selection of a particular unit or subject of analysis.
%In machine learning and statistics, problems of classification, regression, and, more generally, modeling conditional distributions, require choosing a single outcome or dependent variable.
For example, in the social sciences, a commonly used acronym is \textit{UTOS}, for framing an analysis by choosing a unit, treatment, outcome, and setting.
Any method, such as Bayesian network modeling, that assumes IID data makes the implicit assumption that the attributes on data instances correspond to attributes of a single unit or perspective.
In the example, we targeted a specific conditional independence regarding employee instances (as opposed to products or business units).

The concept of perspectives is not new, but it is central to statistical relational learning because relational data sets may be heterogeneous, involving instances that refer to multiple, distinct perspectives.
The inductive logic programming (ILP) community has discussed individual-centered representations \citep{flach1999knowledge}, and many approaches to propositionalizing relational data have been developed to enforce a \textit{single} perspective in order to rely on existing propositional learning algorithms \citep{kramer-etal-rdmchapter01}.
An alternative strategy is to explicitly acknowledge the presence of multiple perspectives and learn jointly among them.
This approach underlies many algorithms that learn the types of probabilistic models of relational data applicable in this work, e.g., learning the structure of probabilistic relational models, relational dependency networks, or parametrized Bayesian networks \citep{friedman1999learning, neville-jensen-jmlr07, schulte2012recursive}.

Often, data sets are derivative, leading to little or no choice about which perspectives to analyze.
However, for relational domains, from which these data sets are derived, it is assumed that there are multiple perspectives, and we can dynamically analyze different perspectives.
In the example, we chose the employee perspective, and the analysis focused on the dependence between an employee's competence and the revenue of business units that fund developed products.
However, if the question were posed from the perspective of business units, then we could conceivably condition on the success of products for each business unit.
In this scenario, reasoning about \textit{d}-separation at the model level \textit{would} lead to a correct conditional independence statement.
Some (though fairly infrequent) \textit{d}-separation queries produce accurate conditional independence facts when applied to relational model structure (see Section~\ref{sec:naive_rds}).
However, the model is often unknown, a perspective may be chosen a priori, and a theory that is occasionally correct is clearly undesirable.
Additionally, to support constraint-based learning algorithms, it is important to reason about conditional independence implications from different perspectives.

One plausible alternative approach would be to answer \textit{d}-separation queries by ignoring perspectives and considering just the attribute classes (i.e., reason about \textit{Competence} and \textit{Revenue} given \textit{Success}).
However, it remains to define explicit semantics for grounding and evaluating the query based on the relational skeleton.
There are at least three options:
\begin{itemize}
\item \textit{Construct three sets of variables, including all instances of competence, revenue, and success variables}:
Although the ground graph has the semantics of a Bayesian network, there is only a \textit{single} ground graph---one data sample \citep{xiang2011onenetwork}.  Consequently, this analysis would be statistically meaningless and is the primary reason why relational learning algorithms dynamically \textit{generate} propositional data for each instance of a given perspective.
\item \textit{Test the Cartesian product of competence and revenue variables, conditioned on all success variables}:	Testing all pairs invariably leads to independence.  Moreover, these semantics are incoherent; only reachable pairs of variables should be compared.  For propositional data, variable pairs are constructed by choosing attribute values, e.g., height and weight, \textit{within} an individual.  The same is true for relational data: Only choose the success of products for employees that actually develop them, following the underlying relational connections.
\item \textit{Test relationally connected pairs of competence and revenue variables, conditioned on all success variables}: Again, this appears plausible based on traditional \textit{d}-separation; every instance in the table conditions on the \textit{same} set of success values.  Therefore, this is akin to not conditioning because the conditioning variable is a constant.
\end{itemize}
We argue that the desired semantics are essentially the explicit semantics of perspective-based queries.
Therefore, we advocate perspective-based analysis as the \textit{only} statistically and semantically meaningful approach for relational data and models.
 
\begin{figure*}[t]
\centering
\includegraphics[width=150mm]{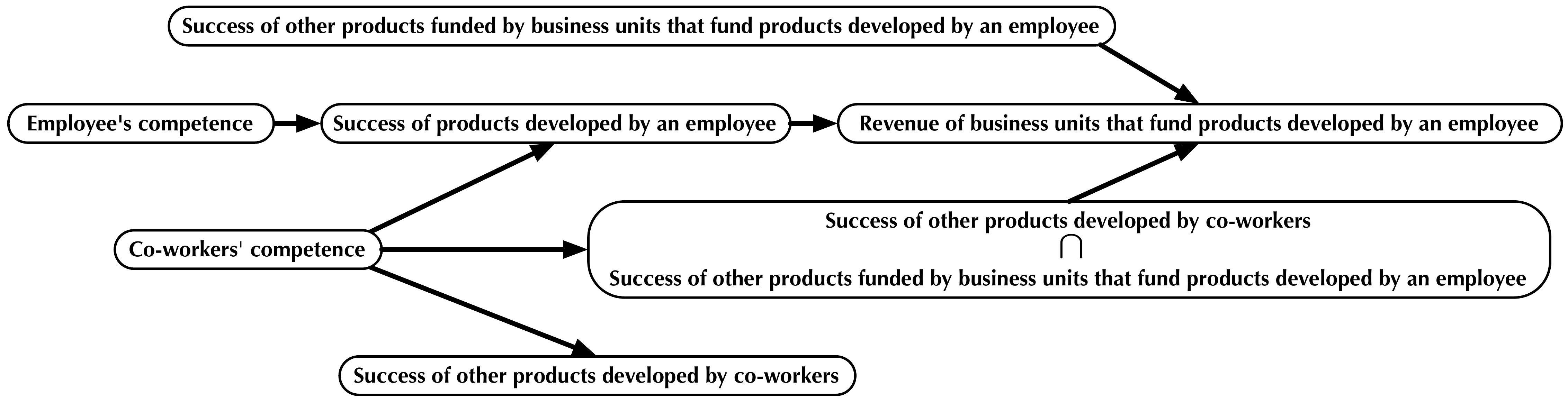}
\caption{Example abstract ground graph from the perspective of employees. Nodes are labeled with their intuitive meaning.}
\label{fig:org_agg_sentences}
\end{figure*}

Our approach to answering relational \textit{d}-separation queries incorporates the two aforementioned semantics.
In Section~\ref{sec:rds}, we describe a new, lifted representation---the abstract ground graph---that is provably sound and complete in its abstraction of all ground graphs for a given relational model.
As their name suggests, abstract ground graphs \textit{abstract} all ground graphs of a relational model, representing any potential relationally \textit{d}-connecting path (recall the example \textit{d}-connecting path that only manifests in the ground graph).
A relational model has a corresponding \textit{set} of abstract ground graphs, one for each perspective (i.e., entity or relationship class in its underlying schema), and can be used to reason about relational \textit{d}-separation with respect to any given perspective.
Figure~\ref{fig:org_agg_sentences} shows a fragment of an abstract ground graph from the employee perspective for the model in Figure~\ref{fig:org_model}.
The nodes are depicted with their intuitive meaning rather than their actual syntax for this example.
Representational details and accompanying theory are presented in Section~\ref{sec:rds}.

\section{Concepts of Relational Data and Models}
\label{sec:rel_concepts}

Propositional representations describe domains with a single entity type, but many real-world systems involve multiple types of interacting entities with probabilistic dependencies among their variables.
For example, in the model in Figure~\ref{fig:org_example_model}\subref{fig:org_model} the competence of employees affects the success of products they develop.
Many researchers have focused on modeling such domains, which are generally characterized as \textit{relational}.
These relational representations can be divided into two main categories: \textit{probabilistic graphical models}---such as probablistic relational models (PRMs) \citep{koller-pfeffer-aaai98}, directed acyclic probabilistic entity-relationship (DAPER) models \citep{heckerman-etal-tr04}, and relational Markov networks (RMNs) \citep{taskar-etal-uai02}---and \textit{probabilistic logic models}---such as Bayesian logic programs (BLPs) \citep{kersting-deraedt-tr02}, Markov logic networks (MLNs) \citep{richardson-domingos-ml06}, parametrized Bayesian networks (PBNs) \citep{poole2003pbn}, Bayesian logic (\textsc{Blog}) \citep{milch-etal-ijcai05}, multi-entity Bayesian networks (MEBNs) \citep{laskey2008mebn}, and relational probability models (RPMs) \citep{russell-norvig-aima}.

To facilitate an extension to the graphical criterion of \textit{d}-separation, we currently focus on directed, acyclic, graphical models of conditional independence.
As most of the above models have similar expressive power, the results in this paper could generalize across representations---even for \textit{undirected} relational models, such as RMNs and MLNs, after moralization.
However, we found it simpler to define and prove relevant theoretical properties for relational \textit{d}-separation in a representation most similar to Bayesian networks.
In this section, we formally define the concepts of relational data and models using a similar representation to PRMs and DAPER models.

A relational schema is a top-level description of what data exist in a particular domain.
Specifically (adapted from \citealp{heckerman-etal-introsrl07}):	

\begin{definition}[Relational schema]
\label{def:rel-schema}
A \textit{relational schema} $\mathcal{S} = (\mathcal{E}, \mathcal{R}, \mathcal{A}, \mathit{card})$ consists of a set of entity classes $\mathcal{E} = \{E_1,\dots,E_m\}$; a set of relationship classes $\mathcal{R} = \{R_1,\dots,R_n\}$, where each $R_i = \langle E^i_1,\dots, E^i_{a_i}\rangle$, with $E^i_j\in\mathcal{E}$ and $a_i$ is the arity for $R_i$; a set of attribute classes $\mathcal{A}(I)$ for each item class $I\in\mathcal{E}\cup\mathcal{R}$; and a cardinality function $\mathit{card}: \mathcal{R}\times\mathcal{E}\rightarrow \{\text{\textsc{one}, \textsc{many}}\}$.
\end{definition}

A relational schema can be represented graphically with an entity-relationship (ER) diagram.
We adopt a slightly modified ER diagram using Barker's notation \citeyearpar{barker-case90}, where entity classes are rectangular boxes, relationship classes are diamonds with dashed lines connecting their associated entity classes, attribute classes are ovals residing on entity and relationship classes, and cardinalities are represented with crow's foot notation.

\begin{example}
\label{ex:rel-schema}
The relational schema $\mathcal{S}$ for the organization domain example depicted in Figure~\ref{fig:org_example_data}\subref{fig:org_schema} consists of entities $\mathcal{E} = \{$\textsc{Employee}, \textsc{Product}, \textsc{Business-Unit}$\}$; relationships $\mathcal{R} = \{$\textsc{Develops}, \textsc{Funds}$\}$, where \textsc{Develops} = $\langle$\textsc{Employee}, \textsc{Product}$\rangle$, \textsc{Funds} = $\langle$\textsc{Business-Unit}, \textsc{Product}$\rangle$ and having cardinalities $\mathit{card}($\textsc{Develops}, \textsc{Employee}$) = $ \textsc{many}, $\mathit{card}($\textsc{Develops}, \textsc{Product}$) = $ \textsc{many}, $\mathit{card}($\textsc{Funds}, \textsc{Business-Unit}$) = $ \textsc{many}, and $\mathit{card}($\textsc{Funds}, \textsc{Product}$) = $ \textsc{one}; and attributes $\mathcal{A}($\textsc{Employee}$) = \{$\textit{Competence}, \textit{Salary}$\}$, $\mathcal{A}($\textsc{Product}$) = \{$\textit{Success}$\}$, and $\mathcal{A}($\textsc{Business-Unit}$) = \{$\textit{Budget}, \textit{Revenue}$\}$.
$\square$
\end{example}

A relational schema is a template for a relational skeleton (also referred to as a data graph by \citealp{neville-jensen-jmlr07}), an instantiation of entity and relationship classes.
Specifically (adapted from \citealp{heckerman-etal-introsrl07}):

\begin{definition}[Relational skeleton]
\label{def:skeleton}
A \textit{relational skeleton} $\sigma$ for relational schema $\mathcal{S} = (\mathcal{E}, \mathcal{R}, \mathcal{A}, \mathit{card})$ specifies a set of entity instances $\sigma(E)$ for each $E\in\mathcal{E}$ and relationship instances $\sigma(R)$ for each $R\in\mathcal{R}$.  
Relationship instances adhere to the cardinality constraints of $\mathcal{S}$: If $\mathit{card}(R, E) = $ \textsc{one}, then for each $e\in\sigma(E)$ there is at most one $r\in\sigma(R)$ such that $e$ participates in $r$.
\end{definition}

For convenience, we use the notation $E\in R$ if entity class $E$ is a component of relationship class $R$, and, similarly, $e\in r$ if entity instance $e$ is a component of the relationship instance $r$.
We also denote the set of all skeletons for schema $\mathcal{S}$ as $\Sigma_\mathcal{S}$.

\begin{example}
\label{ex:rel-skeleton}
The relational skeleton $\sigma$ for the organization example is depicted in Figure~\ref{fig:org_example_data}\subref{fig:org_skeleton}.  The sets of entity instances are $\sigma($\textsc{Employee}$)$ = $\{$Paul, Quinn, Roger, Sally, Thomas$\}$, $\sigma($\textsc{Product}$)$ = $\{${\proda}, {\prodb}, {\prodc}, {\prodd}, {\prode}$\}$, and $\sigma($\textsc{Business-Unit}$)$ = $\{${\businessa}, {\businessb}$\}$.  The sets of relationship instances are $\sigma($\textsc{Develops}$)$ = $\{\langle$Paul, {\proda}$\rangle$, $\langle$Quinn, {\proda}$\rangle$, $\dots$, $\langle$Thomas, {\prode}$\rangle\}$ and $\sigma($\textsc{Funds}$)$ = $\{\langle${\businessa}, {\proda}$\rangle$, $\langle${\businessa}, {\prodb}$\rangle$, $\dots$, $\langle${\businessb}, {\prode}$\rangle\}$.  The relationship instances adhere to their cardinality constraints (e.g., \textsc{Funds} is a \textsc{one}-to-\textsc{many} relationship---within $\sigma($\textsc{Funds}$)$, every product has a single business unit, and every business unit may have multiple products).
$\square$
\end{example}

In order to specify a model over a relational domain, we must define a space of possible variables and dependencies.
Consider the example dependency $[$\textsc{Product}, \textsc{Develops}, \textsc{Employee}$].Competence \rightarrow [$\textsc{Product}$].Success$ from the model in Figure~\ref{fig:org_example_model}\subref{fig:org_model}, expressing that the competence of employees developing a product affects the success of that product.
For relational data, the variable space includes not only intrinsic entity and relationship attributes (e.g., success of a product), but also the attributes on other entity and relationship classes that are reachable by paths along the relational schema (e.g., the competence of employees that develop a product).
We define relational paths to formalize the notion of which item classes are reachable on the schema from a given item class.\footnote{Because the term ``path" is also commonly used to describe chains of dependencies in graphical models, we will explicitly qualify each reference to avoid ambiguity.}

\begin{definition}[Relational path]
\label{def:rel-path}
A \textit{relational path} $[I_j,\dots, I_k]$ for relational schema $\mathcal{S}$ is an alternating sequence of entity and relationship classes $I_j,\dots, I_k\in\mathcal{E}\cup\mathcal{R}$ such that:
\begin{compactenum}
\item[(1)] For every pair of consecutive item classes $[E, R]$ or $[R, E]$ in the path, $E\in R$.
\item[(2)] For every triple of consecutive item classes $[E, R, E']$, $E\neq E'$.\footnotemark
\item[(3)] For every triple of consecutive item classes $[R, E, R']$, if $R = R'$, then $\mathit{card}(R, E)=$ \textsc{many}.
\end{compactenum}
$I_j$ is called the \textit{base item}, or \textit{perspective}, of the relational path.
\end{definition}

Condition (1) enforces that entity classes participate in adjacent relationship classes in the path.
Conditions (2) and (3) remove any paths that would invariably reach an empty terminal set (see Definition~\ref{def:terminal-set} and Appendix~\ref{sec:appendix_valid_relational_paths}).
This definition of relational paths is similar to ``meta-paths" and ``relevance paths" in similarity search and information retrieval in heterogeneous networks \citep{sun2011pathsim, shi2012relevance}.
Relational paths also extend the notion of ``slot chains" from the PRM framework \citep{getoor:prm-ch-srl-book07} by including cardinality constraints and formally describing the semantics under which repeated item classes may appear on a path.
Relational paths are also a specialization of the first-order constraints on arc classes imposed on DAPER models \citep{heckerman-etal-introsrl07}.

\footnotetext{This condition suggests at first glance that self-relationships (e.g., employees manage other employees, individuals in social networks maintain friendships, scholarly articles cite other articles) are prohibited.  We discuss this and other model assumptions in Section~\ref{sec:rel_model_assumptions}.}

\begin{example}
\label{ex:rel-path}
Consider the example relational schema in Figure~\ref{fig:org_example_data}\subref{fig:org_schema}.
Some example relational paths from the \textsc{Employee} perspective (with an intuitive meaning of what the paths describe) include the following: $[$\textsc{Employee}$]$ (an employee), $[$\textsc{Employee}, \textsc{Develops}, \textsc{Product}$]$ (products developed by an employee), $[$\textsc{Employee}, \textsc{Develops}, \textsc{Product}, \textsc{Funds}, \textsc{Business-Unit}$]$ (business units of the products developed by an employee), and $[$\textsc{Employee}, \textsc{Develops}, \textsc{Product}, \textsc{Develops}, \textsc{Employee}$]$ (co-workers developing the same products).
Invalid relational paths include $[$\textsc{Employee}, \textsc{Develops}, \textsc{Employee}$]$ (because \textsc{Employee}=\textsc{Employee} and \textsc{Develops} $\in\mathcal{R}$) and $[$\textsc{Business-Unit}, \textsc{Funds}, \textsc{Product}, \textsc{Funds}, \textsc{Business-Unit}$]$ (because \textsc{Product} $\in\mathcal{E}$ and $\mathit{card}($\textsc{Funds}, \textsc{Product}$)=$ \textsc{one}).
$\square$
\end{example}

Relational paths are defined at the level of relational schemas, and as such are templates for paths in a relational skeleton.
An instantiated relational path produces a set of traversals on a relational skeleton.
However, the quantity of interest is not the traversals, but the set of reachable item instances (i.e., entity or relationship instances).
These reachable instances are the fundamental elements that support model instantiations (i.e., ground graphs).

\begin{definition}[Terminal set]
\label{def:terminal-set}
For skeleton $\sigma\in\Sigma_\mathcal{S}$ and $i_j\in \sigma(I_j)$, the \textit{terminal set} $P\vert_{i_j}$ for relational path $P = [I_j,\dots, I_k]$ of length $n$ is defined inductively as
\begin{center}
$P^1\vert_{i_j} = [I_j]\vert_{i_j} = \lbrace i_j \rbrace$

$\vdots$
\end{center}

$P^n\vert_{i_j} = [I_j,\dots, I_k]\vert_{i_j} = \displaystyle\bigcup_{i_{m}\in P^{n-1}\vert_{i_j}} \Big\lbrace  i_k\ \vert\ \big((i_m\in i_k\mbox{ if }I_k\in\mathcal{R})\ \lor\ (i_k\in i_m\mbox{ if }I_k\in\mathcal{E})\big)$
\vspace{-5mm}

\hspace{70mm} $\land\ i_k\notin \displaystyle\bigcup_{l=1}^{n-1} P^{l}\vert_{i_j}\Big\rbrace$
\end{definition}

A terminal set of a relational path $P = [I_j,\dots, I_k]$ consists of instances of class $I_k$, the terminal item on the path.
Conceptually, a terminal set is produced by traversing a skeleton beginning at a single instance of the base item class, $i_j\in\sigma(I_j)$, following instances of the item classes in the relational path, and reaching a set of instances of class $I_k$.
The term $i_k\notin\bigcup_{l=1}^{n-1} P^{l}\vert_{i_j}$ in the definition implies a ``bridge burning" semantics under which no item instances are revisited ($i_k$ does not appear in the terminal set of any prefix of $P$).\footnote{The bridge burning semantics yield terminal sets that are necessarily subsets of terminal sets that would otherwise be produced without bridge burning.  Although this appears to be limiting, it actually enables a strictly more expressive class of relational models.  See Appendix~\ref{sec:appendix_burning_bridges} for more details and an example.}
The notion of terminal sets is a necessary concept for grounding any relational model and has been described in previous work---e.g., for PRMs \citep{getoor:prm-ch-srl-book07} and MLNs \citep{richardson-domingos-ml06}---but has not been explicitly named.
We emphasize their importance because terminal sets are also critical for defining relational \textit{d}-separation, and we formalize the semantics for bridge burning.

\begin{example}
\label{ex:terminal-set}
We can generate terminal sets by pairing the set of relational paths for the schema in Figure~\ref{fig:org_example_data}\subref{fig:org_schema} with the relational skeleton in Figure~\ref{fig:org_example_data}\subref{fig:org_skeleton}.  Let Quinn be our base item instance.  Then $[$\textsc{Employee}$]\vert_{\text{Quinn}} = \{$Quinn$\}$, $[$\textsc{Employee}, \textsc{Develops}, \textsc{Product}$]\vert_{\text{Quinn}}\! =\! \{${\proda}, {\prodb}, {\prodc}$\}$, $[$\textsc{Employee}, \textsc{Develops}, \textsc{Product}, \textsc{Funds}, \textsc{Business-Unit}$]\vert_{\text{Quinn}} = \{${\businessa}, {\businessb}$\}$, and $[$\textsc{Employee}, \textsc{Develops}, \textsc{Product}, \textsc{Develops}, \textsc{Employee}$]\vert_{\text{Quinn}} = \{$Paul, Roger, Sally$\}$.  The bridge burning semantics enforce that Quinn is not also included in this last terminal set.
$\square$
\end{example}

For a given base item class, it is common (depending on the schema) for distinct relational paths to reach the same terminal item class.
The following lemma states that if two relational paths with the same base item and the same terminal item differ at some point in the path, then for some relational skeleton and some base item instance, their terminal sets will have a non-empty intersection.
This property is important to consider for relational \textit{d}-separation.

\begin{mylemma}
\label{lemma:overlap}
For two relational paths of arbitrary length from $I_j$ to $I_k$ that differ in at least one item class, $P_1=[I_j,\dots,I_m,\dots,I_k]$ and $P_2=[I_j,\dots,I_n,\dots,I_k]$ with $I_m\neq I_n$, there exists a skeleton $\sigma\in\Sigma_\mathcal{S}$ such that $P_1\vert_{i_j} \cap P_2\vert_{i_j}\neq \emptyset$ for some $i_j\in\sigma(I_j)$.
\end{mylemma}

\begin{proof}
See Appendix~\ref{sec:appendix_proofs}.
\end{proof}

\begin{example}
\label{ex:lemma-overlap}
Let $P_1$ = $[$\textsc{Employee}, \textsc{Develops}, \textsc{Product}, \textsc{Develops}, \textsc{Employee}, \textsc{Develops}, \textsc{Product}$]$, the terminal sets for which yield other products developed by collaborating employees.  Let $P_2$  = $[$\textsc{Employee}, \textsc{Develops}, \textsc{Product}, \textsc{Funds}, \textsc{Business-Unit}, \textsc{Funds}, \textsc{Product}$]$, the terminal sets for which consist of other products funded by the business units funding products developed by a given employee.  Intersection among terminal sets for these paths occurs even in the small example skeleton.  In fact, the intersection of the terminal sets for $P_1$ and $P_2$ is non-empty for all employees.  For example, Paul: $P_1\vert_{\text{Paul}} = \{${\prodb}, {\prodc}$\}$ and $P_2\vert_{\text{Paul}} = \{${\prodb}$\}$; Quinn: $P_1\vert_{\text{Quinn}} = \{${\prodd}$\}$ and $P_2\vert_{\text{Quinn}} = \{${\prodd}, {\prode}$\}$.
$\square$
\end{example}

Given the definition for relational paths, it is simple to define relational variables and their instances.

\begin{definition}[Relational variable]
\label{def:rel-var}
A \textit{relational variable} $[I_j,\dots, I_k].X$ consists of a relational path $[I_j,\dots, I_k]$ and an attribute class $X\in\mathcal{A}(I_k)$.
\end{definition}

As with relational paths, we refer to $I_j$ as the perspective of the relational variable.
Relational variables are templates for sets of random variables (see Definition~\ref{def:rel-var-instance}).
Sets of relational variables are the basis of relational \textit{d}-separation queries, and consequently they are also the nodes of the abstract representation that answers those queries.
There is an equivalent formulation in the PRM framework, although not explicitly named (they are simply denoted as attribute classes of $\mathbf{K}$-related item classes via slot chain $\mathbf{K}$).
As they are critical to relational \textit{d}-separation, we provide this concept with an explicit designation.

\begin{example}
\label{ex:rel-var}
Relational variables for the relational paths in Example~\ref{ex:rel-path} include intrinsic attributes such as $[$\textsc{Employee}$]$.\textit{Competence} and $[$\textsc{Employee}$]$.\textit{Salary}, and also attributes on related entity classes such as $[$\textsc{Employee}, \textsc{Develops}, \textsc{Product}$]$.\textit{Success}, $[$\textsc{Employee}, \textsc{Develops}, \textsc{Product}, \textsc{Funds}, \textsc{Business-Unit}$]$.\textit{Revenue}, and $[$\textsc{Employee}, \textsc{Develops}, \textsc{Product}, \textsc{Develops}, \textsc{Employee}$]$.\textit{Salary}.
$\square$
\end{example}

\begin{definition}[Relational variable instance]
\label{def:rel-var-instance}
For skeleton $\sigma\in\Sigma_\mathcal{S}$ and $i_j\in \sigma(I_j)$, a \textit{relational variable instance} $[I_j,\dots, I_k].X\vert_{i_j}$ for relational variable $[I_j,\dots, I_k].X$ is the set of random variables $\lbrace i_k.X\ \vert\ X\!\in\!\mathcal{A}(I_k)\ \land\ i_k\!\in\![I_j,\dots, I_k]\vert_{i_j}\ \land\ i_k\!\in\!\sigma(I_k) \rbrace$.
\end{definition}

To instantiate a relational variable $[I_j,\dots, I_k].X$ for a specific base item instance $i_j$, we first find the terminal set of the underlying relational path $[I_j,\dots, I_k]\vert_{i_j}$ and then take the $X$ attributes of the $I_k$ item instances in that terminal set.
This produces a set of random variables $i_k.X$, which also correspond to nodes in the ground graph.
As a notational convenience, if $\mathbf{X}$ is a set of relational variables, all from a common perspective $I_j$, then we say that $\mathbf{X}\vert_{i_j}$ for some item $i_j\in\sigma(I_j)$ is the union of all instantiations, $\{x\ \vert\ x\!\in\!X\vert_{i_j}\ \land\ X\!\in\!\mathbf{X}\}$.

\begin{example}
\label{ex:rel-var-instance}
Instantiating the relational variables from Example~\ref{ex:rel-var} with base item instance Sally yields $[$\textsc{Employee}$]$.\textit{Competence}$\vert_{\text{Sally}} = \{$Sally.\textit{Competence}$\}$, $[$\textsc{Employee}, \textsc{Develops}, \textsc{Product}$]$.\textit{Success}$\vert_{\text{Sally}}\ =\ \{${\prodc}.\textit{Success}, {\prodd}.\textit{Success}$\}$, $[$\textsc{Employee}, \textsc{Develops}, \textsc{Product}, \textsc{Funds}, \textsc{Business-Unit}$]$.\textit{Revenue}$\vert_{\text{Sally}}\! =\! \{${\businessb}.\textit{Revenue}$\}$, and $[$\textsc{Employee}, \textsc{Develops}, \textsc{Product}, \textsc{Develops}, \textsc{Employee}$]$.\textit{Salary}$\vert_{\text{Sally}}\! =\! \{$Quinn.\textit{Salary}, Thomas.\textit{Salary}$\}$.
$\square$
\end{example}

Given the definitions for relational variables, we can now define relational dependencies.

\begin{definition}[Relational dependency]
\label{def:rel-dependency}
A \textit{relational dependency} $\ [I_j,\dots,I_k].Y\rightarrow$ \\$[I_j].X$ is a directed probabilistic dependence from attribute class $Y$ to $X$ through the relational path $[I_j,\dots,I_k]$.
\end{definition}

Depending on the context, $[I_j,\dots,I_k].Y$ and $[I_j].X$ can be referred to as \textit{treatment} and \textit{outcome}, \textit{cause} and \textit{effect}, or \textit{parent} and \textit{child}.
A relational dependency consists of two relational variables having a common perspective.
The relational path of the child is restricted to a single item class, ensuring that the terminal sets consist of a single value.
This is consistent with PRMs, except that we explicitly delineate dependencies rather than define parent sets of relational variables.
Note that relational variables are not nodes in a relational model, but they form the space of parent variables for relational dependencies.
The relational path specification (before the attribute class of the parent) is equivalent to a slot chain, as in PRMs, or the logical constraint on a dependency, as in DAPER models.

\begin{example}
\label{ex:rel-dependency}
The dependencies in the relational model displayed in Figure~\ref{fig:org_example_model}\subref{fig:org_model} can be specified as: $[$\textsc{Product}, \textsc{Develops}, \textsc{Employee}$].Competence \rightarrow [$\textsc{Product}$].Success$ (product success is influenced by the competence of the employees developing the product), $[$\textsc{Employee}$].Competence \rightarrow [$\textsc{Employee}$].Salary$ (an employee's competence affects his or her salary), $[$\textsc{Business-Unit}, \textsc{Funds}, \textsc{Product}$].Success \rightarrow [$\textsc{Business-Unit}$].Revenue$ (the success of the products funded by a business unit influences that unit's revenue), 
$[$\textsc{Employee}, \textsc{Develops}, \textsc{Product}, \textsc{Funds}, \textsc{Business-\!Unit}$].Budget\!\! \rightarrow\!\! [$\textsc{Employee}$].Salary$ (employee salary is governed by the budget of the business units for which they develop products), and $[$\textsc{Business-Unit}$].Revenue \rightarrow [$\textsc{Business-Unit}$].Budget$ (the revenue of a business unit influences its budget).
$\square$
\end{example}

We now have sufficient information to define relational models.

\begin{definition}[Relational model]
\label{def:rel-model}
A \textit{relational model} $\mathcal{M}_\Theta$ consists of two parts:
\begin{enumerate}
\item The \textit{structure} $\mathcal{M} = (\mathcal{S}, \mathcal{D})$: a schema $\mathcal{S}$ paired with a set of relational dependencies $\mathcal{D}$ defined over $\mathcal{S}$.
\item The \textit{parameters} $\Theta$: a conditional probability distribution $P\big([I_j].X\ \vert\ \mathit{parents}([I_j].X)\big)$ for each relational variable of the form $[I_j].X$, where $I_j\in\mathcal{E\cup R}$, $X\in\mathcal{A}(I_j)$ and $\mathit{parents}\big([I_j].X\big)=\big\{[I_j,\dots, I_k].Y\ \vert\ [I_j,\dots, I_k].Y\rightarrow [I_j].X\in\mathcal{D}\big\}$ is the set of parent relational variables.
\end{enumerate}
\end{definition}

The structure of a relational model can be represented graphically by superimposing dependencies on the ER diagram of a relational schema (see Figure~\ref{fig:org_example_model}\subref{fig:org_model} for an example).
A relational dependency of the form $[I_j,\dots,I_k].Y\rightarrow [I_j].X$ is depicted as a directed arrow from attribute class $Y$ to $X$ with the specification listed separately.
Note that the subset of relational variables with singleton paths $[I].X$ in the definition correspond to the set of attribute classes in the schema.

A common technique in relational learning is to use aggregation functions to transform parent multi-sets to single values within the conditional probability distributions.
Typically, aggregation functions are simple, such as mean or mode, but they can be complex, such as those based on vector distance or object identifiers, as in the ACORA system \citep{perlich2006acora}.
However, aggregates are a convenience for increasing power and accuracy during learning, but they are not necessary for model specification.

This definition of relational models is consistent with and yields structures expressible as DAPER models \citep{heckerman-etal-introsrl07}.
These relational models are also equivalent to PRMs, but we extend slot chains as relational paths and provide a formal semantics for their instantiation.
These models are also more general than plate models because dependencies can be specified with arbitrary relational paths as opposed to simple intersections among plates \citep{buntine-jair94, gilks1994BUGS}.

Just as the relational schema is a template for skeletons, the structure of a relational model can be viewed as a template for ground graphs: dependencies applied to skeletons.

\begin{definition}[Ground graph]
\label{def:ground-graph}
The \textit{ground graph} $\mathit{GG}_{\mathcal{M}\sigma} = (V, E)$ for relational model structure $\mathcal{M} = (\mathcal{S}, \mathcal{D})$ and skeleton $\sigma\in\Sigma_\mathcal{S}$ is a directed graph with nodes $V = \big\lbrace i.X\ \vert\ I\!\in\!\mathcal{E}\cup\mathcal{R}\ \land\ X\!\in\!\mathcal{A}(I)\ \land\ i\!\in\!\sigma(I) \big\rbrace$ and edges $E = \big\lbrace i_k.Y\rightarrow i_j.X\ \vert\ i_k.Y, i_j.X\!\in\! V\ \land\ i_k.Y\!\in\! [I_j,\dots, I_k].Y\vert_{i_j}\ \land\ [I_j,\dots, I_k].Y\rightarrow [I_j].X\!\in\!\mathcal{D} \big\rbrace$.
\end{definition}

A ground graph is a directed graph with (1) a node (random variable) for each attribute of every entity and relationship instance in a skeleton and (2) an edge from $i_k.Y$ to $i_j.X$ if they belong to the parent and child relational variable instances, respectively, of some dependency in the model.
The concept of a ground graph appears for any type of relational model, graphical or logic-based.
For example, PRMs produce ``ground Bayesian networks" that are structurally equivalent to ground graphs, and Markov logic networks yield ground Markov networks by applying all formulas to a set of constants \citep{richardson-domingos-ml06}.
The example ground graph shown in Figure~\ref{fig:org_example_model}\subref{fig:org_ground_graph} is the result of applying the dependencies in the relational model shown in Figure~\ref{fig:org_example_model}\subref{fig:org_model} to the skeleton in Figure~\ref{fig:org_example_data}\subref{fig:org_skeleton}.

Similar to Bayesian networks, given the parameters of a relational model, a \textit{parameterized ground graph} can express a joint distribution that factors as a product of the conditional distributions:
\begin{equation*}
P(\mathit{GG}_{\mathcal{M}_\Theta\sigma}) = \displaystyle\prod_{I\in\mathcal{E}\cup\mathcal{R}} \displaystyle\prod_{X\in\mathcal{A}(I)} \displaystyle\prod_{i\in\sigma(I)} P\big(i.X\ \vert\ \mathit{parents}(i.X)\big)
\end{equation*}
where each $i.X$ is assigned the conditional distribution defined for $[I].X$ (a process referred to as parameter-tying).

Relational models only define coherent joint probability distributions if they produce acyclic ground graphs.
A useful construct for checking model acyclicity is the class dependency graph \citep{getoor:prm-ch-srl-book07}, defined as:

\begin{definition}[Class dependency graph]
\label{def:class-dependency-graph}
The class dependency graph $G_\mathcal{M} = (V, E)$ for relational model structure $\mathcal{M} = (\mathcal{S}, \mathcal{D})$ is a directed graph with a node for each attribute of every item class $V = \big\{I.X\ \vert\ I\!\in\!\mathcal{E\cup R}\ \land\ X\!\in\!\mathcal{A}(I)\big\}$ and edges between pairs of attributes supported by relational dependencies in the model $E = \big\{I_k.Y\rightarrow I_j.X\ \vert\ [I_j,\dots,I_k].Y\rightarrow [I_j].X\!\in\!\mathcal{D}\big\}$.
\end{definition}

If the relational dependencies form an acyclic class dependency graph, then every possible ground graph of that model is acyclic as well \citep{getoor:prm-ch-srl-book07}.
Given an acyclic relational model, the ground graph has the same semantics as a Bayesian network \citep{getoor2001thesis, heckerman-etal-introsrl07}.
All future references to acyclic relational models refer to relational models whose structure forms acyclic class dependency graphs.

By Lemma~\ref{lemma:overlap} and Definition~\ref{def:ground-graph}, one relational dependency may imply dependence between the instances of many relational variables.
If there is an edge from $i_k.Y$ to $i_j.X$ in the ground graph, then there is an implied dependency between all relational variables for which $i_k.Y$ and $i_j.X$ are elements of their instances.

\begin{example}
\label{ex:gg-implied-dep}
The relational dependency $[$\textsc{Employee}$].Competence\!\! \rightarrow\!\! [$\textsc{Employee}$].Salary$ yields the edge Roger.\textit{Competence} $\rightarrow$ Roger.\textit{Salary} in the ground graph of Figure~\ref{fig:org_example_model}\subref{fig:org_ground_graph} because Roger.\textit{Competence} $\in [$\textsc{Employee}$].Competence\vert_{\text{Roger}}$.
However, Roger.\textit{Competence} $\in [$\textsc{Employee}, \textsc{Develops}, \textsc{Product}, \textsc{Develops}, \textsc{Employee}$].Competence\vert_{\text{Sally}}$ (as is Roger.\textit{Salary}, replacing \textit{Competence} with \textit{Salary}).  Consequently, the relational dependency \textit{implies} dependence among the random variables in the instances of $[$\textsc{Employee}, \textsc{Develops}, \textsc{Product}, \textsc{Develops}, \textsc{Employee}$].Competence$ and $[$\textsc{Employee}, \textsc{Develops}, \textsc{Product}, \textsc{Develops}, \textsc{Employee}$].Salary$.
$\square$
\end{example}

These implied dependencies form the crux of the challenge of identifying independence in relational models.
Additionally, the intersection between the terminal sets of two relational paths is crucial for reasoning about independence because a random variable can belong to the instances of more than one relational variable.
Since \textit{d}-separation only guarantees independence when there are no \textit{d}-connecting paths, we must consider all possible paths between pairs of random variables, either of which may be a member of multiple relational variable instances.
In Section~\ref{sec:rds}, we define relational \textit{d}-separation and provide an appropriate representation, the abstract ground graph, that enables straightforward reasoning about \textit{d}-separation.

\section{Relational \textit{d}-Separation}
\label{sec:rds}

Conditional independence facts are correctly entailed by the rules of \textit{d}-separation, but only when applied to the graphical structure of Bayesian networks.
Every ground graph of a Bayesian network consists of a set of identical copies of the model structure (see Appendix~\ref{sec:appendix_prop_background}).
Thus, the implications of \textit{d}-separation on Bayesian networks hold for all instances in every ground graph.
In contrast, the structure of a relational model is a template for ground graphs, and the structure of a ground graph varies with the underlying skeleton (which is typically more complex than a set of disconnected instances).
Conditional independence facts are only useful when they hold across all ground graphs that are consistent with the model, which leads to the following definition:
	
\begin{definition}[Relational \textit{d}-separation]
\label{def:rel-d-sep}
Let $\mathbf{X}$, $\mathbf{Y}$, and $\mathbf{Z}$ be three distinct sets of relational variables with the same perspective $B\in\mathcal{E}\cup\mathcal{R}$ defined over relational schema $\mathcal{S}$.  Then, for relational model structure $\mathcal{M}$, $\mathbf{X}$ and $\mathbf{Y}$ are \textit{d}-separated by $\mathbf{Z}$ if and only if, for all skeletons $\sigma\in\Sigma_\mathcal{S}$, $\mathbf{X}\vert_b$ and $\mathbf{Y}\vert_b$ are \textit{d}-separated by $\mathbf{Z}\vert_b$ in ground graph $\mathit{GG}_{\mathcal{M}\sigma}$ for all $b\in \sigma(B)$.
\end{definition}

For any relational \textit{d}-separation query, it is necessary that all relational variables in $\mathbf{X}$, $\mathbf{Y}$, and $\mathbf{Z}$ have the same perspective (otherwise, the query would be incoherent).\footnote{This trivially holds for \textit{d}-separation in Bayesian networks as all ``propositional" variables have the same implicit perspective.}
For $\mathbf{X}$ and $\mathbf{Y}$ to be \textit{d}-separated by $\mathbf{Z}$ in relational model structure $\mathcal{M}$, \textit{d}-separation must hold for all instantiations of those relational variables for all possible skeletons.
This is a conservative definition, but it is consistent with the semantics of \textit{d}-separation on Bayesian networks---it guarantees independence, but it does not guarantee dependence.
If there exists even one skeleton and faithful distribution represented by the relational model for which $\mathbf{X}\Perpn \mathbf{Y}\ \vert\ \mathbf{Z}$, then 
$\mathbf{X}$ and $\mathbf{Y}$ are not \textit{d}-separated by $\mathbf{Z}$.

Given the semantics specified in Definition~\ref{def:rel-d-sep}, answering relational \textit{d}-separation queries is challenging for several reasons:

\textit{D-separation must hold over all ground graphs}: 
The implications of \textit{d}-separation on Bayesian networks hold for all possible ground graphs.
However, the ground graphs of a Bayesian network consist of identical copies of the structure of the model, and it is sufficient to reason about \textit{d}-separation on a single subgraph.
Although it is possible to verify \textit{d}-separation on a single ground graph of a relational model, the conclusion may not generalize, and ground graphs can be arbitrarily large.  
	
\textit{Relational models are templates}: The structure of a relational model is a directed acyclic graph, but the dependencies are actually templates for constructing ground graphs.  
The rules of \textit{d}-separation do not directly apply to relational models, only to their ground graphs.  
Applying the rules of \textit{d}-separation to a relational model frequently leads to incorrect conclusions because of unrepresented \textit{d}-connecting paths that are only manifest in ground graphs.

\textit{Instances of relational variables may intersect}: The instances of two different relational variables may have non-empty intersections, as described by Lemma~\ref{lemma:overlap}.  
These intersections may be involved in relationally \textit{d}-connecting paths, such as the example in Section~\ref{sec:example}.
As a result, a sound and complete approach to answering relational \textit{d}-separation queries must account for these paths.
	
\textit{Relational models may be specified from multiple perspectives}: 
Relational models are defined by relational dependencies, each specified from a single perspective.  
However, variables in a ground graph may contribute to \textit{multiple} relational variable instances, each defined from a \textit{different} perspective.
Thus, reasoning about implied dependencies between arbitrary relational variables, such as the one described in Example~\ref{ex:gg-implied-dep},  requires a method to translate dependencies across perspectives.

\subsection{Abstracting over All Ground Graphs}
\label{sec:aggs}

The definition of relational \textit{d}-separation and its challenges suggest a solution that abstracts over all possible ground graphs and explicitly represents the potential intersection between pairs of relational variable instances.
We introduce a new lifted representation, called the \textit{abstract ground graph}, that captures all dependencies among arbitrary relational variables for all ground graphs, using knowledge of only the schema and the model.
To represent all dependencies, the construction of an abstract ground graph uses the \textit{extend} method, which maps a relational dependency to a set of implied dependencies for different perspectives.
Each abstract ground graph of a relational model is defined with respect to a given perspective and can be used to reason about relational \textit{d}-separation queries for that perspective.

\begin{definition}[Abstract ground graph]
\label{def:abstract-gg}
An \textit{abstract ground graph} $\mathit{AGG}_{\mathcal{M}B} = (V, E)$ for relational model structure $\mathcal{M} = (\mathcal{S}, \mathcal{D})$ and perspective $B\in\mathcal{E}\cup\mathcal{R}$ is a directed graph that abstracts the dependencies $\mathcal{D}$ for all ground graphs $\mathit{GG}_{\mathcal{M}\sigma}$, where $\sigma\in\Sigma_\mathcal{S}$.

The set of nodes in $\mathit{AGG}_{\mathcal{M}B}$ is $V = \mathit{RV} \cup \mathit{IV}$, where 
\begin{itemize}
\item $\mathit{RV}$ is the set of all relational variables of the form $[B,\dots,I_j].X$
\item $\mathit{IV}$ is the set of all pairs of relational variables that could have non-empty intersections (referred to as intersection variables):
\begin{align*}
\big\lbrace \mathit{RV}\!_1\cap \mathit{RV}\!_2\ \vert\ \mathit{RV}\!_1, \mathit{RV}\!_2\!\in\!\mathit{RV}\ &\land\ \mathit{RV}\!_1=[B,\dots,I_k,\dots, I_j].X\\
 &\land\ \mathit{RV}\!_2=[B,\dots,I_l,\dots, I_j].X\ \land\ I_k\ne I_l \big\rbrace
\end{align*}
\end{itemize}

The set of edges in $\mathit{AGG}_{\mathcal{M}B}$ is $E = \mathit{RVE}\ \cup\ \mathit{IVE}$, where
\begin{itemize}
\item $\mathit{RVE}\subset \mathit{RV}\times \mathit{RV}$ is the set of edges between pairs of relational variables:
\begin{align*}
\mathit{RVE} = \big\lbrace [B,\dots, I_k].Y\rightarrow [B,\dots, I_j].X\ \vert\ & [I_j,\dots, I_k].Y\rightarrow [I_j].X\in\mathcal{D}\ \ \land\\ & [B,\dots, I_k]\in \mathit{extend}([B,\dots, I_j], [I_j,\dots, I_k]) \big\rbrace 
\end{align*}

\item $\mathit{IVE}\subset \mathit{IV}\times \mathit{RV}\ \cup\ \mathit{RV}\times \mathit{IV}$ is the set of edges inherited from both relational variables involved in every intersection variable in $\mathit{IV}$: 
\begin{align*}
\mathit{IVE} = \big\lbrace \hat{Y}\rightarrow [B,\dots, I_j].X\ \vert\ &\hat{Y} = P_1.Y\cap P_2.Y\in \mathit{IV}\ \ \land\ \\
&(P_1.Y\rightarrow [B,\dots, I_j].X\in \mathit{RVE}\ \ \lor\ \\
& \hspace{1.5mm} P_2.Y\rightarrow [B,\dots, I_j].X\in \mathit{RVE}) \big\rbrace\
\end{align*}
\hspace{64mm} $\bigcup$
\begin{align*}
\hspace{12mm} \big\lbrace [B,\dots, I_k].Y\rightarrow \hat{X}\ \vert\ &\hat{X} = P_1.X\cap P_2.X\in \mathit{IV}\ \ \land\ \\ 
&([B,\dots, I_k].Y\rightarrow P_1.X\in \mathit{RVE}\ \ \lor\ \\
&\hspace{1.5mm} [B,\dots, I_k].Y\rightarrow P_2.X\in \mathit{RVE}) \big\rbrace
\end{align*}
\end{itemize}
\end{definition}

The \textit{extend} method is described in Definition~\ref{def:extend} below.  
Essentially, the construction of an abstract ground graph for relational model structure $\mathcal{M}$ and perspective $B\in\mathcal{E}\cup\mathcal{R}$ follows three simple steps: (1) Add a node for all relational variables from perspective $B$.\footnotemark{} (2) Insert edges for the direct causes of every relational variable by translating the dependencies in $\mathcal{D}$ using \textit{extend}. (3) For each pair of potentially intersecting relational variables, create a new node that inherits the direct causes and effects from both participating relational variables in the intersection.
Then, to answer queries of the form ``Are $\mathbf{X}$ and $\mathbf{Y}$ \textit{d}-separated by $\mathbf{Z}$?" simply (1) augment $\mathbf{X}$, $\mathbf{Y}$, and $\mathbf{Z}$ with their corresponding intersection variables that they subsume and (2) apply the rules of \textit{d}-separation on the abstract ground graph for the common perspective of $\mathbf{X}$, $\mathbf{Y}$, and $\mathbf{Z}$.
Since abstract ground graphs are defined from a specific perspective, every relational model produces a \textit{set of abstract ground graphs}, one for each perspective in its underlying schema.
\footnotetext{In theory, abstract ground graphs can have an infinite number of nodes as relational paths may have no bound. In practice, a \textit{hop threshold} $h\in\mathbb{N}^0$ is enforced to limit the length of these paths. Hops are defined as the number of times the path ``hops" between item classes in the schema, or one less than the length of the path.}

\begin{example}
\label{ex:agg}
Figure~\ref{fig:org_agg_h6} shows the abstract ground graph $\mathit{AGG}_{\mathcal{M},\,\text{\textsc{Employee}}}$ for the organization example from the \textsc{Employee} perspective with hop threshold $h=6$.\footnote{The variables \textit{Salary} and \textit{Budget} are removed for simplicity.  They are irrelevant for this \textit{d}-separation example as they are solely effects of other variables.}
As in Section~\ref{sec:example}, we derive an appropriate conditioning set $\mathbf{Z}$ in order to \textit{d}-separate individual employee competence ($\mathbf{X} = \{[$\textsc{Employee}$].Competence\}$) from the revenue of the employee's funding business units ($\mathbf{Y} = \{[$\textsc{Employee}, \textsc{Develops}, \textsc{Product}, \textsc{Funds}, \textsc{Business-Unit}$].Revenue\}$).
Applying the rules of \textit{d}-separation to the abstract ground graph, we see that it is necessary to condition on both product success ($[$\textsc{Employee}, \textsc{Develops}, \textsc{Product}$].Success$) and the competence of other employees developing the same products ($[$\textsc{Employee}, \textsc{Develops}, \textsc{Product}, \textsc{Develops}, \textsc{Employee}$].Competence$).
For $h=6$, augmenting $\mathbf{X}$, $\mathbf{Y}$, and $\mathbf{Z}$ with their corresponding intersection variables does not result in any changes.
For $h=8$, the abstract ground graph includes a node for relational variable $[$\textsc{Employee}, \textsc{Develops}, \textsc{Product}, \textsc{Develops}, \textsc{Employee}, \textsc{Develops}, \textsc{Product}, \textsc{Funds}, \textsc{Business-Unit}$].Revenue$ (the revenue of the business units funding the other products of collaborating employees) which, by Lemma~\ref{lemma:overlap}, could have a non-empty intersection with $[$\textsc{Employee}, \textsc{Develops}, \textsc{Product}, \textsc{Funds}, \textsc{Business-Unit}$].Revenue$.
Therefore, $\mathbf{Y}$ would also include the intersection with this other relational variable.
However, for this query, the conditioning set $\mathbf{Z}$ for $h=6$ happens to also \textit{d}-separate at $h=8$ (and any $h\in\mathbb{N}^0$).
$\square$
\end{example}

Using the algorithm devised by \citet{geiger1990iibn}, relational \textit{d}-separation queries can be answered in $O(\vert E\vert)$ time with respect to the number of edges in the abstract ground graph.
In practice, the size of an abstract ground graph depends on the relational schema and model (e.g., the number of entity classes, the types of cardinalities, the number of dependencies---see the experiment in Section~\ref{sec:agg-size}), as well as the hop threshold limiting the length of relational paths.
For the example in Figure~\ref{fig:org_agg_h6}, the abstract ground graph has 7 nodes and 7 edges (including 1 intersection node with 2 edges); for $h=8$, it would have 13 nodes and 21 edges (including 4 intersection nodes with 13 edges).
Abstract ground graphs are invariant to the size of ground graphs, even though ground graphs can be arbitrarily large---that is, relational databases have no maximum size.

	\begin{figure}[t]
	\centering	
	\includegraphics[width=152mm]{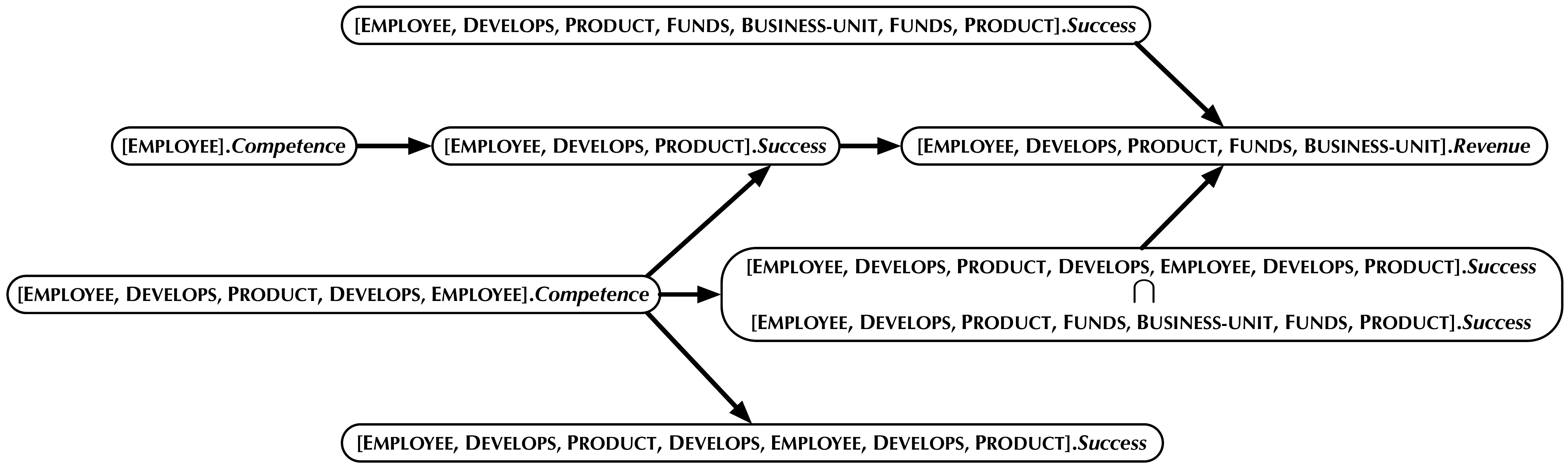}
	\caption{The abstract ground graph for the organization domain model in Figure~\ref{fig:org_example_model}\protect\subref{fig:org_model} from the \textsc{Employee} perspective with hop threshold $h=6$ (with the variables for \textit{Salary} and \textit{Budget} omitted for simplicity). This abstract ground graph includes one intersection node.}
	\label{fig:org_agg_h6}
	\end{figure}

Next, we formally define the \textit{extend} method, which is used internally for the construction of abstract ground graphs.
This method translates dependencies specified in the model into dependencies in the abstract ground graph.

\begin{definition}[Extending relational paths]
\label{def:extend}
Let $P_{\mathit{orig}}$ and $P_{\mathit{ext}}$ be two relational paths for schema $\mathcal{S}$. The following three functions extend $P_{\mathit{orig}}$ with $P_{\mathit{ext}}$:
\begin{align*}
\mathit{extend}(P_{\mathit{orig}}, P_{\mathit{ext}}) =\ & \big\lbrace P\! =\! P_{\mathit{orig}}^{1,n_o-i+1}\! +\! P_{\mathit{ext}}^{i+1,n_e}\, \vert\, i\!\in\!\mathit{pivots}(\mathit{reverse}(P_{\mathit{orig}}), P_{\mathit{ext}}) \land \mathit{isValid}(P)\big\rbrace\\
& \\
\mathit{pivots}(P_1, P_2) =\ & \{i\ \vert\ P_1^{1,i} = P_2^{1,i}\}\\
& \\
\mathit{isValid}(P) =\ & \left\lbrace \begin{array}{ll}
\text{True} & \text{if $P$ does not violate Definition \ref{def:rel-path}}\\
\text{False} & \text{otherwise}
\end{array}\right.
\end{align*}
\noindent where $n_o$ is the length of $P_{\mathit{orig}}$, $n_e$ is the length of $P_{\mathit{ext}}$, $P^{i,j}$ corresponds to 1-based $i$-inclusive, $j$-inclusive subpath indexing, $+$ is concatenation of paths, and \textit{reverse} is a method that reverses the order of the path.
\end{definition}

The \textit{extend} method constructs a set of valid relational paths from two input relational paths.  
It first finds the indices (called pivots) of the item classes for which the input paths ($\mathit{reverse}(P_{\mathit{orig}})$ and $P_{\mathit{ext}}$) have a common starting subpath.
Then, it concatenates the two input paths at each pivot, removing one of the duplicated subpaths (see Example~\ref{ex:extend}).
Since \textit{d}-separation requires blocking all paths of dependence between two sets of variables, the \textit{extend} method is critical to ensure the soundness and completeness of our approach.
The abstract ground graph must capture all paths of dependence among the random variables in the relational variable instances for all represented ground graphs.
However, relational model structures are specified by relational dependencies, each from a given perspective and with outcomes that have singleton relational paths.
The \textit{extend} method is called repeatedly during the creation of an abstract ground graph, with $P_{\mathit{orig}}$ set to some relational path and $P_{\mathit{ext}}$ drawn from the relational path of the treatment in some relational dependency.

\begin{example}
\label{ex:extend}
During the construction of the abstract ground graph $\mathit{AGG}_{\mathcal{M},\,\text{\textsc{Employee}}}$ depicted in Figure~\ref{fig:org_agg_h6}, the \textit{extend} method is called several times.
First, all relational variables from the \textsc{Employee} perspective are added as nodes in the graph.
Next, \textit{extend} is used to insert edges corresponding to direct causes.
Consider the node for $[$\textsc{Employee}, \textsc{Develops}, \textsc{Product}$].Success$.
The construction of $\mathit{AGG}_{\mathcal{M},\,\text{\textsc{Employee}}}$ calls $\mathit{extend}(P_{\mathit{orig}}, P_{\mathit{ext}})$ with $P_{\mathit{orig}} = [$\textsc{Employee}, \textsc{Develops}, \textsc{Product}$]$ and $P_{\mathit{ext}} = [$\textsc{Product}, \textsc{Develops}, \textsc{Employee}$]$ because $[$\textsc{Product}, \textsc{Develops}, \textsc{Employee}$].Competence\! \rightarrow\! [$\textsc{Product}$].Success\! \in\!\mathcal{D}$.
Here, $\mathit{extend}(P_{\mathit{orig}}, P_{\mathit{ext}}) = \{[$\textsc{Employee}$]$, $[$\textsc{Employee}, \textsc{Develops}, \textsc{Product}, \textsc{Develops}, \textsc{Employee}$]\}$, which leads to the insertion of two edges in the abstract ground graph.
Note that $\mathit{pivots}(\mathit{reverse}(P_{\mathit{orig}}), P_{\mathit{ext}}) = \{1,2,3\}$, and the pivot at $i=2$ yields the invalid relational path $[$\textsc{Employee}, \textsc{Develops}, \textsc{Employee}$]$.
$\square$
\end{example}

We also describe two important properties of the \textit{extend} method with the following two lemmas.
The first lemma states that every relational path produced by \textit{extend} yields a terminal set for some skeleton such that there is an item instance also reachable by the two original paths.
This lemma is useful for proving the soundness of our abstraction: All edges inserted in an abstract ground graph correspond to edges in some ground graph.

\begin{mylemma}
\label{lemma:extend}
Let $P_{\mathit{orig}} = [I_1,\dots, I_j]$ and $P_{\mathit{ext}} = [I_j,\dots, I_k]$ be two relational paths with $\mathbf{P} = \mathit{extend}(P_{\mathit{orig}}, P_{\mathit{ext}})$.  Then, $\forall P\in\mathbf{P}$ there exists a relational skeleton $\sigma\in\Sigma_\mathcal{S}$ such that $\exists i_1\in \sigma(I_1)$ such that $\exists i_k\in P\vert_{i_1}$ and $\exists i_j\in P_{\mathit{orig}}\vert_{i_1}$ such that $i_k\in P_{\mathit{ext}}\vert_{i_j}$.
\end{mylemma}

\begin{proof}
See Appendix~\ref{sec:appendix_proofs}.
\end{proof}

\begin{example}
\label{ex:lemma-extend}
Let $\sigma$ be the skeleton shown in Figure~\ref{fig:org_example_data}\subref{fig:org_skeleton}, let $P_{\mathit{orig}} = [$\textsc{Employee}, \textsc{Develops}, \textsc{Product}$]$, let $P_{\mathit{ext}} = [$\textsc{Product}, \textsc{Develops}, \textsc{Employee}$]$, and let $i_1$ = Sally $\in \sigma($\textsc{Employee}$)$.
From Example~\ref{ex:extend}, we know that $\mathbf{P} = \mathit{extend}(P_{\mathit{orig}}, P_{\mathit{ext}}) = \{[$\textsc{Employee}$]$, $[$\textsc{Employee}, \textsc{Develops}, \textsc{Product}, \textsc{Develops}, \textsc{Employee}$]\}$.
We also have $[$\textsc{Employee}$]\vert_{\text{Sally}} = \{$Sally$\}$ and $[$\textsc{Employee}, \textsc{Develops}, \textsc{Product}, \textsc{Develops}, \textsc{Employee}$]\vert_{\text{Sally}} = \{$Quinn, Roger, Thomas$\}$.
By Lemma~\ref{lemma:extend}, there should exist an $i_j\in P_{\mathit{orig}}\vert_{i_1}$ such that Sally and at least one of Quinn, Roger, and Thomas are in the terminal set $P_{\mathit{ext}}\vert_{i_j}$.
We have $P_{\mathit{orig}}\vert_{\text{Sally}} = \{${\prodc}, {\prodd}$\}$, and $P_{\mathit{ext}}\vert_{\text{{\prodc}}} = \{$Quinn, Roger, Sally$\}$ and $P_{\mathit{ext}}\vert_{\text{{\prodd}}} = \{$Sally, Thomas$\}$.
So, the lemma clearly holds for this example.
$\square$
\end{example}

Lemma~\ref{lemma:extend} guarantees that, for some relational skeleton, there exists an item instance in the terminal sets produced by \textit{extend} that also appears in the terminal set of $P_{\mathit{ext}}$ via some instance in the terminal set of $P_{\mathit{orig}}$.
It is also possible (although infrequent) that there exist items reachable by $P_{\mathit{orig}}$ and $P_{\mathit{ext}}$ that are not in the terminal set of any path produced with $\mathit{extend}(P_{\mathit{orig}}, P_{\mathit{ext}})$.
The following lemma describes this unreachable set of items, stating that there must exist an alternative relational path $P_{\mathit{orig}}'$ that intersects with $P_{\mathit{orig}}$ and, when using \textit{extend}, catches those remaining items.
This lemma is important for proving the completeness of our abstraction: All edges in all ground graphs are represented in the abstract ground graph.

\begin{mylemma}
\label{lemma:extend-reverse}
Let $\sigma\in\Sigma_\mathcal{S}$ be a relational skeleton, and let $P_{\mathit{orig}} = [I_1,\dots, I_j]$ and $P_{\mathit{ext}} = [I_j,\dots, I_k]$ be two relational paths with $\mathbf{P} = \mathit{extend}(P_{\mathit{orig}}, P_{\mathit{ext}})$.  Then, $\forall i_1\!\in\! \sigma(I_1)\ \forall i_j\!\in\! P_{\mathit{orig}}\vert_{i_1}\ \forall i_k\!\in\! P_{\mathit{ext}}\vert_{i_j}$ if $\forall P\in\mathbf{P}\ i_k\notin P\vert_{i_1}$, then $\exists P_{\mathit{orig}}'$ such that $P_{\mathit{orig}}\vert_{i_1}\ \cap\ P_{\mathit{orig}}'\vert_{i_1}\neq\emptyset$ and $i_k\in P'\vert_{i_1}$ for some $P'\in \mathit{extend}(P_{\mathit{orig}}', P_{\mathit{ext}})$.
\end{mylemma}

\begin{proof}
See Appendix~\ref{sec:appendix_proofs}.
\end{proof}

\begin{example}
\label{ex:lemma-extend-reverse}
Although Lemma~\ref{lemma:extend-reverse} does not apply to the organization domain as currently represented, it could apply if either (1) there were cycles in the relational schema or (2) the path specifications on the relational dependencies included a cycle.
Consider additional relationships between employees and products.
If employees could be involved with products at various stages (e.g., research, development, testing, marketing), then there would be alternative relational paths for which the lemma might apply.
The proof of the lemma in Appendix~\ref{sec:appendix_proofs} provides abstract conditions describing when the lemma applies.
$\square$
\end{example}

\subsection{Proof of Correctness}
\label{sec:proof-of-correctness}

The correctness of our approach to relational \textit{d}-separation relies on several facts: (1) \textit{d}-separation is valid for directed acyclic graphs; (2) ground graphs are directed acyclic graphs; and (3) abstract ground graphs are directed acyclic graphs that represent exactly the edges that could appear in all possible ground graphs.
It follows that \textit{d}-separation on abstract ground graphs, augmented by intersection variables, is sound and complete for all ground graphs.\footnotemark{}
Additionally, we show that since relational \textit{d}-separation is sound and complete, it is also equivalent to the Markov condition for relational models.
Using the previous definitions and lemmas, the following sequence of results proves the correctness of our approach to identifying independence in relational models.
\footnotetext{In Appendix~\ref{sec:hop_thresholds}, we provide proofs of soundness and completeness for abstract ground graphs and relational \textit{d}-separation that are limited by practical hop threshold bounds.}

\begin{mytheorem}
\label{thm:d-sep}
The rules of \textit{d}-separation are sound and complete for directed acyclic graphs.
\end{mytheorem}

\begin{proof}
Due to \citet{verma1988cnse} for soundness and \citet{geiger1988lcm} for completeness. $\blacksquare$
\end{proof}

Theorem~\ref{thm:d-sep} implies that (1) all conditional independence facts derived by \textit{d}-separation on a Bayesian network structure hold in any distribution represented by that model (soundness) and (2) all conditional independence facts that hold in a faithful distribution can be inferred from \textit{d}-separation applied to the Bayesian network that encodes the distribution (completeness).

\begin{mylemma}
\label{lemma:rolled-out-dags}
For every acyclic relational model structure $\mathcal{M}$ and skeleton $\sigma\in\Sigma_\mathcal{S}$, the ground graph $\mathit{GG}_{\mathcal{M}\sigma}$ is a directed acyclic graph.
\end{mylemma}

\begin{proof}
	Due to both \citet{heckerman-etal-introsrl07} for DAPER models and \citet{getoor2001thesis} for PRMs. $\blacksquare$
\end{proof}

By Theorem~\ref{thm:d-sep} and Lemma~\ref{lemma:rolled-out-dags}, \textit{d}-separation is sound and complete when applied to a ground graph.
However, Definition~\ref{def:rel-d-sep} states that relational \textit{d}-separation must hold across \textit{all possible} ground graphs, which is the reason for constructing the abstract ground graph representation.

\begin{mytheorem}
\label{thm:agg-abstracts}
For every acyclic relational model structure $\mathcal{M}$ and perspective $B\in\mathcal{E}\cup\mathcal{R}$, the abstract ground graph $\mathit{AGG}_{\mathcal{M}B}$ is sound and complete for all ground graphs $\mathit{GG}_{\mathcal{M}\sigma}$ with skeleton $\sigma\in\Sigma_\mathcal{S}$.
\end{mytheorem}

\begin{proof}
See Appendix~\ref{sec:appendix_proofs}.
\end{proof}

Theorem~\ref{thm:agg-abstracts} guarantees that, for a given perspective, an abstract ground graph captures all possible paths of dependence between any pair of variables in any ground graph.
The details of the proof provide the reasons why explicitly representing intersection variables is necessary for ensuring a sound and complete abstraction.

\begin{mytheorem}
\label{thm:agg-dags}
For every acyclic relational model structure $\mathcal{M}$ and perspective $B\in\mathcal{E}\cup\mathcal{R}$, the abstract ground graph $\mathit{AGG}_{\mathcal{M}B}$ is directed and acyclic.
\end{mytheorem}

\begin{proof}
See Appendix~\ref{sec:appendix_proofs}.
\end{proof}

Theorem~\ref{thm:agg-dags} ensures that the standard rules of \textit{d}-separation can apply directly to abstract ground graphs because they are acyclic given an acyclic model.
We now have sufficient supporting theory to prove that \textit{d}-separation on abstract ground graphs is sound and complete.
In the following theorem, we define $\mathbf{\bar{W}}$ as the set of nodes augmented with their corresponding intersection nodes for the set of relational variables $\mathbf{W}$: $\mathbf{\bar{W}} = \mathbf{W}\ \cup\ \bigcup_{W\in\mathbf{W}} \{ W\cap W' \ \vert\ W\cap W'\text{ is an intersection node in } \mathit{AGG}_{\mathcal{M}B}\}$.

\begin{mytheorem}
\label{thm:relational-d-sep}
Relational \textit{d}-separation is sound and complete for abstract ground graphs.
Let $\mathcal{M}$ be an acyclic relational model structure, and let $\mathbf{X}$, $\mathbf{Y}$, and $\mathbf{Z}$ be three distinct sets of relational variables for perspective $B\in\mathcal{E}\cup\mathcal{R}$ defined over relational schema $\mathcal{S}$.
Then, $\mathbf{\bar{X}}$ and $\mathbf{\bar{Y}}$ are \textit{d}-separated by $\mathbf{\bar{Z}}$ on the abstract ground graph $\mathit{AGG}_{\mathcal{M}B}$ if and only if for all skeletons $\sigma\in\Sigma_\mathcal{S}$ and for all $b\in \sigma(B)$, $\mathbf{X}\vert_b$ and $\mathbf{Y}\vert_b$ are \textit{d}-separated by $\mathbf{Z}\vert_b$ in ground graph $\mathit{GG}_{\mathcal{M}\sigma}$.
\end{mytheorem}

\begin{proof}
We must show that \textit{d}-separation on an abstract ground graph implies \textit{d}-separation on all ground graphs it represents (soundness) and that \textit{d}-separation facts that hold across all ground graphs are also entailed by \textit{d}-separation on the abstract ground graph (completeness).

\textbf{Soundness}: Assume that $\mathbf{\bar{X}}$ and $\mathbf{\bar{Y}}$ are \textit{d}-separated by $\mathbf{\bar{Z}}$ on $\mathit{AGG}_{\mathcal{M}B}$.
Assume for contradiction that there exists an item instance $b\in\sigma(B)$ such that $\mathbf{X}\vert_b$ and $\mathbf{Y}\vert_b$ are \textit{not d}-separated by $\mathbf{Z}\vert_b$ in the ground graph $\mathit{GG}_{\mathcal{M}\sigma}$ for some arbitrary skeleton $\sigma$.
Then, there must exist a \textit{d}-connecting path $p$ from some $x\in\mathbf{X}\vert_b$ to some $y\in\mathbf{Y}\vert_b$ given all $z\in\mathbf{Z}\vert_b$.
By Theorem~\ref{thm:agg-abstracts}, $\mathit{AGG}_{\mathcal{M}B}$ is complete, so all edges in $\mathit{GG}_{\mathcal{M}\sigma}$ are captured by edges in $\mathit{AGG}_{\mathcal{M}B}$.
So, path $p$ must be represented from some node in $\{N_x\ \vert\ x\in N_x\vert_b\}$ to some node in $\{N_y\ \vert\ y\in N_y\vert_b\}$, where $N_x, N_y$ are nodes in $\mathit{AGG}_{\mathcal{M}B}$.
If $p$ is \textit{d}-connecting in $\mathit{GG}_{\mathcal{M}\sigma}$, then it is \textit{d}-connecting in $\mathit{AGG}_{\mathcal{M}B}$, implying that $\mathbf{\bar{X}}$ and $\mathbf{\bar{Y}}$ are \textit{not d}-separated by $\mathbf{\bar{Z}}$.
So, $\mathbf{X}\vert_b$ and $\mathbf{Y}\vert_b$ must be \textit{d}-separated by $\mathbf{Z}\vert_b$.

\textbf{Completeness}: Assume that $\mathbf{X}\vert_b$ and $\mathbf{Y}\vert_b$ are \textit{d}-separated by $\mathbf{Z}\vert_b$ in the ground graph $\mathit{GG}_{\mathcal{M}\sigma}$ for all skeletons $\sigma$ for all $b\in\sigma(B)$.
Assume for contradiction that $\mathbf{\bar{X}}$ and $\mathbf{\bar{Y}}$ are \textit{not d}-separated by $\mathbf{\bar{Z}}$ on $\mathit{AGG}_{\mathcal{M}B}$.
Then, there must exist a \textit{d}-connecting path $p$ for some relational variable $X\in\mathbf{\bar{X}}$ to some $Y\in\mathbf{\bar{Y}}$ given all $Z\in\mathbf{\bar{Z}}$.
By Theorem~\ref{thm:agg-abstracts}, $\mathit{AGG}_{\mathcal{M}B}$ is sound, so every edge in $\mathit{AGG}_{\mathcal{M}B}$ must correspond to some pair of variables in some ground graph.
So, if $p$ is \textit{d}-connecting in $\mathit{AGG}_{\mathcal{M}B}$, then there must exist some skeleton $\sigma$ such that $p$ is \textit{d}-connecting in $\mathit{GG}_{\mathcal{M}\sigma}$ for some $b\in\sigma(B)$, implying that \textit{d}-separation does not hold for that ground graph.
So, $\mathbf{\bar{X}}$ and $\mathbf{\bar{Y}}$ must be \textit{d}-separated by $\mathbf{\bar{Z}}$ on $\mathit{AGG}_{\mathcal{M}B}$.
$\blacksquare$
\end{proof}

Theorem~\ref{thm:relational-d-sep} proves that \textit{d}-separation on abstract ground graphs is a sound and complete solution to identifying independence in relational models.
Theorem~\ref{thm:d-sep} also implies that the set of conditional independence facts derived from abstract ground graphs is exactly the same as the set of conditional independencies that all distributions represented by all possible ground graphs have in common.

\begin{mycorollary}
\label{cor:relational-d-sep}
$\mathbf{\bar{X}}$ and $\mathbf{\bar{Y}}$ are \textit{d}-connected given $\mathbf{\bar{Z}}$ on the abstract ground graph $\mathit{AGG}_{\mathcal{M}B}$ if and only if there exists a skeleton $\sigma\in\Sigma_\mathcal{S}$ and an item instance $b\in \sigma(B)$ such that $\mathbf{X}\vert_b$ and $\mathbf{Y}\vert_b$ are \textit{d}-connected given $\mathbf{Z}\vert_b$ in ground graph $\mathit{GG}_{\mathcal{M}\sigma}$.
\end{mycorollary}

Corollary~\ref{cor:relational-d-sep} is logically equivalent to Theorem~\ref{thm:relational-d-sep}.
While a simple restatement of the previous theorem, it is important to emphasize that relational \textit{d}-separation claims \textit{d}-connection if and only if there exists a ground graph for which $X\vert_b$ and $Y\vert_b$ are \textit{d}-connected given $\mathbf{Z}\vert_b$.
This implies that there may be some ground graphs for which $X\vert_b$ and $Y\vert_b$ are \textit{d}-separated by $\mathbf{Z}\vert_b$, but the abstract ground graph still claims \textit{d}-connection.
This may happen if the relational skeleton does not enable certain underlying relational connections.
For example, if the relational skeleton in Figure~\ref{fig:org_example_data}\subref{fig:org_skeleton} included only products that were developed by a single employee, then there would be no relationally \textit{d}-connecting path in the example in Section~\ref{sec:example}.
If this is a fundamental property of the domain (e.g., there are products developed by a single employee and products developed by multiple employees), then revising the underlying schema to include two different classes of products would yield a more accurate model implying a larger set of conditional independencies.

Additionally, we can show that relational \textit{d}-separation is equivalent to the Markov condition on relational models.

\begin{definition}[Relational Markov condition]
\label{def:rel-markov-cond}
Let $X$ be a relational variable for perspective $B\in\mathcal{E}\cup\mathcal{R}$ defined over relational schema $\mathcal{S}$.
Let $\mathit{nd}(X)$ be the non-descendant variables of $X$, and let $\mathit{pa}(X)$ be the set of parent variables of $X$.
Then, for relational model $\mathcal{M}_\Theta$, $P\big(X\ \vert\ \mathit{nd}(X), \mathit{pa}(X)\big) = P\big(X\ \vert\ \mathit{pa}(X)\big)$ if and only if $\forall x\!\in\! X\vert_b\ P\big(x\ \vert\ \mathit{nd}(x), \mathit{pa}(x)\big) = P\big(x\ \vert\ \mathit{pa}(x)\big)$ in parameterized ground graph $\mathit{GG}_{\mathcal{M}_\Theta\sigma}$ for all skeletons $\sigma\in\Sigma_{\mathcal{S}}$ and for all $b\in \sigma(B)$.
\end{definition}

In other words, a relational variable $X$ is independent of its non-descendants given its parents if and only if, for all possible parameterized ground graphs, the Markov condition holds for all instances of $X$.
For Bayesian networks, the Markov condition is equivalent to \textit{d}-separation \citep{neapolitan2004learning}.
Because parameterized ground graphs are Bayesian networks (implied by Lemma~\ref{lemma:rolled-out-dags}) and relational \textit{d}-separation on abstract ground graphs is sound and complete (by Theorem~\ref{thm:relational-d-sep}), we can conclude that relational \textit{d}-separation is equivalent to the relational Markov condition.

\section{Na\"{i}ve Relational \textit{d}-Separation Is Frequently Incorrect}
\label{sec:naive_rds}

If the rules of \textit{d}-separation for Bayesian networks were applied directly to the structure of relational models, how frequently would the conditional independence conclusions be correct?
In this section, we evaluate the necessity of our approach---relational \textit{d}-separation executed on abstract ground graphs.
We empirically compare the consistency of a na\"{i}ve approach against our sound and complete solution over a large space of synthetic causal models.
To promote a fair comparison, we restrict the space of relational models to those with underlying dependencies that could feasibly be represented and recovered by a na\"{i}ve approach.
We describe this space of models, present a reasonable approach for applying traditional \textit{d}-separation to the structure of relational models, and quantify its decrease in expressive power and accuracy.

Consider the following limited definition of relational paths, which itself limits the space of models and conditional independence queries.
A \textit{simple relational path} $P=[I_j,\dots, I_k]$ for relational schema $\mathcal{S}$ is a relational path such that $I_j\neq\cdots\neq I_k$.
The sole difference between relational paths (Definition~\ref{def:rel-path}) and simple relational paths is that no item class may appear more than once along the latter.
This yields paths drawn directly from a schema diagram.
For the example in Figure~\ref{fig:org_example_data}\subref{fig:org_schema}, $[$\textsc{Employee}, \textsc{Develops}, \textsc{Product}$]$ is simple whereas $[$\textsc{Employee}, \textsc{Develops}, \textsc{Product}, \textsc{Develops}, \textsc{Employee}$]$ is not.
 
Additionally, we define \textit{simple relational schemas} such that, for any two item classes $I_j, I_k\!\in\!\mathcal{E}\cup\mathcal{R}$, there exists at most one simple relational path between them (i.e., no cycles occur in the schema diagram).
The example in Figure~\ref{fig:org_example_data}\subref{fig:org_schema} is a simple relational schema.
The restriction to simple relational paths and schemas yields similar definitions for \textit{simple relational variables}, \textit{simple relational dependencies}, and \textit{simple relational models}.
The relational model in Figure~\ref{fig:org_example_model}\subref{fig:org_model} is simple because it includes only simple relational dependencies.

A first approximation to relational \textit{d}-separation would be to apply the rules of traditional \textit{d}-separation directly to the graphical representation of relational models.
This is equivalent to applying \textit{d}-separation to the class dependency graph $G_\mathcal{M}$ (see Definition~\ref{def:class-dependency-graph}) of relational model $\mathcal{M}$.
The class dependency graph for the model in Figure~\ref{fig:org_example_model}\subref{fig:org_model} is shown in Figure~\ref{fig:org_simple_rep}\subref{fig:org_class_dep_graph}.
Note that the class dependency graph ignores path designators on dependencies, does not include all implications of dependencies among arbitrary relational variables, and does not represent intersection variables.

Although the class dependency graph is independent of perspectives, testing any conditional independence fact \textit{requires} choosing a perspective.
All relational variables must have a common base item class; otherwise, no method can produce a single consistent, propositional table from a relational database.
For example, consider the construction of a table describing employees with columns for their salary, the success of products they develop, and the revenue of the business units they operate under.
This procedure requires joining the instances of three relational variables ($[$\textsc{Employee}$]$.\textit{Salary}, $[$\textsc{Employee}, \textsc{Develops}, \textsc{Product}$]$.\textit{Success}, and $[$\textsc{Employee}, \textsc{Develops}, \textsc{Product}, \textsc{Funds}, \textsc{Business-Unit}$]$.\textit{Revenue}) for every common base item instance, from Paul to Thomas.
See, for example, the resulting propositional table for these relational variables and an example query in Table~\ref{tab:employee-propositional} and Figure~\ref{fig:employee-query}, respectively.
An individual relational variable requires joining the item classes within its relational path, but combining a collection of relational variables requires joining on their common base item class.
Fortunately, given a perspective and the space of simple relational schemas and models, a class dependency graph is equivalent to a \textit{simple abstract ground graph}.

\begin{figure}[t]
\centering	
\subfloat[]{
\includegraphics[width=90mm]{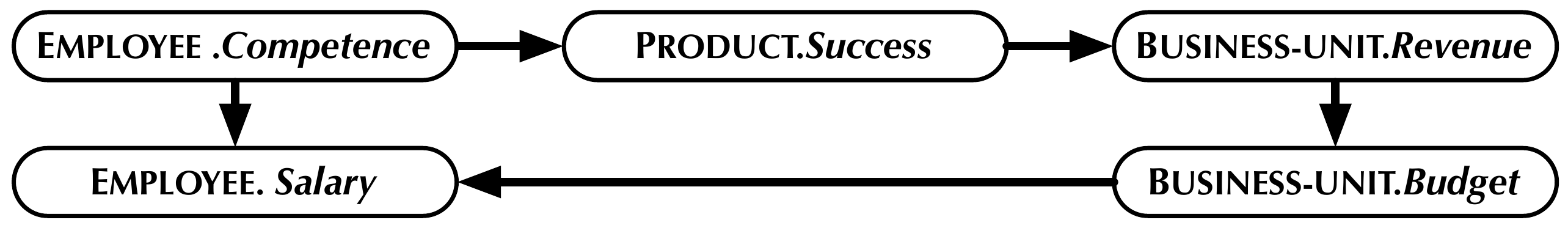}
\label{fig:org_class_dep_graph}}\\
\subfloat[]{
\includegraphics[width=150mm]{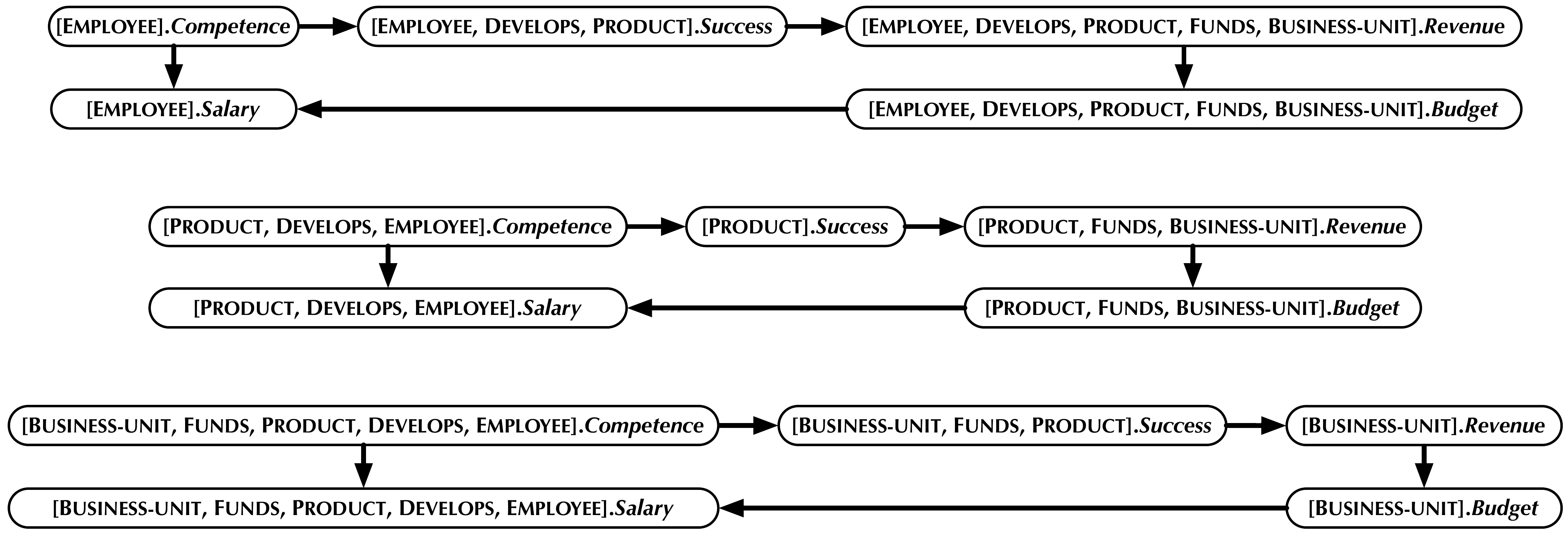}
\label{fig:org_simple_abstract_gg}}
\caption{For the model in Figure~\ref{fig:org_example_model}\protect\subref{fig:org_model}, (a) the class dependency graph and (b) three simple abstract ground graphs for the \textsc{Employee}, \textsc{Product}, and \textsc{Business-Unit} perspectives.}
\label{fig:org_simple_rep}
\end{figure}

\begin{definition}[Simple abstract ground graph]
\label{def:simple abstract-gg}
For simple relational model $\ \mathcal{M}$ = $(\mathcal{S}, \mathcal{D})$ and perspective $B\in\mathcal{E}\cup\mathcal{R}$, the \textit{simple abstract ground graph} $AGG^s_{\mathcal{M}B}$ is the directed acyclic graph $(V, E)$ that abstracts the dependencies $\mathcal{D}$ among simple relational variables.  The nodes consist of simple relational variables $\big\lbrace [B,\dots,I_j].X\ \vert\ B\neq\cdots\neq I_j\big\rbrace$, and the edges connect those nodes $\big\lbrace [B,\dots, I_k].Y\rightarrow [B,\dots, I_j].X\ \vert\ [I_j,\dots,I_k].Y\rightarrow [I_j].X\in\mathcal{D}\ \land\ [B,\dots,I_k]\in extend([B,\dots,I_j], [I_j,\dots,I_k])\ \land\ [B,\dots, I_k].Y, [B,\dots, I_j].X\in V \big\rbrace$.
\end{definition}

Simple abstract ground graphs only include nodes for simple relational variables and necessarily exclude intersection variables.
Lemma~\ref{lemma:overlap}---which characterizes the intersection between a pair of relational paths---only applies to pairs of \textit{simple} relational paths if the schema contains cycles, which is not the case for simple relational schemas by definition.
As a result, the simple abstract ground graph for a given schema and model contains the same number of nodes and edges, regardless of perspective; the nodes simply have path designators redefined from the given perspective.
Figure~\ref{fig:org_simple_rep}\subref{fig:org_simple_abstract_gg} shows three simple abstract ground graphs from distinct perspectives for the model in Figure~\ref{fig:org_example_model}\subref{fig:org_model}.
As noted above, simple abstract ground graphs are qualitatively the same as the class dependency graph, but they enable answering relational \textit{d}-separation queries, which requires a common perspective in order to be semantically meaningful.

The na\"{i}ve approach to relational \textit{d}-separation derives conditional independence facts from simple abstract ground graphs (Definition~\ref{def:simple abstract-gg}).
The sound and complete approach described in this paper applies \textit{d}-separation---with input variable sets augmented by their intersection variables---to ``regular" abstract ground graphs, as described by Definition~\ref{def:abstract-gg}.
Clearly, if \textit{d}-separation on a simple abstract ground graph claims that $\mathbf{X}$ is \textit{d}-separated from $\mathbf{Y}$ given $\mathbf{Z}$, then \textit{d}-separation on the regular abstract ground graph yields the same conclusion if and only if there are no \textit{d}-connecting paths in the regular abstract ground graph.
Either $\mathbf{X}$ and $\mathbf{Y}$ can be \textit{d}-separated by a set of simple relational variables $\mathbf{Z}$, or they require non-simple relational variables---those involving relational paths with repeated item classes.\footnote{The theoretical conditions under which equivalence occurs are sufficiently complex that they provide little utility as they essentially require reconstructing the regular abstract ground graph and checking a potentially exponential number of dependency paths.}

To evaluate the necessity of regular abstract ground graphs (i.e., the additional paths involving non-simple relational variables and intersection variables), we compared the frequency of equivalence between the conclusions of \textit{d}-separation on simple and regular abstract ground graphs.
The two approaches are only equivalent if a minimal separating set consists entirely of simple relational variables.\footnotemark
\footnotetext{If $\mathbf{X}$ and $\mathbf{Y}$ are \textit{d}-separated given $\mathbf{Z}$, then $\mathbf{Z}$ is a \textit{separating set} for $\mathbf{X}$ and $\mathbf{Y}$.  A separating set $\mathbf{Z}$ is \textit{minimal} if there is no proper subset of $\mathbf{Z}$ that is also a separating set.}

Thus, for an arbitrary pair of relational variables $X$ and $Y$ with a common perspective, we test the following on regular abstract ground graphs:

\begin{enumerate}
\item Is either $X$ or $Y$ a non-simple relational variable?
\item Are $X$ and $Y$ marginally independent?
\item Does a minimal separating set $\mathbf{Z}$ \textit{d}-separate $X$ and $Y\!$, where $\mathbf{Z}$ consists solely of simple relational variables?
\item Is there any separating set $\mathbf{Z}$ that \textit{d}-separates $X$ and $Y$?
\end{enumerate}

If the answer to (1) is yes, then the na\"{i}ve approach cannot apply since either $X$ or $Y$ is undefined for the simple abstract ground graph.
If the answer to (2) is yes, then there is equivalence; this is a trivial case because there are no \textit{d}-connecting paths for $\mathbf{Z} = \emptyset$.
If the answer to (3) is yes, then there is a minimal separating set $\mathbf{Z}$ consisting of only simple relational variables.
In this case, the simple abstract ground graph is sufficient, and we also have equivalence.
If the answer to (4) is no, then no separating set $\mathbf{Z}$, simple or otherwise, renders $X$ and $Y$ conditionally independent.

We randomly generated simple relational schemas and models for 100 trials for each setting using the following parameters:

\begin{itemize}
\item Number of entity classes, ranging from 1 to 4.
\item Number of relationship classes, fixed at one less than the number of entities, ensuring simple, connected relational schemas.  Relationship cardinalities are chosen uniformly at random.
\item Number of attributes for each entity and relationship class, randomly drawn from a shifted Poisson distribution with $\lambda=1.0$ ($\sim \mathit{Pois}(1.0)+1$).
\item Number of dependencies in the model, ranging from 1 to 10.
\end{itemize}

Then, for all pairs of relational variables with a common perspective limited by a hop threshold of $h=4$, we ran the aforementioned tests against the regular abstract ground graph, limiting its relational variables by a hop threshold of $h=8$ (a sufficient hop threshold for soundness and completeness---see Appendix~\ref{sec:hop_thresholds}).

	\begin{figure}[t]
	\centering
	\includegraphics[width=150mm]{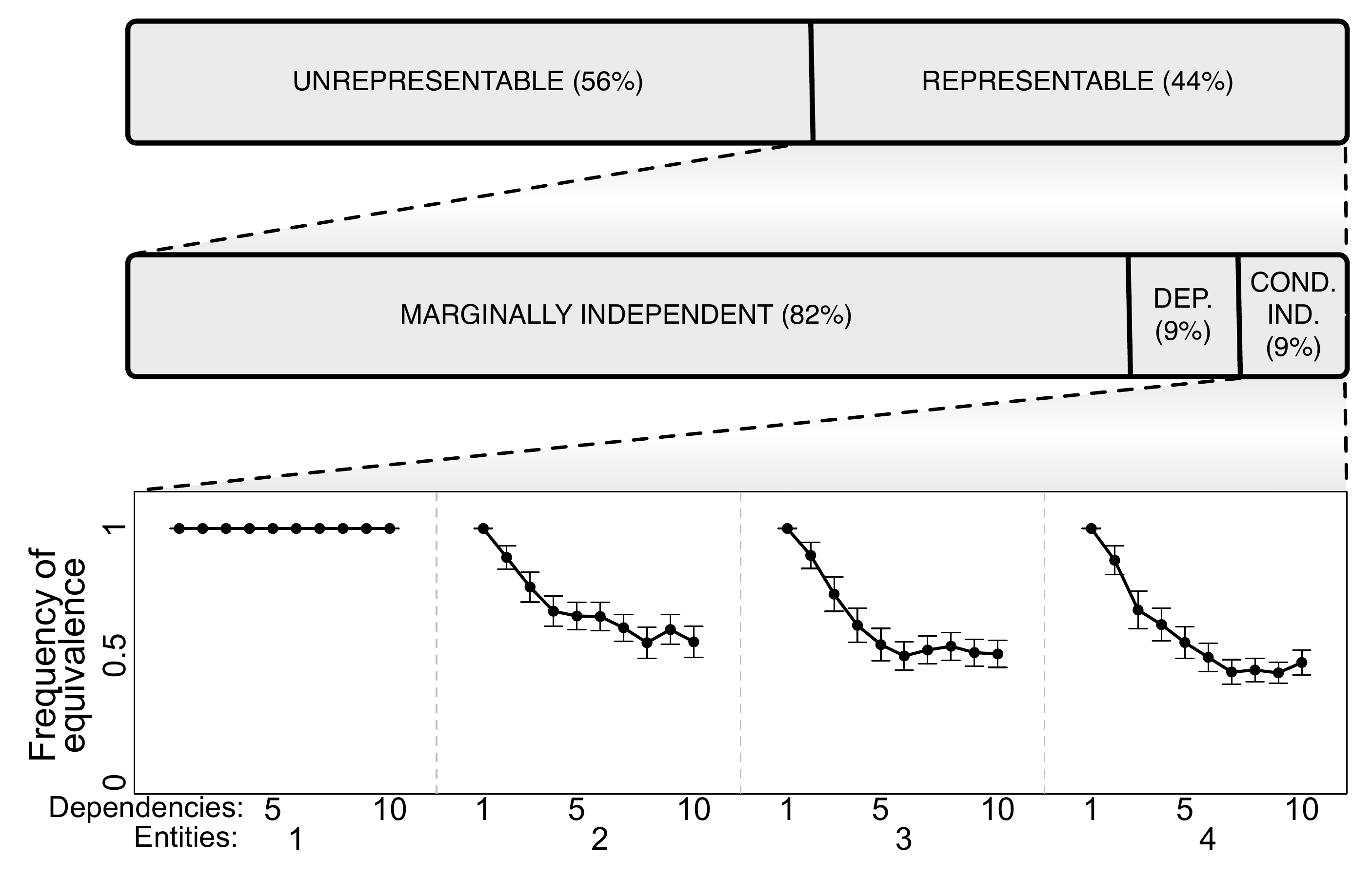}
	\caption{The majority (56\%) of generated relational \textit{d}-separation queries are not representable with the na\"{i}ve approach.  Of the 44\% that are representable (involving only simple relational variables), 82\% are marginally independent and 9\% are dependent.  Pairs of relational variables in the remaining 9\% are conditionally independent given a non-empty separating set ($X\Perp Y\ \vert\ \mathbf{Z}$, where $\mathbf{Z}\neq\emptyset$).  We test whether the \textit{conditioning set} consists solely of simple relational variables.  If so, then the na\"{i}ve approach to relational \textit{d}-separation is equivalent to \textit{d}-separation on fully specified abstract ground graphs.  This graph plots the frequency of equivalence across schemas with increasing numbers of entity classes (1--4) for varying numbers of dependencies (1--10).  For schemas with more than one entity class, the frequency of equivalence decreases as the number of dependencies increases. Shown with 95\% confidence intervals.}
 	\label{fig:rds-equivalence}
	\end{figure}

This procedure generated a total of almost 3.6 million pairs of relational variables to test.
Approximately 56\% included a non-simple relational variable; the na\"{i}ve approach cannot be used to derive a conditional independence statement in these cases, requiring the full abstract ground graph in order to represent these variables.
Of the remaining 44\% (roughly 1.6 million), 82\% were marginally independent, and 9\% were not conditionally independent given any conditioning set $\mathbf{Z}$.
Then, of the remaining 9\% (roughly 145,000), we can test the frequency of equivalence for conditional independence facts with non-empty separating sets---the proportion of cases for which only simple relational variables are required in a minimal separating set $\mathbf{Z}$.

Figure~\ref{fig:rds-equivalence} shows this frequency for schemas of increasing numbers of entity classes (1--4) for varying numbers of dependencies in the causal model (1--10).
Since relational schemas with a single entity class and no relationships describe propositional data, the simple abstract ground graph is equivalent to the full abstract ground graph, which is also equivalent to the model itself.
In this case, the na\"{i}ve approach is always equivalent because it is exactly \textit{d}-separation on Bayesian networks.
For truly relational schemas (with more than one entity class and at least one relationship class), the frequency of equivalence decreases as the number of dependencies in the model increases.
Additionally, the frequency of equivalence decreases more as the number of entities in the schema increases.
For example, the frequency of equivalence for nine dependencies is 60.3\% for two entities, 51.2\% for three entities, and 43.2\% for four entities.

We also learned statistical models that predict the number of equivalent and non-equivalent statements in order to identify key factors that affect the frequency of equivalence. 
We found that the number of dependencies and size of the relational model (regulated by the number of entities and \textsc{many} cardinalities) dictate the equivalence.
As a relational model deviates from a Bayesian network, we should expect more \textit{d}-connecting paths in the regular but not simple abstract ground graph.
This property also depends on the specific combination of dependencies in the model.
Appendix~\ref{sec:appendix_condindequiv_prediction} presents details of this analysis.

This experiment suggests that applying traditional \textit{d}-separation directly to a relational model structure will frequently derive incorrect conditional independence facts.
Additionally, there is a large class of conditional independence queries involving non-simple variables for which such an approach is undefined.
These results indicate that fully specifying abstract ground graphs and applying \textit{d}-separation augmented with intersection variables (as described in Section~\ref{sec:rds}) is critical for accurately deriving most conditional independence facts from relational models.
%This experiment also tested \textit{theoretical} differences without respect to an underlying distribution; the experiment in Section~\ref{sec:empirical-validity} compares how closely \textit{d}-separation on abstract ground graphs match the conclusions of conventional statistical tests on real data.

\section{Experiments}
\label{sec:experiments}

To complement the theoretical results, we present three experiments on synthetic data.
The primary goal of these empirical results is to demonstrate the feasibility of applying relational \textit{d}-separation in practice.
The experiment in Section~\ref{sec:agg-size} describes the factors that influence the size of abstract ground graphs and thus the computational complexity of relational \textit{d}-separation.
The experiment in Section~\ref{sec:sepset-size} evaluates the growth rate of separating sets produced by relational \textit{d}-separation as abstract ground graphs become large.
The results indicate that minimal separating sets grow much more slowly than abstract ground graphs.
The experiment in Section~\ref{sec:empirical-validity} tests how the expectations of the relational \textit{d}-separation theory match statistical conclusions on simulated data.
As expected from the proofs of correctness in Section~\ref{sec:proof-of-correctness}, the results indicate a close match, aside from Type I errors and certain biases of conventional statistical tests on relational data.

\subsection{Abstract Ground Graph Size}
\label{sec:agg-size}

Relational \textit{d}-separation is executed on abstract ground graphs.
Consequently, it is important to quantify the size of abstract ground graphs and identify which factors influence their size.
We randomly generated relational schemas and models for 1,000 trials of each setting using the following parameters:

\begin{itemize}
\item Number of entity classes, ranging from 1 to 4.
\item Number of relationship classes, ranging from 0 to 4.  The schema is guaranteed to be fully connected and includes at most a single relationship between a pair of entities.  Relationship cardinalities are selected uniformly at random.
\item Number of attributes for each entity and relationship class, randomly drawn from a shifted Poisson distribution with $\lambda=1.0$ ($\sim \mathit{Pois}(1.0)+1$).
\item Number of dependencies in the model, ranging from 1 to 15.
\end{itemize}

This procedure generated a total of 450,000 abstract ground graphs, which included every perspective (all entity and relationship classes) for each experimental combination.
We measure size as the number of nodes and edges in a given abstract ground graph.
Figure~\ref{fig:agg-sizes}\subref{fig:agg-sizes-cards} depicts how the size of abstract ground graphs varies with respect to the number of \textsc{many} cardinalities in the schema (fixed for models with 10 dependencies), and Figure~\ref{fig:agg-sizes}\subref{fig:agg-sizes-deps} shows how it varies with respect to the number of dependencies in the model.
Recall that for a single entity, abstract ground graphs are equivalent to Bayesian networks.

	\begin{figure}[t]
	\centering
	\subfloat[]{
	\includegraphics[width=70mm]{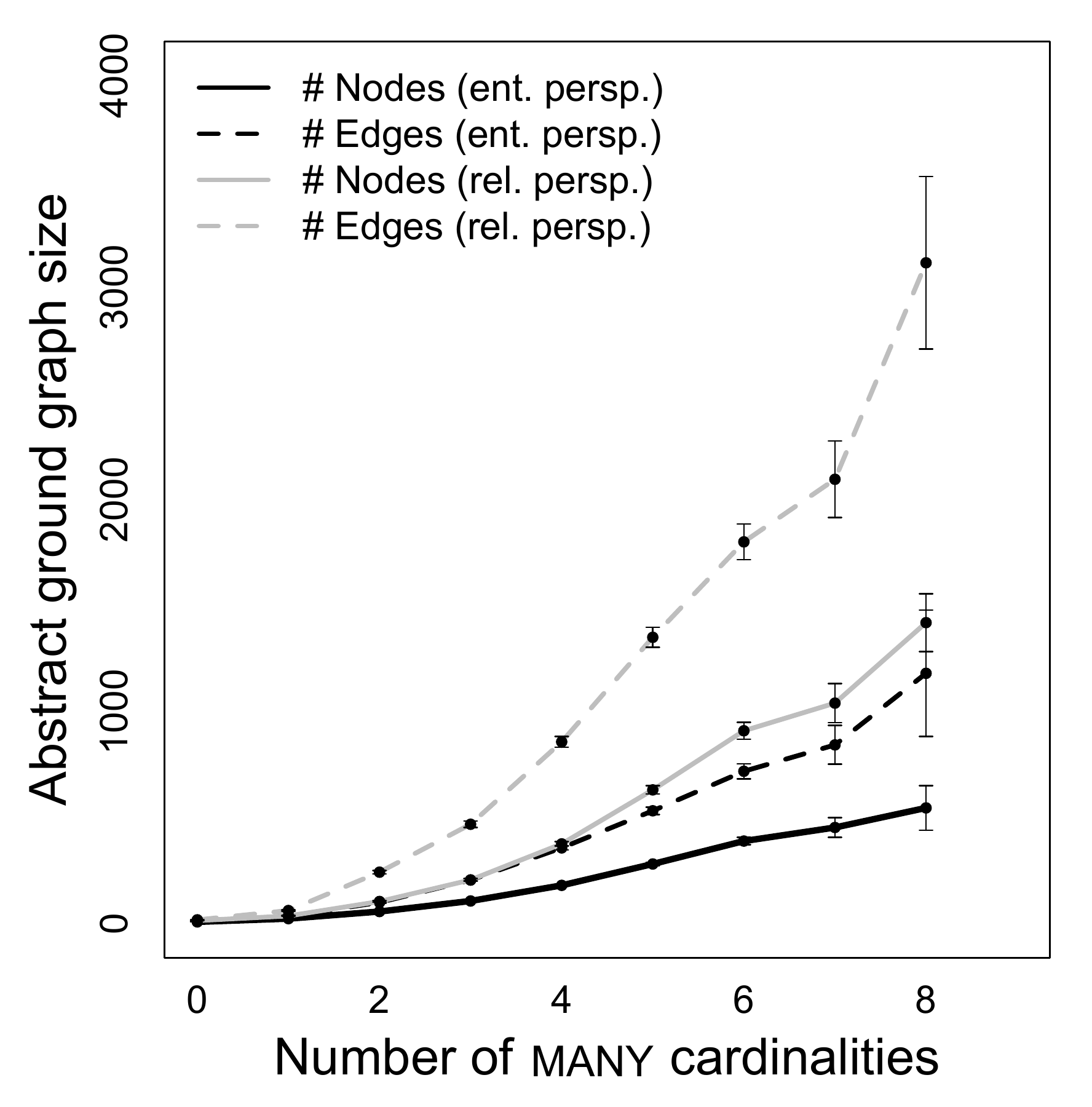}
	\label{fig:agg-sizes-cards}}
	\subfloat[]{
	\includegraphics[width=70mm]{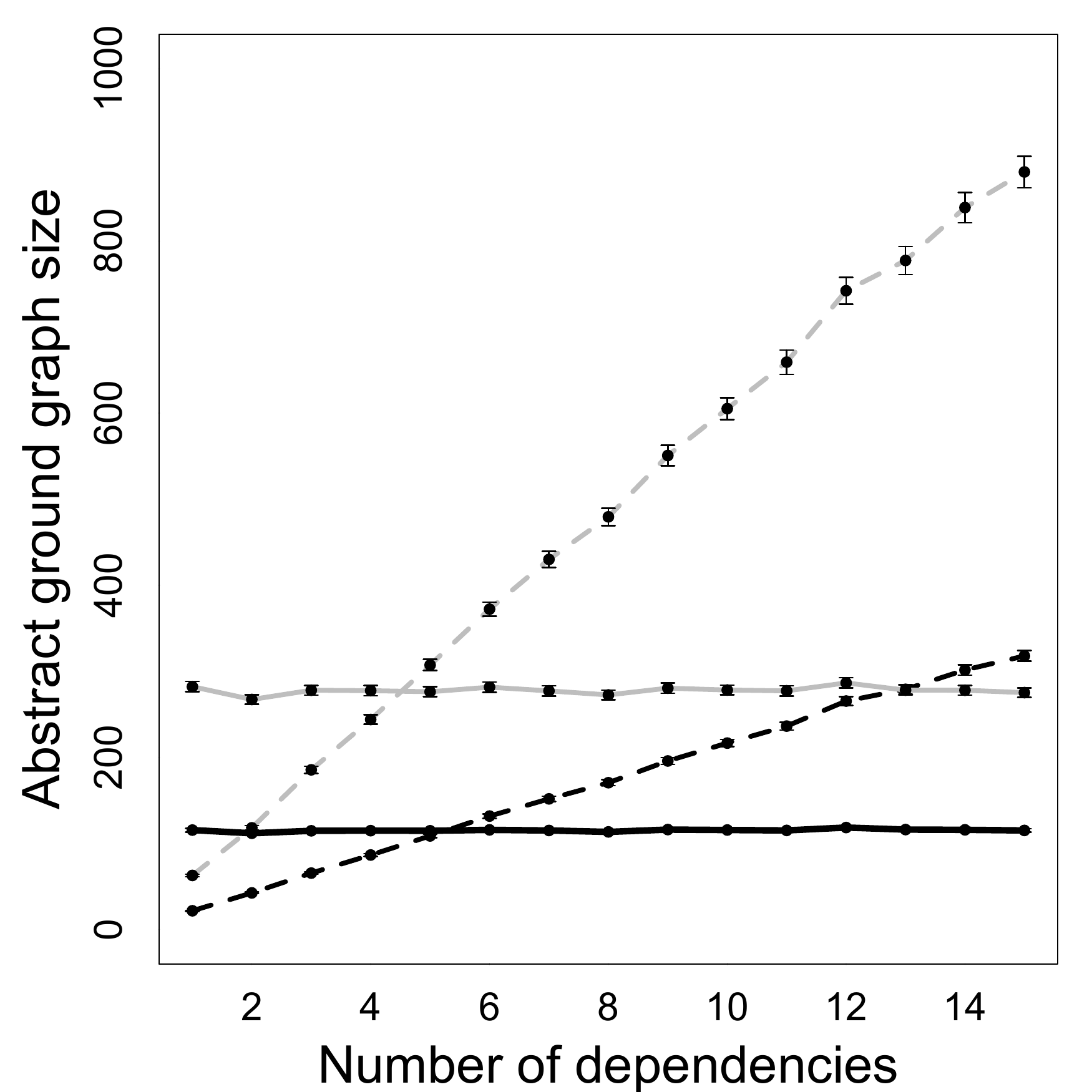}
	\label{fig:agg-sizes-deps}}
	\caption{Variation of abstract ground graph size as (a) the number of \textsc{many} cardinalities in the schema increases (dependencies fixed at 10) and (b) the number of dependencies increases.  Shown with 95\% confidence intervals.}
	\label{fig:agg-sizes}
	\end{figure}

To determine the most influential factors of abstract ground graph size, we ran log-linear regression using independent variables that describe only the schema and model.
Detailed results are provided in Appendix~\ref{sec:appendix_aggsize_prediction}.
This analysis indicates that (1) as the number of entities, relationships, attributes, and \textsc{many} cardinalities increases, the number of nodes and edges grows at an exponential rate. 
(2) As the number of dependencies in the model increases, the number of edges increases linearly, but the number of nodes remains invariant. 
And (3) abstract ground graphs for relationship perspectives are larger than entity perspectives because more relational variables can be defined.

\subsection{Minimal Separating Set Size}
\label{sec:sepset-size}

Because abstract ground graphs can become large, one might expect that separating sets could also grow to impractical sizes.
Fortunately, relational \textit{d}-separation produces minimal separating sets that are empirically observed to be small.
We ran 1,000 trials of each setting using the following parameters:
\begin{itemize}
\item Number of entity classes, ranging from 1 to 4.
\item Number of relationship classes, fixed at one less than the number of entities.  Relationship cardinalities are selected uniformly at random.
\item Total number of attributes across entity and relationship classes, fixed at 10.
\item Number of dependencies in the model, ranging from 1 to 10.
\end{itemize}
For each relational model, we identified a single minimal separating set for up to 100 randomly chosen pairs of conditionally independent relational variables.
This procedure generated almost 2.5 million pairs of variables.

	\begin{figure}[t]
	\centering
	\includegraphics[width=150mm]{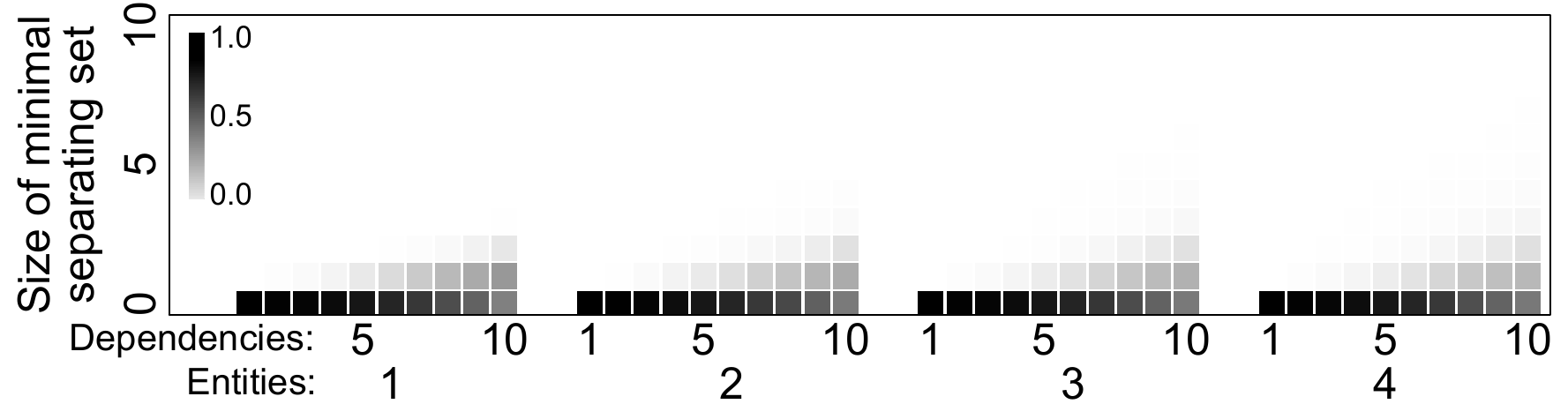}
	\caption{Minimal separating sets have reasonable sizes, growing only with the size of the schema and the model density.  In this experiment, $99.9\%$ of variable pairs have a minimal separating set with five or fewer variables.}
 	\label{fig:minimal-sep-sets}
	\end{figure}

To identify a minimal separating set between relational variables $X$ and $Y$\!, we modified Algorithm~4 devised by \citet{tianFMDS1998} by starting with all parents of $\bar{X}$ and $\bar{Y}$\!, the variables augmented with the intersection variables they subsume in the abstract ground graph.
While the discovered separating sets are \textit{minimal}, they are not necessarily of \textit{minimum} size because of the greedy process for removing conditioning variables from the separating set.
Figure~\ref{fig:minimal-sep-sets} shows the frequency of separating set size as both the number of entities and dependencies vary.	
In summation, roughly $83\%$ of the pairs are marginally independent (having empty separating sets), $13\%$ have separating sets of size one, and less than $0.1\%$ have separating sets with more than five variables.
The experimental results indicate that separating set size is strongly influenced by model density, primarily because the number of potential \textit{d}-connecting paths increases as the number of dependencies increases.

\subsection{Empirical Validity}
\label{sec:empirical-validity}

As a practical demonstration, we examined how the expectations of the relational \textit{d}-separation theory match the results of statistical tests on actual data.
We use a standard procedure for empirically measuring internal validity of algorithms.
In this case, we (1) randomly generate a relational schema, (2) randomly generate a relational model structure for that schema, (3) parameterize the model structure, (4) generate synthetic data according to the model structure and parameters, (5) randomly choose relational \textit{d}-separation queries according to the known ground-truth model, and (6) compare the model theory (i.e., the \textit{d}-separation conclusions) against corresponding statistical tests of conditional independence.

For steps (1) and (2), we randomly generated a relational schema $\mathcal{S}$ and relational model structure $\mathcal{M}$ for $\mathcal{S}$ for 100 trials using the following settings:
\begin{itemize}
\item Number of entity classes, ranging from 1 to 4.
\item Number of relationship classes, fixed at one less than the number of entities.  Relationship cardinalities are selected uniformly at random.
\item Number of attributes for each entity and relationship class, randomly drawn from a shifted Poisson distribution with $\lambda=1.0$ ($\sim \mathit{Pois}(1.0)+1$).
\item Number of dependencies in the model, fixed at 10.
\end{itemize}
\noindent Dependencies were selected greedily, choosing each one uniformly at random, subject to a maximum of 3 parent relational variables for each attribute $[I_j].X$ and enforcing acyclicity of the model structure.

For step (3), we parameterized relational models using simple additive linear equations with independent, normally distributed error and the average aggregate for relational variable instances.
For each attribute $[I_j].X$, we assign a conditional probability distribution
\begin{center}
$\displaystyle\sum_{[I_j,\dots,I_k].Y\in \mathit{parents}([I_j].X)}\ \big(\beta \cdot \mathit{avg}([I_j,\dots,I_k].Y)\big) + 0.1\epsilon$
\end{center}
\noindent if $[I_j].X$ has parents, where 
\vspace{-4mm}
\begin{center}
$\beta = \dfrac{0.9}{\vert \mathit{parents}([I_j].X)\vert}$
\end{center}
\noindent to provide equal contribution for each direct cause and $\epsilon \sim\! N(0,1)$ (error drawn from a standard normal distribution).
If $[I_j].X$ has no parents, its value is just drawn from $\epsilon$.

For step (4), we first generated a relational skeleton $\sigma$ (because the current model space assumes that attributes do not cause entity or relationship existence) and then populated each attribute value by drawing from its corresponding conditional distribution.
Each entity class is initialized to 1,000 instances.
Relationship instances were constructed via a latent homophily process, similar to the method used by \citet{shalizi2011homophily}.
Each entity instance received a single latent variable, marginally independent from all other variables.
The probability of any relationship instance was drawn from 
\vspace{-2mm}
\begin{center}
$\dfrac{e^{-\alpha d}}{1 + e^{-\alpha d}}\, ,$
\end{center}
\noindent the inverse logistic function, where $d=\vert L_{E_1} - L_{E_2} \vert$, the difference between the latent variables on the two entities, and $\alpha=10$, set as the decay parameter.
We also scaled the probabilities in order to produce an expected degree of five for each entity instance when the cardinality of the relationship is \textsc{many}.
Since the latent variables are marginally independent of all others, they are safely omitted from abstract ground graphs; their sole purpose is to generate relational skeletons that provide a greater probability of non-empty intersection variables as opposed to a random underlying link structure.
We generated 100 independent relational skeletons and attribute values (i.e., 100 instantiated relational databases) for each schema and model.

	\begin{figure}[t]
	\centering
	\includegraphics[width=145mm]{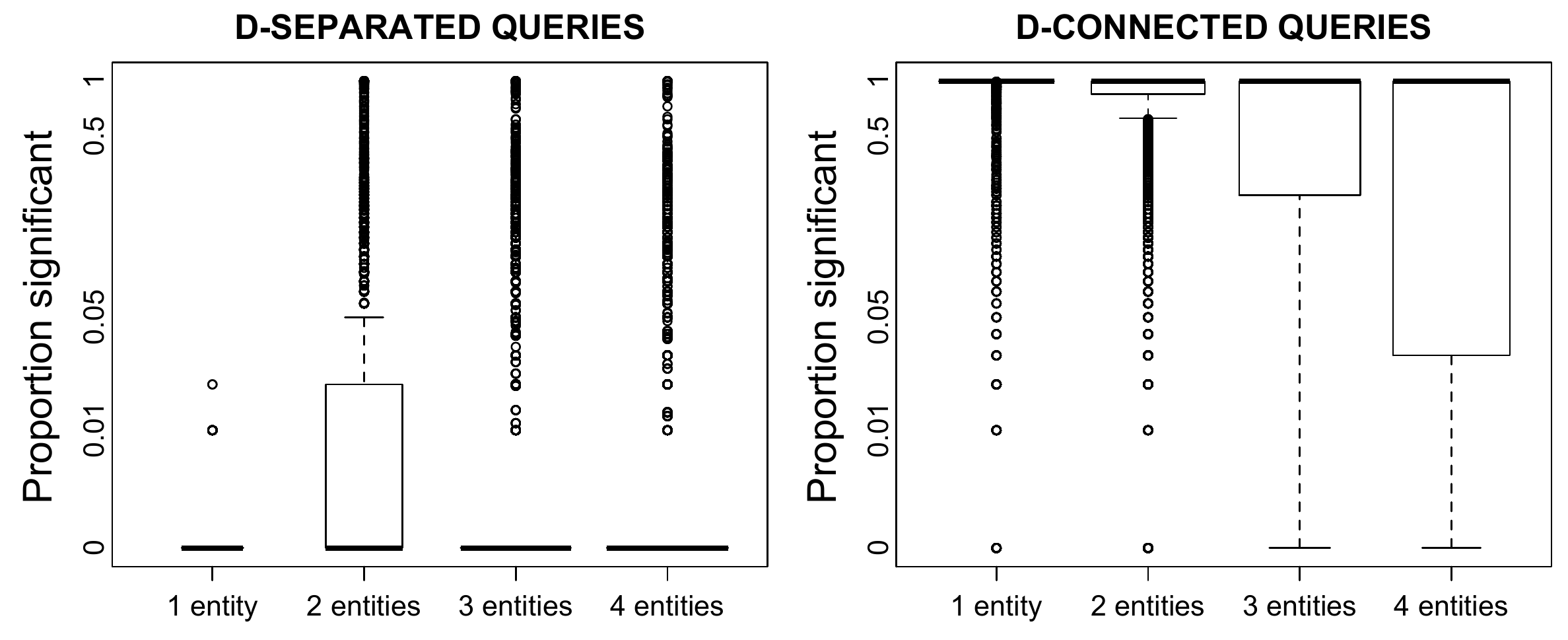}
	\caption{The proportion of significant trials for statistical tests of conditional independence on actual data. (Left) Evaluating queries that the model claims to be \textit{d}-separated produces low rates of significant effects. (Right) Queries that the model claims are \textit{d}-connected produce high rates of significant effects.  Note that the generative process yields denser models for 2 entity classes since the number of dependencies is fixed at 10.}
 	\label{fig:syntheticDataResults-propsig}
	\end{figure}

Step (5) randomly chooses up to 100 true and false relational \textit{d}-separation queries for a given model.\footnote{Depending on the properties of the schema and model, it may not always be feasible to identify 100 true or false \textit{d}-separation statements.}
Since we have the ground-truth model, we can evaluate with our approach (abstract ground graphs and relational \textit{d}-separation) whether these queries are true (\textit{d}-separated) or false (\textit{d}-connected).
Each query is of the form $X\Perp Y\ \vert\ \mathbf{Z}$ such that $X$ and $Y$ are single relational variables, $\mathbf{Z}$ is a set of relational variables, $Y$ has a singleton relational path (e.g., $[I_k].Y$), and all variables are from a common perspective.
These queries correspond to testing potential direct causal dependencies in the relational model, similar to the tests used by constraint-based methods for learning relational models, such as RPC \citep{maier2010rpc} and RCD \citep{maier2013rcd}.

Finally, step (6) tests for conditional independence for all such $\langle X, Y, \mathbf{Z}\rangle$ \textit{d}-separation queries using linear regression (because the models were parameterized linearly) for each of the 100 data instantiations.
Specifically, we tested the $t$-statistic for the coefficient of $\mathit{avg}(X)$ in the equation $Y = \beta_0 + \beta_1\cdot \mathit{avg}(X) + \sum_{Z_i\in\mathbf{Z}} \beta_i\cdot \mathit{avg}(Z_i)$.
For each query, we recorded two measurements:
\begin{itemize}
\item The average strength of effect, measured as squared partial correlation---the proportion of remaining variance of $Y$ explained by $X$ after conditioning on $\mathbf{Z}$
\item The proportion of trials for which each query was deemed significant at $\alpha=0.01$ adjusted using Bonferroni correction with the number of queries per trial
\end{itemize}

	\begin{figure}[t]
	\centering
	\includegraphics[width=145mm]{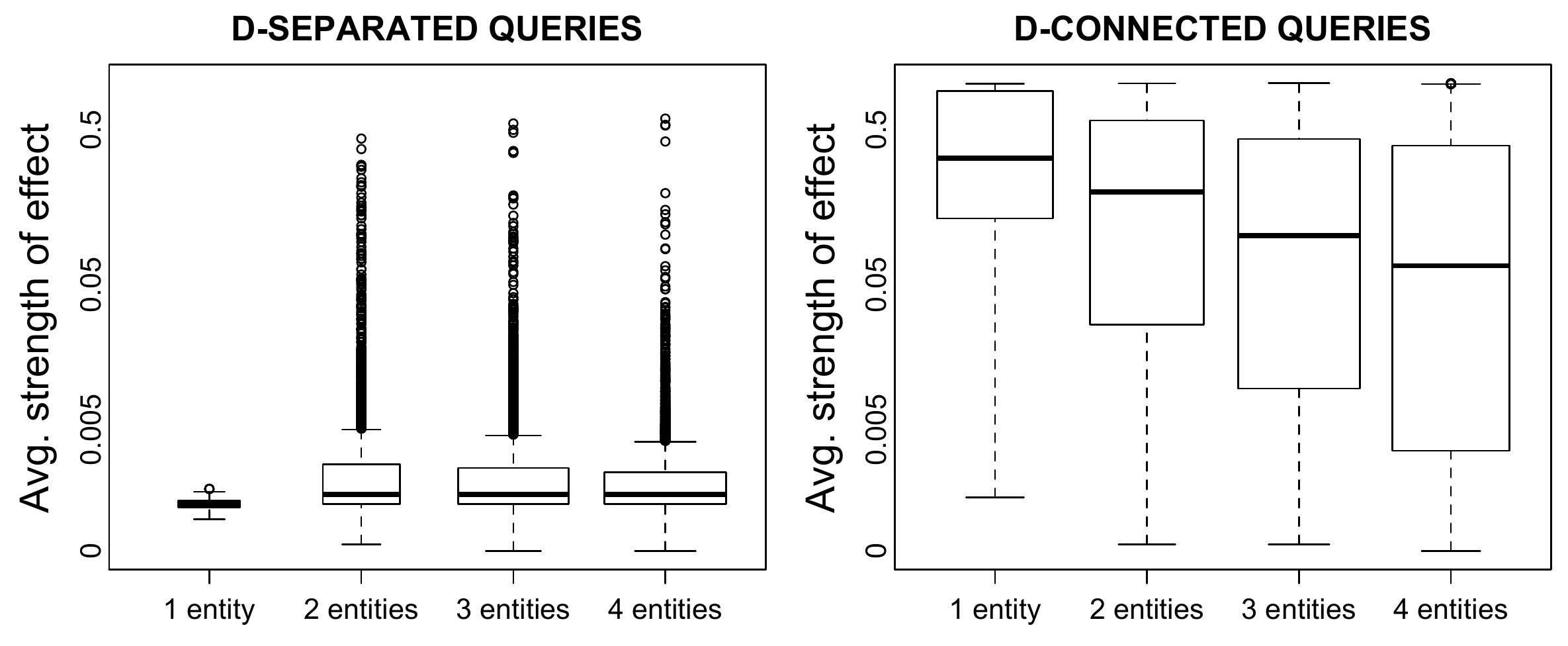}
	\caption{The average strength of effect of each query (measured as squared partial correlation) on actual data. (Left) Evaluating queries that the model claims to be \textit{d}-separated or conditionally independent produces low average effect sizes.  (Right) Queries that the model claims are \textit{d}-connected or dependent produce high average effect sizes.}
 	\label{fig:syntheticDataResults-soe}
	\end{figure}

\noindent Figure~\ref{fig:syntheticDataResults-propsig} shows the distribution of the proportion of significant trials for both true (left) and false queries (right) for varying numbers of entities.
Figure~\ref{fig:syntheticDataResults-soe} shows the corresponding average strength of effects for true (left) and false (right) queries.
The graph uses a standard box-and-whisker plot with values greater or less than 1.5 times the inner quartile range---the difference between the upper and lower quartiles---marked as outliers.

In the vast majority of cases, relational \textit{d}-separation is consistent with tests on actual data (i.e., most \textit{d}-separated queries have low effect sizes and are rarely deemed significant, whereas most \textit{d}-connected queries have high effect sizes and are mostly deemed significant).
For approximately 23,000 true queries, 14.9\% are significant in more than one trial, but most are insubstantive, with only 2.2\% having an average effect size greater than $0.01$.
There are three potential reasons why a \textit{d}-separation in theory may appear to be \textit{d}-connected in practice: (1) Type~I error; (2) high power given a large sample size; or (3) bias.
We have discovered that a small number of cases exhibit an interaction between aggregation and relational structure (i.e., degree or the cardinality of relational variable instances).
This interaction violates the identically distributed assumption of data instances, which produces a biased estimate of effect size for simple linear regression.
Linear regression does not account for these interaction effects, suggesting the need for more accurate statistical tests of conditional independence for relational data.

\section{Model Assumptions and Related Work}
\label{sec:rel_model_assumptions}

The class of relational models considered in Section~\ref{sec:rel_concepts}, while strictly more expressive than Bayesian networks, has limitations in its current formalization.
In this section, we highlight these assumptions and discuss how related and future work could address them.

\textit{Self-relationships}: 
Self-relationships are relationship classes that involve the same entity class more than once.
Relational schemas, as defined in Definition~\ref{def:rel-schema}, can express these types of relationships.
Only the definition of relational paths---which govern the space of variables and dependencies---requires unique entity class names within [E,R,E] triples (see condition (2) of Definition~\ref{def:rel-path}).
However, a common procedure in entity-relationship modeling is to map entity names to unique \textit{role indicators} within the context of a self-relationship, such as manager/subordinate, friend1/friend2, or citing-paper/cited-paper \citep{ramakrishnan2002dbms}.
This approach does not duplicate entity instances in the skeleton or ground graph; it only modifies their reference names within the relational path, requiring extended semantics for terminal sets.
Incorporating self-relationships is a straightforward extension, but for simplicity, we omit this additional layer of complexity.

\textit{Relational autocorrelation}:
In contrast to self-relationships, relational autocorrelation is a statistical dependency among the values of the same attribute class frequently found in relational data sets \citep{jensen-neville-icml02}.
Various models and learning algorithms have been developed to capture these types of dependencies, such as RDNs \citep{neville-jensen-jmlr07}, PBNs with an extended normal form \citep{schulte2012recursive}, and PRMs with dependencies that follow guaranteed acyclic relationships \citep{getoor:prm-ch-srl-book07}.
Our formalism, and equivalently PRMs (without guaranteed acyclic relationships), can represent a class of models for apparent autocorrelation.
Any relational dependency that yields a common cause for grounded variables of the same attribute class---essentially any dependency that crosses a \textsc{many} cardinality---produces relational autocorrelation.
The only autocorrelations not accounted for involve latent causes or those produced by temporal processes (e.g., feedback).

\textit{Context-specific independence}:
Context-specific independence (CSI) introduces independence of some variable and its parents, depending on the values of other variables.
This can be achieved within the specification of conditional probability distributions as if-then-else statements of logical conditions, such as in DAPER models \citep{heckerman-etal-introsrl07} or RPMs \citep{russell-norvig-aima}, encoded as regularities in conditional probability tables \citep{boutilier1996csi}, or with the recent graphical convention of gates \citep{minka2009gates}.
However, this introduces a notion of independence that cannot be inferred from model structure via traditional \textit{d}-separation.
In fact, \citet{boutilier1996csi} define an analogous approach based on \textit{d}-separation of a manipulated Bayesian network through deletion of vacuous dependencies given some context.
\citet{winn2012causality} extends the rules of \textit{d}-separation to reason over the additional paths and their collective state introduced by gates.
An alternative and more general approach to encoding CSIs is to develop an \textit{ontology} for which (in)dependencies hold depending on the type of entity or relationship.
PRMs with class hierarchies allow a hierarchy of entity types where the dependency structure can vary depending on the type \citep{getoor2000prmch}.
Rules of inheritance derived from object-oriented programming are used to define a coherent joint probability distribution.
This aligns with our formalism, as relational schemas can be viewed as an ontology defined at a particular level.
However, the semantics of \textit{d}-separation under inheritance has not been developed and is a profitable direction of future research.

\textit{Causes of entity and relationship existence}:
Without a generative model of relational skeletons, the relational models are not truly generative as the skeleton must be generated prior to the attributes.
However, the same issue occurs for Bayesian networks: Relational skeletons consist of disconnected entity instances, but the model does not specify how many instances to create.
There are relational models that attempt to learn and represent models with unknown numbers of entity instances, such as \textsc{Blog} \citep{milch-etal-ijcai05}, or uncertain relationship instances, such as PRMs with existence uncertainty \citep{getoor2002link}.
However, reasoning about the connection between conditional independence and existence is an open problem.
For relationship existence, selection bias (conditioning) occurs when testing marginal dependence between variables across a particular relationship \citep{maier2010rpc}.
For entity existence, some researchers argue that existence cannot be represented as a variable or predicate \citep{poole2007logical}, while others represent them as predicates \citep{laskey2008mebn}.
Therefore, we currently choose simple processes for generating skeletons, allowing us to focus on relational models of attributes and leaving structural causes and effects as future work.

\textit{Causal sufficiency}:
The relational models we consider assume that all common causes of observed variables are also observed and included in the model---an assumption commonly referred to as causal sufficiency.
Many researchers have developed methods for learning and inference by explicitly modeling unobserved variables---typically termed latent variable models \citep{bishop-lgm99}---or inferring the presence of latent entity classes---for example, latent group models \citep{neville-jensen-icdm05}.
However, only ancestral graphs and acyclic directed mixed graphs (ADMGs) do so in order to preserve an underlying conditional independence structure \citep{richardson2002ancestral, richardson2009admg}.
These models are paired with the theory of \textit{m}-separation, which is a generalization of \textit{d}-separation for Bayesian networks.
The generalization of ancestral graphs or ADMGs to relational models requires extensive theoretical exploration; therefore, we leave this as an important direction for future work.
Given that a primary motivation for \textit{d}-separation is to support constraint-based causal discovery, any relational extension to algorithms that learn causal models without assuming causal sufficiency, such as FCI \citep{spirtes-etal-uai95, zhang2008completeness}, its variants \citep{claassen-heskes-uai2011, colombo2012rfci}, and BCCD \citep{claassen-heskes-uai2012}, would require such an extension to \textit{m}-separation.

\textit{Temporal and cyclic models}:
Currently, the relational model is assumed to be acyclic (with respect to the class dependency graph), and consequently, atemporal.
Model-level cycles typically result from temporal processes for which grounding across time would yield an acyclic ground graph, such as in dynamic Bayesian networks \citep{dean1989model, murphy2002dynamic}.
However, cycles can also be due to temporal processes where the interaction occurs at a faster rate than measurement.
As a result, there has been considerable attention devoted to models that explicitly encode cyclic dependencies, such as the work by \citet{spirtes-etal-uai95-dcg}, \citet{pearl-decther-uai96}, \citet{richardson1996thesis}, \citet{dash-aistats2005}, \citet{schmidt-murphy-uai2009}, and \citet{hyttinen-et-al-jmlr2012}.
Our formalism currently prohibits any relational dependency that has a common attribute class for the cause and effect, regardless of the relational path constraint.
Relaxing this assumption would require either explicitly modeling temporal dynamics or enabling feedback loops.
We reserve temporal dynamics and feedback as another important avenue for future research.

Despite these assumptions, our current work extends the notion of \textit{d}-separation to a much more expressive class of models than Bayesian networks.
This work is a first step toward deriving conditional independencies from expressive classes of models.
Incorporating existence, ontologies, temporal dynamics and feedback, and latent variables into our model is important future work, especially in the context of representing and learning causal models of realistic domains.

\section{Discussion}
\label{sec:discussion}

In this paper, we extend the theory of \textit{d}-separation to graphical models of relational data.
We present the \textit{abstract ground graph}, a new representation that is sound and complete in its abstraction of dependencies across all possible ground graphs of a given relational model.
We formally define relational \textit{d}-separation and offer a sound, complete, and computationally efficient approach to deriving conditional independence facts from relational models by exploiting their abstract ground graphs.
We also show that relational \textit{d}-separation is equivalent to the Markov condition for relational models.
We provide an empirical analysis of relational \textit{d}-separation on synthetic data, demonstrating a close correspondence between the theory and statistical results in practice.
Finally, we evaluate how frequently the additional complexity of abstract ground graphs proves necessary for accurately deriving conditional independence facts.
	
The results of this paper imply potential flaws in the design and analysis of some real-world studies.
If researchers of social or economic systems choose inappropriate data and model representations, then their analyses may omit important classes of dependencies.
Specifically, our theory implies that choosing a propositional representation from an inherently relational domain may lead to serious errors.
An abstract ground graph from a given perspective defines the exact set of variables that must be included in any propositionalization.
The absence of any relational variable (including intersection variables) may unnecessarily violate causal sufficiency, which could result in the inference of a causal dependency where conditional independence was not detected.
Our work indicates that researchers should carefully consider how to represent their domains in order to accurately reason about conditional independence.

The abstract ground graph representation also presents an opportunity to derive new edge orientation rules for algorithms that learn the structure of relational models, such as RPC \citep{maier2010rpc} and RCD \citep{maier2013rcd}.
There are unique orientations of edges that are consistent with a given pattern of association that can only be recognized in an abstract ground graph.
For example, in contrast to bivariate IID data, it is simple to establish the direction of causality for bivariate relational data.  Consider the two bivariate, two-entity relational models depicted in Figure~\ref{fig:bivariate-relational-models}(a).
The first model implies that values of $X$ on $A$ entities are caused by the values of $Y$ on related $B$ entities.
The second model implies the opposite, that values of $Y$ on $B$ entities are caused by the values of $X$ on related $A$ entities.
For simplicity, we show the relationship class only as a dashed line between entity classes and omit it from relational paths.

\begin{figure}[t]
\centering
\includegraphics[width=150mm]{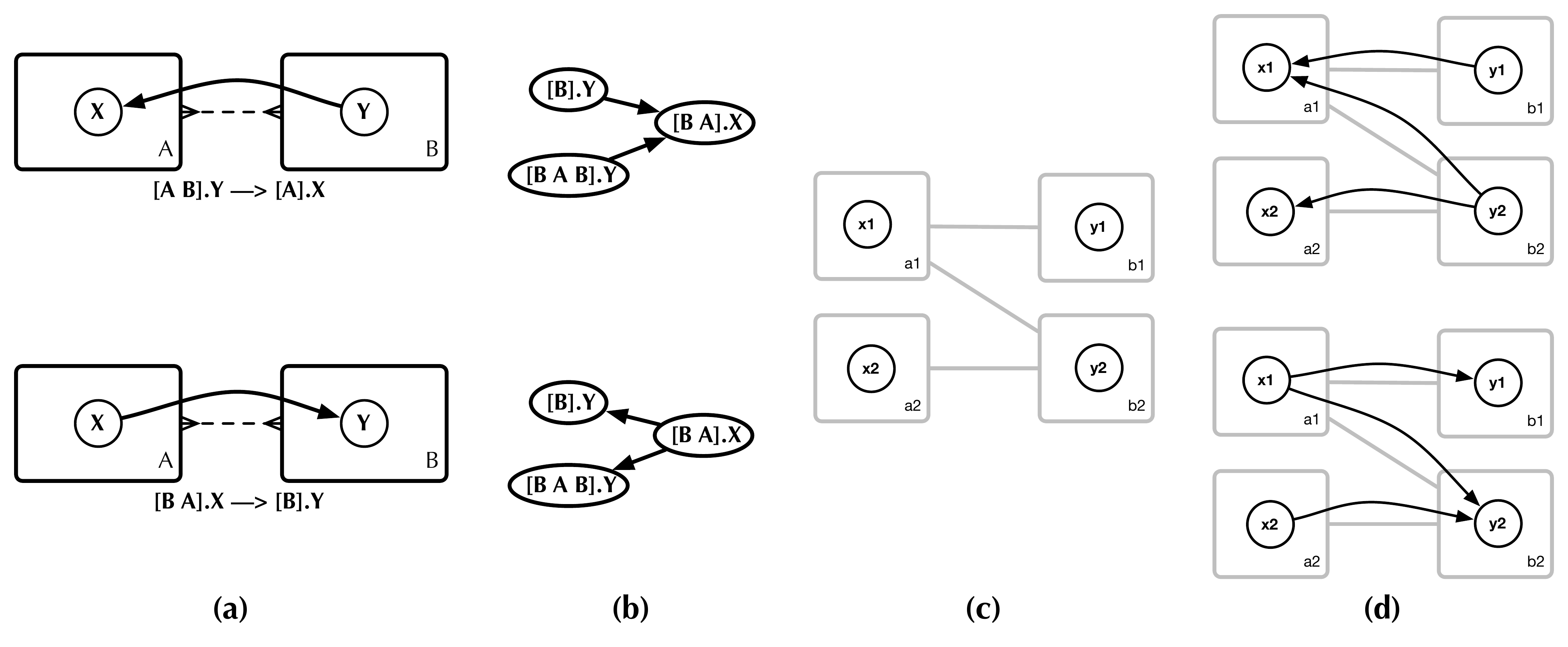}
\caption{(a) Two models of a bivariate relational domain with opposite directions of causality for a single dependency (relationship class omitted for simplicity); (b) a single dependency implies additional dependencies among arbitrary relational variables, shown here in a fragment of the abstract ground graph for $B$'s perspective; (c) an example relational skeleton; and (d) the ground graphs resulting from applying the relational model to the skeleton.}
\label{fig:bivariate-relational-models}
\end{figure}

Figure~\ref{fig:bivariate-relational-models}(b) illustrates a fragment of the abstract ground graph (for hop threshold $h$=4) that each of the two relational models implies.
As expected, the directions of the edges in the two abstract ground graphs are counterposed.
Both models produce observable statistical dependencies for relational variable pairs $\langle [B].Y, [B, A].X\rangle$ and $\langle [B, A].X, [B, A, B].Y\rangle$.
However, the relational variables $[B].Y$ and $[B, A, B].Y$ have different observable statistical dependencies: In the first model, they are marginally independent and conditionally dependent given $[B, A].X$, and in the second model, they are marginally dependent and conditionally independent given $[B, A].X$.
As a result, we can uniquely determine the direction of causality of the single dependence by exploiting relational structure.
(There is symmetric reasoning for relational variables from $A$'s perspective, and this result is also applicable to \textsc{one}-to-\textsc{many} data.)

To illustrate this fact more concretely, consider the small relational skeleton shown in Figure~\ref{fig:bivariate-relational-models}(c) and the ground graphs applied to this skeleton in Figure~\ref{fig:bivariate-relational-models}(d).
In the first ground graph, we have $y_1 \Perp y_2$ and $y_1 \Perpn y_2\ \vert x_1$, but in the second ground graph, we have $y_1 \Perpn y_2$ and $y_1 \Perp y_2\ \vert x_1$.
These opposing conditional independence relations uniquely determine the correct causal model.	
In prior work, we formalized this idea as a new rule, called relational bivariate orientation (RBO) \citep{maier2013rcd}, to orient dependencies in a constraint-based causal discovery algorithm.

Deriving and formalizing the implications of relational \textit{d}-separation is a main direction of future research.
Additionally, our experiments suggest that more accurate tests of conditional independence for relational data need to be developed, specifically tests that can address the interaction between relational structure and aggregation across terminal sets of relational variables.
This work has also focused solely on relational models of attributes; future work should consider models of relationship and entity existence to fully characterize generative models of relational structure.
The theory could also be extended to incorporate functional or deterministic dependencies, as \textit{D}-separation extends \textit{d}-separation for Bayesian networks.
Finally, the work on identifying causal effects in Bayesian networks could be extended to relational models.
This may similarly require an extension of \textit{do}-calculus to consider the space of relational interventions, which may include adding or removing entity or relationship instances, as well as fixing attribute values.

\section*{Acknowledgments}
The authors wish to thank Cindy Loiselle for her editing expertise.
The authors also thank the anonymous reviewers for their helpful comments, prompting us to create a much more readable, precise, correct, and useful paper.
This effort is supported by the Intelligence Advanced Research Project Agency (IARPA) via Department of Interior National Business Center Contract number D11PC20152, Air Force Research Lab under agreement number FA8750-09-2-0187, the National Science Foundation under grant number 0964094, and Science Applications International Corporation (SAIC) and DARPA under contract number P010089628. The U.S. Government is authorized to reproduce and distribute reprints for governmental purposes notwithstanding any copyright notation hereon. The views and conclusions contained herein are those of the authors and should not be interpreted as necessarily representing the official policies or endorsements, either expressed or implied, of IARPA, DoI/NBC, AFRL, NSF, SAIC, DARPA, or the U.S. Government.
Katerina Marazopoulou received scholarship support from the Greek State Scholarships Foundation.

\appendix
\section{Proofs}
\label{sec:appendix_proofs}

In this appendix, we provide detailed proofs for all previous lemmas, theorems, and corollaries.

\begin{figure}[t]
\centering	
\includegraphics[width=120mm]{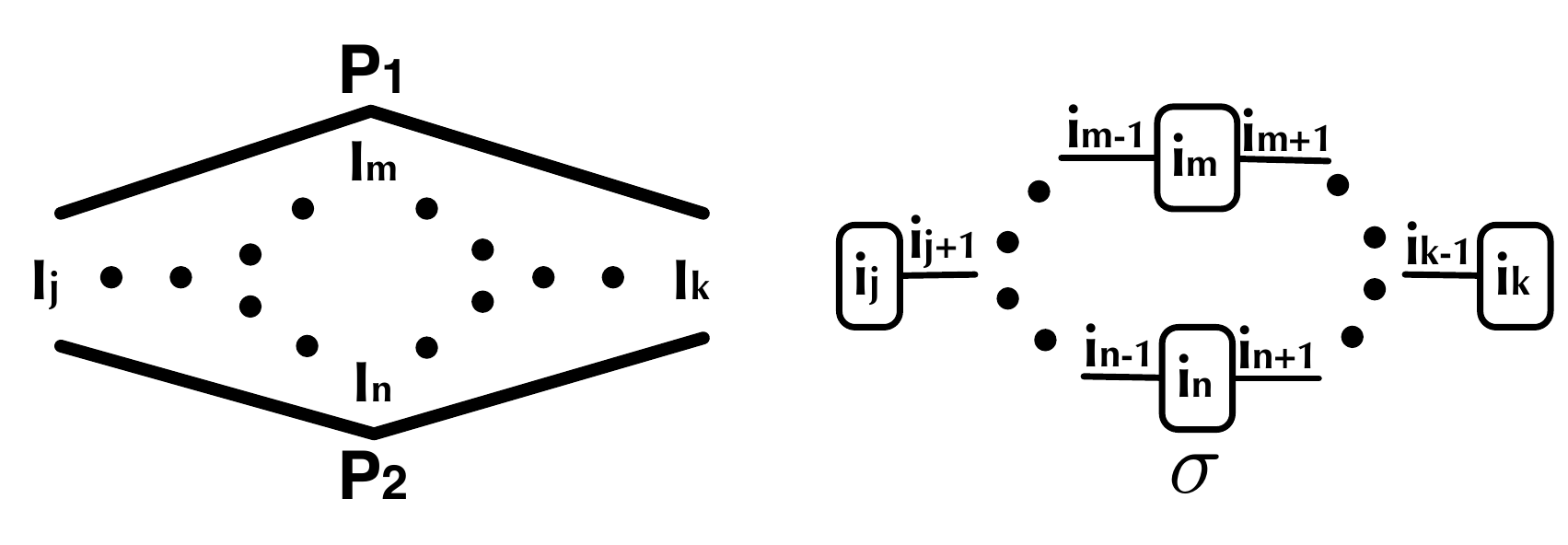}
\caption{Schematic of two relational paths $P_1$ and $P_2$ for which Lemma~\ref{lemma:overlap} guarantees that some skeleton $\sigma$ yields a non-empty intersection of their terminal sets. The example depicts a possible constructed skeleton based on the procedure used in the proof of Lemma~\ref{lemma:overlap}.}
\label{fig:lemma-overlap}
\end{figure}

\begin{repeatlemma}[\ref{lemma:overlap}]
For two relational paths of arbitrary length from $I_j$ to $I_k$ that differ in at least one item class, $P_1=[I_j,\dots,I_m,\dots,I_k]$ and $P_2=[I_j,\dots,I_n,\dots,I_k]$ with $I_m\neq I_n$, there exists a skeleton $\sigma\in\Sigma_\mathcal{S}$ such that $P_1\vert_{i_j} \cap P_2\vert_{i_j}\neq \emptyset$ for some $i_j\in\sigma(I_j)$.
\end{repeatlemma}

\begin{proof}
Proof by construction.  
Let $\mathcal{S}$ be an arbitrary schema with two arbitrary relational paths $P_1=[I_j,\dots, I_m,\dots, I_k]$ and $P_2=[I_j,\dots, I_n,\dots, I_k]$ where $I_m\neq I_n$.
We will construct a skeleton $\sigma\in\Sigma_\mathcal{S}$ such that the terminal sets for item $i_j\in\sigma(I_j)$ along $P_1$ and $P_2$ have a non-empty intersection, that is, an item $i_k\in P_1\vert_{i_j} \cap P_2\vert_{i_j}\neq \emptyset$ (roughly depicted in Figure~\ref{fig:lemma-overlap}).
We use the following procedure to build $\sigma$:
\begin{enumerate}
\item Simultaneously traverse $P_1$ and $P_2$ from $I_j$ until the paths diverge.  For each entity class $E\in\mathcal{E}$ reached, add a unique entity instance $e$ to $\sigma(E)$.
\item Simultaneously traverse $P_1$ and $P_2$ backwards from $I_k$ until the paths diverge.  For each entity class $E\in\mathcal{E}$ reached, add a unique entity instance $e$ to $\sigma(E)$.
\item For the divergent subpaths of both $P_1$ and $P_2$, add unique entity instances for each entity class $E\in\mathcal{E}$.
\item Repeat 1--3 for relationship classes.  For each $R\in\mathcal{R}$ reached, add a unique relationship instance $r$ connecting the entity instances from classes on $P_1$ and $P_2$, and add unique entity instances for classes $E\in R$ not appearing on $P_1$ and $P_2$.
\end{enumerate}
This process constructs an admissible skeleton---all instances are unique and this process assumes no cardinality constraints aside from those required by Definition~\ref{def:rel-path}.
By construction, there exists an item $i_j\in\sigma(I_j)$ such that $P_1\vert_{i_j} \cap P_2\vert_{i_j} = \{i_k\}\neq \emptyset$. $\blacksquare$
\end{proof}

\begin{figure}[t]
\centering	
\includegraphics[width=120mm]{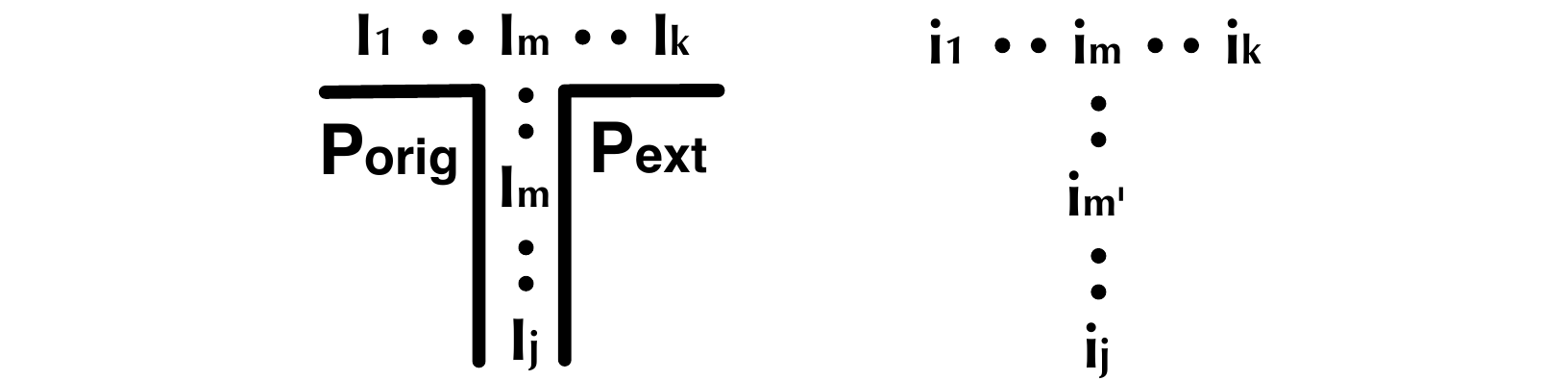}
\caption{Example construction of a relational skeleton for two relational paths $P_{orig} = [I_1,\dots, I_m\dots, I_j]$ and $P_{ext} = [I_j,\dots, I_m\dots, I_k]$, where item class $I_m$ is repeated between $I_m$ and $I_j$.  This construction is used within the proof of Lemma~\ref{lemma:extend}.}
\label{fig:extend}
\end{figure}

\begin{repeatlemma}[\ref{lemma:extend}]
Let $P_{\mathit{orig}} = [I_1,\dots, I_j]$ and $P_{\mathit{ext}} = [I_j,\dots, I_k]$ be two relational paths with $\mathbf{P} = \mathit{extend}(P_{\mathit{orig}}, P_{\mathit{ext}})$.  Then, $\forall P\in\mathbf{P}$ there exists a relational skeleton $\sigma\in\Sigma_\mathcal{S}$ such that $\exists i_1\in \sigma(I_1)$ such that $\exists i_k\in P\vert_{i_1}$ and $\exists i_j\in P_{\mathit{orig}}\vert_{i_1}$ such that $i_k\in P_{\mathit{ext}}\vert_{i_j}$.
\end{repeatlemma}

\begin{proof}
Let $P\in\mathbf{P}$ be an arbitrary valid relational path, where $P = P_{\mathit{orig}}^{1,n_o-c+1} + P_{\mathit{ext}}^{c+1,n_e}$ for pivot $c$.
There are two subcases:

(a) $c=1$ and $P = [I_1,\dots, I_j,\dots, I_k]$.
This subcase holds generally for any skeleton.
Proof by contradiction.
Let $\sigma$ be an arbitrary skeleton, choose $i_1\in\sigma(I_1)$ arbitrarily, and choose $i_k\in P\vert_{i_1}$ arbitrarily.
Assume for contradiction that there is no $i_j$ in the terminal set $P_{\mathit{orig}}\vert_{i_1}$ such that $i_k$ would be in the terminal set $P_{\mathit{ext}}\vert_{i_j}$, that is, $\forall i_j\in P_{\mathit{orig}}\vert_{i_1}\ i_k\notin P_{\mathit{ext}}\vert_{i_j}$.
Since $P = [I_1,\dots, I_j,\dots, I_k]$, we know that $i_k$ is reached by traversing $\sigma$ from $i_1$ via some $i_j$ to $i_k$.
But the traversal from $i_1$ to $i_j$ implies that $i_j\in[I_1,\dots, I_j]\vert_{i_1} = P_{\mathit{orig}}\vert_{i_1}$, and the traversal from $i_j$ to $i_k$ implies that $i_k\in[I_j,\dots, I_k]\vert_{i_j} = P_{\mathit{ext}}\vert_{i_j}$.
Therefore, there must exist an $i_j\in P_{\mathit{orig}}\vert_{i_1}$ such that $i_k\in P_{\mathit{ext}}\vert_{i_j}$.

(b) $c > 1$ and $P = [I_1,\dots, I_m,\dots, I_k]$.
Proof by construction.
We build a relational skeleton $\sigma$ following the same procedure as outlined in the proof of Lemma~\ref{lemma:overlap}.
Add instances to $\sigma$ for every item class that appears on $P_{\mathit{orig}}$ and $P_{\mathit{ext}}$.
Since $P = [I_1,\dots, I_m,\dots, I_k]$, we know that $i_k$ is reached by traversing $\sigma$ from $i_1$ via some $i_m$ to $i_k$.
By case (a), $\exists i_m\in [I_1,\dots, I_m]\vert_{i_1}$ such that $i_k\in[I_m,\dots, I_k]\vert_{i_m}$.
We then must show that there exists an $i_j\in[I_m,\dots, I_j]\vert_{i_m}$ with $i_m\in[I_j,\dots, I_m]\vert_{i_j}$.
But constructing the skeleton with unique item instances for every appearance of an item class on the relational paths provides this and does not violate any cardinality constraints.
If any item class appears more than once, then the bridge burning semantics are induced.
However, adding an additional item instance for every reappearance of an item class enables the traversal from $i_j$ to $i_m$ and vice versa.
An example of this construction is displayed in Figure~\ref{fig:extend}.
This is also a valid relational skeleton because $P_{\mathit{orig}}$ and $P_{\mathit{ext}}$ are valid relational paths, and by definition, the cardinality constraints of the schema permit multiple instances in the skeleton of any repeated item class.
By this procedure, we show that there exists a skeleton $\sigma$ such that there exists an $i_j\in P_{orig}\vert_{i_1}$ such that $i_k\in P_{ext}\vert_{i_j}$. $\blacksquare$
\end{proof}

\begin{figure}[t]
\centering	
\includegraphics[width=120mm]{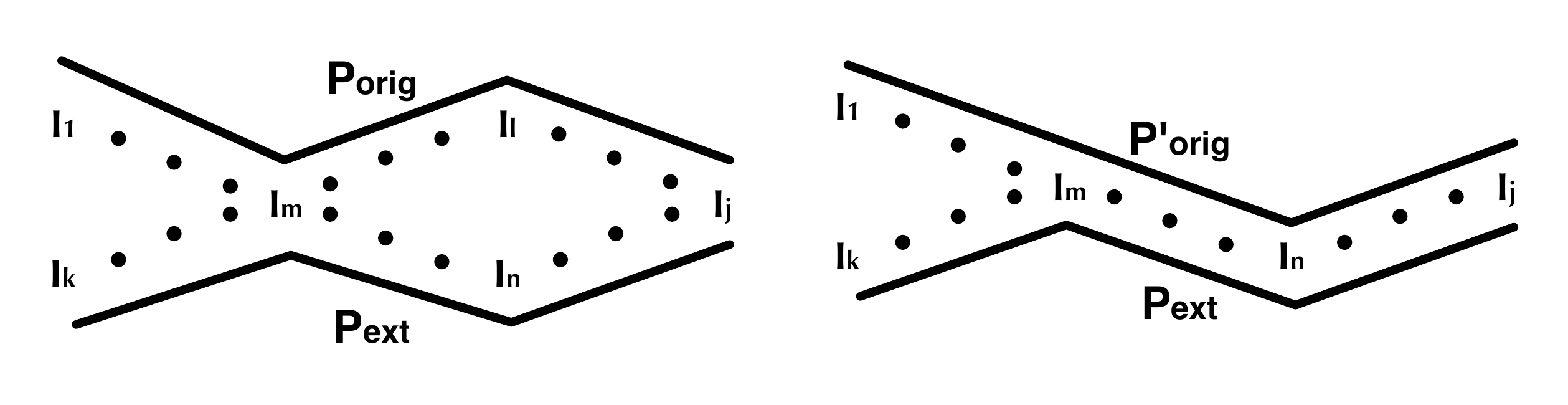}
\caption{Schematic of the relational paths expected in Lemma~\ref{lemma:extend-reverse}. If item $i_k$ is unreachable via $\mathit{extend}(P_{\mathit{orig}}, P_{\mathit{ext}})$, then there must exist a $P_{\mathit{orig}}'$ of the form $[I_1,\dots, I_m,\dots, I_n,\dots, I_j]$.}
\label{fig:extend-reverse}
\end{figure}

\begin{repeatlemma}[\ref{lemma:extend-reverse}]
Let $\sigma\in\Sigma_\mathcal{S}$ be a relational skeleton, and let $P_{\mathit{orig}} = [I_1,\dots, I_j]$ and $P_{\mathit{ext}} = [I_j,\dots, I_k]$ be two relational paths with $\mathbf{P} = \mathit{extend}(P_{\mathit{orig}}, P_{\mathit{ext}})$.  Then, $\forall i_1\!\in\! \sigma(I_1)\ \forall i_j\!\in\! P_{\mathit{orig}}\vert_{i_1}\ \forall i_k\!\in\! P_{\mathit{ext}}\vert_{i_j}$ if $\forall P\in\mathbf{P}\ i_k\notin P\vert_{i_1}$, then $\exists P_{\mathit{orig}}'$ such that $P_{\mathit{orig}}\vert_{i_1}\ \cap\ P_{\mathit{orig}}'\vert_{i_1}\neq\emptyset$ and $i_k\in P'\vert_{i_1}$ for some $P'\in \mathit{extend}(P_{\mathit{orig}}', P_{\mathit{ext}})$.
\end{repeatlemma}

\begin{proof}
Proof by construction.
Let $i_1\in\sigma(I_1)$, $i_j\in P_{\mathit{orig}}\vert_{i_1}$, and $i_k\in P_{\mathit{ext}}\vert_{i_j}$ be arbitrary instances such that $i_k\notin P\vert_{i_1}$ for all $P\in\mathbf{P}$.

Since $i_j\in P_{\mathit{orig}}\vert_{i_1}$ and $i_k\in P_{\mathit{ext}}\vert_{i_j}$, but $i_k\notin P\vert_{i_1}$, there exists no pivot that yields a common subsequence in $P_{\mathit{orig}}$ and $P_{\mathit{ext}}$ that produces a path in $\mathit{extend}$ that can reach $i_k$.
Let the first divergent item class along the reverse of $P_{\mathit{orig}}$ be $I_l$ and along $P_{\mathit{ext}}$ be $I_n$.
The two paths must not only diverge, but they also necessarily reconverge at least once.
If $P_{\mathit{orig}}$ and $P_{\mathit{ext}}$ do not reconverge, then there are no reoccurrences of an item class along any $P\in\mathbf{P}$ that would restrict the inclusion of $i_k$ in some terminal set $P\vert_{i_1}$.
The sole reason that $i_k\notin P\vert_{i_1}$ for all $P\in\mathbf{P}$ is due to the bridge burning semantics specified in Definition~\ref{def:terminal-set}.

Without loss of generality, assume $P_{\mathit{orig}}$ and $P_{\mathit{ext}}$ reconverge once, at item class $I_m$.
So, $P_{\mathit{orig}} = [I_1,\dots, I_m,\dots, I_l,\dots, I_j]$ and $P_{\mathit{ext}} = [I_j,\dots, I_n,\dots, I_m,\dots, I_k]$ with $I_l\ne I_n$, as depicted in Figure~\ref{fig:extend-reverse}.
Let $P_{\mathit{orig}}' = [I_1,\dots, I_m,\dots, I_n,\dots, I_j]$, which is a valid relational path because $[I_1,\dots, I_m]$ is a subpath of $P_{\mathit{orig}}$ and $[I_m,\dots, I_n,\dots, I_j]$ is a subpath of $P_{\mathit{ext}}$.

By construction, $i_j\in P_{\mathit{orig}}\vert_{i_1}\ \cap\ P_{\mathit{orig}}'\vert_{i_1}$.
If $P' = [I_1,\dots, I_m,\dots, I_k]\in \mathit{extend}(P_{\mathit{orig}}', P_{\mathit{ext}})$ with pivot at $I_m$, then $i_k\in P'\vert_{i_1}$. $\blacksquare$
\end{proof}

\begin{repeattheorem}[\ref{thm:agg-abstracts}]
For every acyclic relational model structure $\mathcal{M}$ and perspective $B\in\mathcal{E}\cup\mathcal{R}$, the abstract ground graph $\mathit{AGG}_{\mathcal{M}B}$ is sound and complete for all ground graphs $\mathit{GG}_{\mathcal{M}\sigma}$ with skeleton $\sigma\in\Sigma_\mathcal{S}$.
\end{repeattheorem}

\begin{proof}
Let $\mathcal{M}=(\mathcal{S}, \mathcal{D})$ be an arbitrary acyclic relational model structure and let $B\in\mathcal{E}\cup\mathcal{R}$ be an arbitrary perspective.

\textbf{Soundness}: To prove that $\mathit{AGG}_{\mathcal{M}B}$ is sound, we must show that for every edge $P_k.X\rightarrow P_j.Y$ in $\mathit{AGG}_{\mathcal{M}B}$, there exists a corresponding edge $i_k.X\rightarrow i_j.Y$ in the ground graph $\mathit{GG}_{\mathcal{M}\sigma}$ for some skeleton $\sigma\in\Sigma_\mathcal{S}$, where $i_k\in P_k\vert_b$ and $i_j\in P_j\vert_b$ for some $b\in\sigma(B)$.
There are three subcases, one for each type of edge in an abstract ground graph:

\vspace{2mm}
\begingroup
(a) Let $[B,\dots, I_k].X\rightarrow [B,\dots, I_j].Y\in \mathit{RVE}$ be an arbitrary edge in $\mathit{AGG}_{\mathcal{M}B}$ between a pair of relational variables.
Assume for contradiction that there exists no edge $i_k.X\rightarrow i_j.Y$ in any ground graph:
\abovedisplayskip=4pt
\belowdisplayskip=4pt
\begin{align*}
\forall\sigma\!\in\!\Sigma_\mathcal{S}\ \forall b\!\in\!\sigma(B)\ \forall i_k\!\in\![B,\dots, I_k]\vert_{b}\ \forall i_j\!\in\![B,\dots, I_j]\vert_{b}\ \big(i_k.X\rightarrow i_j.Y\!\notin\! \mathit{GG}_{\mathcal{M}\sigma}\big)
\end{align*}
By Definition~\ref{def:abstract-gg} for abstract ground graphs, if $[B,\dots, I_k].X\rightarrow [B,\dots, I_j].Y\in \mathit{RVE}$, then the model must have dependency $[I_j,\dots, I_k].X\rightarrow [I_j].Y\in\mathcal{D}$ such that $[B,\dots, I_k]\in \mathit{extend}([B,\dots, I_j], [I_j,\dots, I_k])$.
So, by Definition~\ref{def:ground-graph} for ground graphs, there is an edge from every $i_k.X$ to every $i_j.Y$, where $i_k$ is in the terminal set for $i_j$ along $[I_j,\dots, I_k]$:
\begin{align*}
\forall\sigma\in\Sigma_\mathcal{S}\ \forall i_j\in\sigma(I_j)\ \forall i_k\in[I_j,\dots, I_k]\vert_{i_j}\ \big(i_k.X\rightarrow i_j.Y\in \mathit{GG}_{\mathcal{M}\sigma}\big)
\end{align*}
Since $[B,\dots, I_k]\in \mathit{extend}([B,\dots, I_j], [I_j,\dots, I_k])$, by Lemma~\ref{lemma:extend} we know that
\begin{align*}
\exists\sigma\in\Sigma_\mathcal{S}\ \exists b\in\sigma(B)\ \exists i_k\in[B,\dots, I_k]\vert_b\ \exists i_j\in [B,\dots, I_j]\vert_b\ \big(i_k\in[I_j,\dots, I_k]\vert_{i_j}\big)
\end{align*}
Therefore, there exists a ground graph $\mathit{GG}_{\mathcal{M}\sigma}$ such that $i_k.X\rightarrow i_j.Y\in \mathit{GG}_{\mathcal{M}\sigma}$, which contradicts the assumption.
\endgroup

\vspace{2mm}
(b) Let $P_1.X\cap P_2.X\rightarrow [B,\dots, I_j].Y\in \mathit{IVE}$ be an arbitrary edge in $\mathit{AGG}_{\mathcal{M}B}$ between an intersection variable and a relational variable, where $P_1\!\!=\![B,\dots, I_m,\dots, I_k]$ and $P_2=[B,\dots, I_n,\dots, I_k]$ with $I_m\neq I_n$.
By Lemma~\ref{lemma:overlap}, there exists a skeleton $\sigma\in\Sigma_\mathcal{S}$ and $b\in\sigma(B)$ such that $P_1\vert_b\cap P_2\vert_b\neq \emptyset$.
Let $i_k\in P_1\vert_b\cap P_2\vert_b$ and assume for contradiction that for all $i_j\in [B,\dots, I_j]\vert_b$ there is no edge $i_k.X\rightarrow i_j.Y$ in the ground graph $\mathit{GG}_{\mathcal{M}\sigma}$.
By Definition~\ref{def:abstract-gg}, if the abstract ground graph has edge $P_1.X\cap P_2.X\rightarrow [B,\dots, I_j].Y\in \mathit{IVE}$, then either $P_1.X\rightarrow [B,\dots, I_j].Y\in \mathit{RVE}$ or $P_2.X\rightarrow [B,\dots, I_j].Y\in \mathit{RVE}$.
Then, as shown in case (a), there exists an $i_j\in[B,\dots, I_j]\vert_{b}$ such that $i_k.X\rightarrow i_j.Y\in \mathit{GG}_{\mathcal{M}\sigma}$, which contradicts the assumption.

\vspace{2mm}
(c) Let $[B,\dots, I_k].X\rightarrow P_1.Y\cap P_2.Y\in \mathit{IVE}$ be an arbitrary edge in $\mathit{AGG}_{\mathcal{M}B}$ between a relational variable and an intersection variable, where $P_1\!=\![B,\dots, I_m,\dots, I_j]$ and $P_2=[B,\dots, I_n,\dots, I_j]$ with $I_m\neq I_n$. The proof follows case (b) to show that there exists a skeleton $\sigma\in\Sigma_\mathcal{S}$ and $b\in\sigma(B)$ such that for all $i_k\in[B,\dots, I_k]\vert_b$ there exists an $i_j\in P_1\cap P_2\vert_{b}$ such that $i_k.X\rightarrow i_j.Y\in \mathit{GG}_{\mathcal{M}\sigma}$.\\

\textbf{Completeness}: To prove that the abstract ground graph $\mathit{AGG}_{\mathcal{M}B}$ is complete, we show that for every edge $i_k.X\rightarrow i_j.Y$ in every ground graph $\mathit{GG}_{\mathcal{M}\sigma}$ where $\sigma\in\Sigma_\mathcal{S}$, there is a set of corresponding edges in $\mathit{AGG}_{\mathcal{M}B}$.
Specifically, the edge $i_k.X\rightarrow i_j.Y$ yields two sets of relational variables for some $b\in\sigma(B)$, namely $\mathbf{P_k.X} = \{ P_k.X\ \vert\ i_k\in P_k\vert_b\}$ and $\mathbf{P_j.Y} = \{ P_j.Y\ \vert\ i_j\in P_j\vert_b\}$.
Note that all relational variables in both $\mathbf{P_k.X}$ and $\mathbf{P_j.Y}$ are nodes in $\mathit{AGG}_{\mathcal{M}B}$, as are all pairwise intersection variables: $\forall P_k.X, P_k'.X\!\in\!\mathbf{P_k.X}\ \big(P_k.X\cap P_k'.X\in \mathit{AGG}_{\mathcal{M}B}\big)$ and $\forall P_j.Y, P_j'.Y\!\in\!\mathbf{P_j.Y}\ \big(P_j.Y\cap P_j'.Y\in \mathit{AGG}_{\mathcal{M}B}\big)$.
We show that for all $P_k.X\!\in\!\mathbf{P_k.X}$ and for all $P_j.Y\!\in\!\mathbf{P_j.Y}$ either (a) $P_k.X\rightarrow P_j.Y\in \mathit{AGG}_{\mathcal{M}B}$, (b) $P_k.X\cap P_k'.X\rightarrow P_j.Y\in \mathit{AGG}_{\mathcal{M}B}$, where $P_k'.X\in\mathbf{P_k.X}$, or (c) $P_k.X\rightarrow P_j.Y\cap P_j'.Y\in \mathit{AGG}_{\mathcal{M}B}$, where $P_j'.Y\in\mathbf{P_j.Y}$.

Let $\sigma\in\Sigma_\mathcal{S}$ be an arbitrary skeleton, let $i_k.X\rightarrow i_j.Y\in \mathit{GG}_{\mathcal{M}\sigma}$ be an arbitrary edge drawn from $[I_j,\dots, I_k].X\rightarrow [I_j].Y\in\mathcal{D}$, and let $P_k.X\in\mathbf{P_k.X}, P_j.Y\in\mathbf{P_j.Y}$ be an arbitrary pair of relational variables.

(a) If $P_k\in \mathit{extend}(P_j, [I_j,\dots, I_k])$, then $P_k.X\rightarrow P_j.Y\in \mathit{AGG}_{\mathcal{M}B}$ by Definition~\ref{def:abstract-gg}.

(b) If $P_k\notin \mathit{extend}(P_j, [I_j,\dots, I_k])$, but $\exists P_k'\in \mathit{extend}(P_j, [I_j,\dots, I_k])$ such that $P_k'.X\in\mathbf{P_k.X}$, then $P_k'.X\rightarrow P_j.Y\in \mathit{AGG}_{\mathcal{M}B}$, and $P_k.X\cap P_k'.X\rightarrow P_j.Y\in \mathit{AGG}_{\mathcal{M}B}$ by Definition~\ref{def:abstract-gg}.

(c) If $\forall P\in \mathit{extend}(P_j, [I_j,\dots, I_k])\ \big(P.X\notin \mathbf{P_k.X}\big)$, then by Lemma~\ref{lemma:extend-reverse}, $\exists P_j'$ such that $i_j\in P_j'\vert_b$ and $P_k\in \mathit{extend}(P_j', [I_j,\dots, I_k])$.
Therefore, $P_j'.Y\in\mathbf{P_j.Y}$, $P_k.X\rightarrow P_j'.Y\in \mathit{AGG}_{\mathcal{M}B}$, and $P_k.X\rightarrow P_j'.Y\cap P_j.Y\in \mathit{AGG}_{\mathcal{M}B}$ by Definition~\ref{def:abstract-gg}. $\blacksquare$
\end{proof}

\begin{repeattheorem}[\ref{thm:agg-dags}]
For every acyclic relational model structure $\mathcal{M}$ and perspective $B\in\mathcal{E}\cup\mathcal{R}$, the abstract ground graph $\mathit{AGG}_{\mathcal{M}B}$ is directed and acyclic.
\end{repeattheorem}

\begin{proof}
Let $\mathcal{M}$ be an arbitrary acyclic relational model structure, and let $B\in\mathcal{E}\cup\mathcal{R}$ be an arbitrary perspective.
It is clear by Definition~\ref{def:abstract-gg} that every edge in the abstract ground graph $\mathit{AGG}_{\mathcal{M}B}$ is directed by construction.
All edges inserted in any abstract ground graph are drawn from the directed dependencies in a relational model.
Since $\mathcal{M}$ is acyclic, the class dependency graph $G_\mathcal{M}$ is also acyclic by Definition~\ref{def:class-dependency-graph}.
Assume for contradiction that there exists a cycle of length $n$ in $\mathit{AGG}_{\mathcal{M}B}$ that contains both relational variables and intersection variables.
By Definition~\ref{def:abstract-gg}, all edges inserted in $\mathit{AGG}_{\mathcal{M}B}$ are drawn from some dependency in $\mathcal{M}$, even for nodes corresponding to intersection variables.
Retaining only the final item class in each relational path for every node in the cycle will yield a cycle in $G_\mathcal{M}$ by Definition~\ref{def:class-dependency-graph}.
Therefore, $\mathcal{M}$ could not have been acyclic, which contradicts the assumption.
$\blacksquare$
\end{proof}

\section{The Semantics of Bridge Burning}
\label{sec:appendix_burning_bridges}
In this appendix, we provide an example to show that the bridge burning semantics for terminal sets of relational paths yields a strictly more expressive class of relational models than semantics without bridge burning.
The bridge burning semantics produces terminal sets that are necessarily \textit{subsets} of terminal sets which would otherwise be produced without bridge burning.
Paradoxically, this enables a \textit{superset} of relational models.

	\begin{figure}[t]
	\centering
	\subfloat[Relational model]{
	\includegraphics[width=70mm]{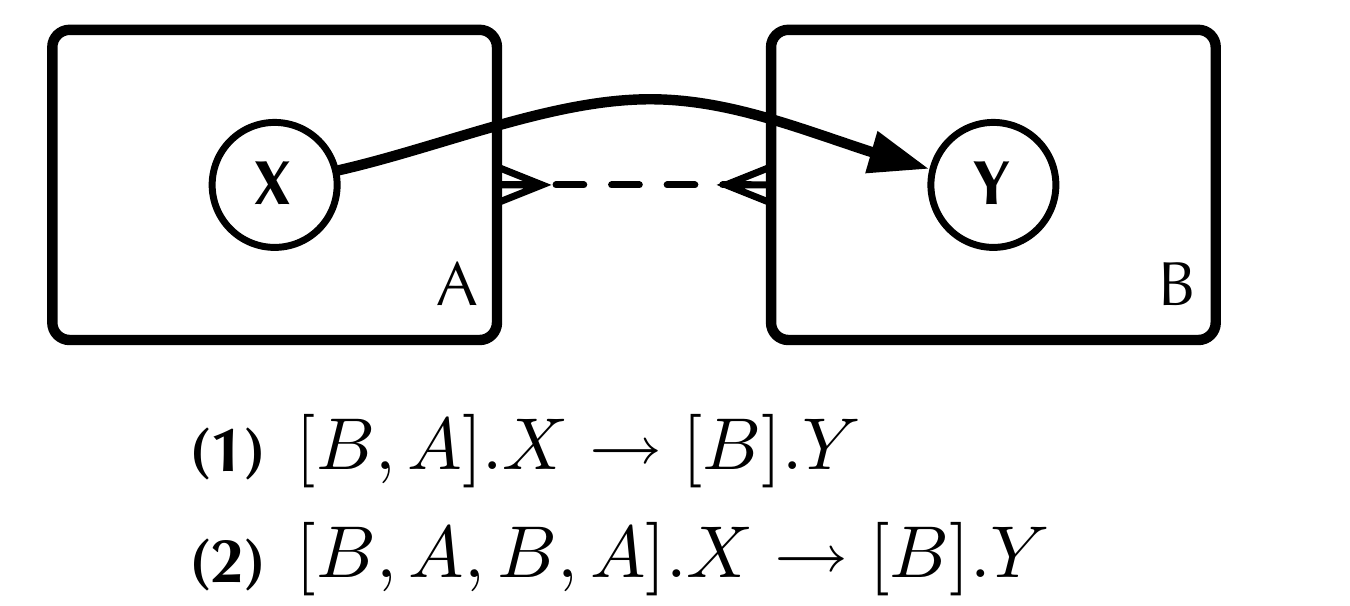}
	\label{fig:bridge-burning-model}}
	\subfloat[Relational skeleton]{
	\includegraphics[width=35mm]{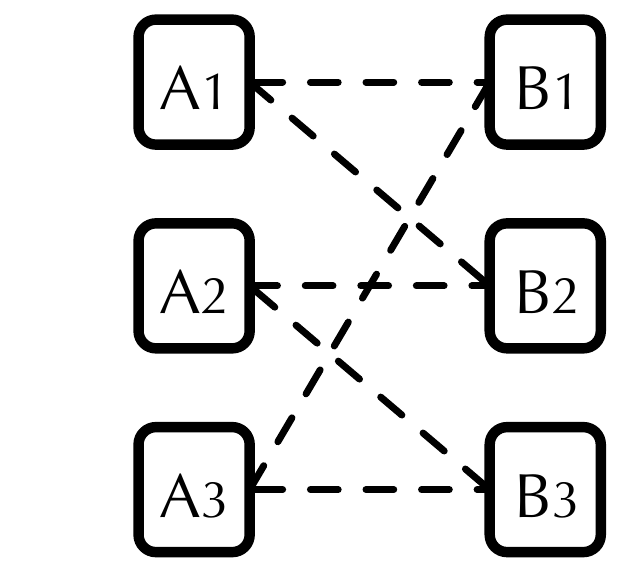}
	\label{fig:bridge-burning-skeleton}}\\
	\subfloat[Ground graphs]{
	\includegraphics[width=130mm]{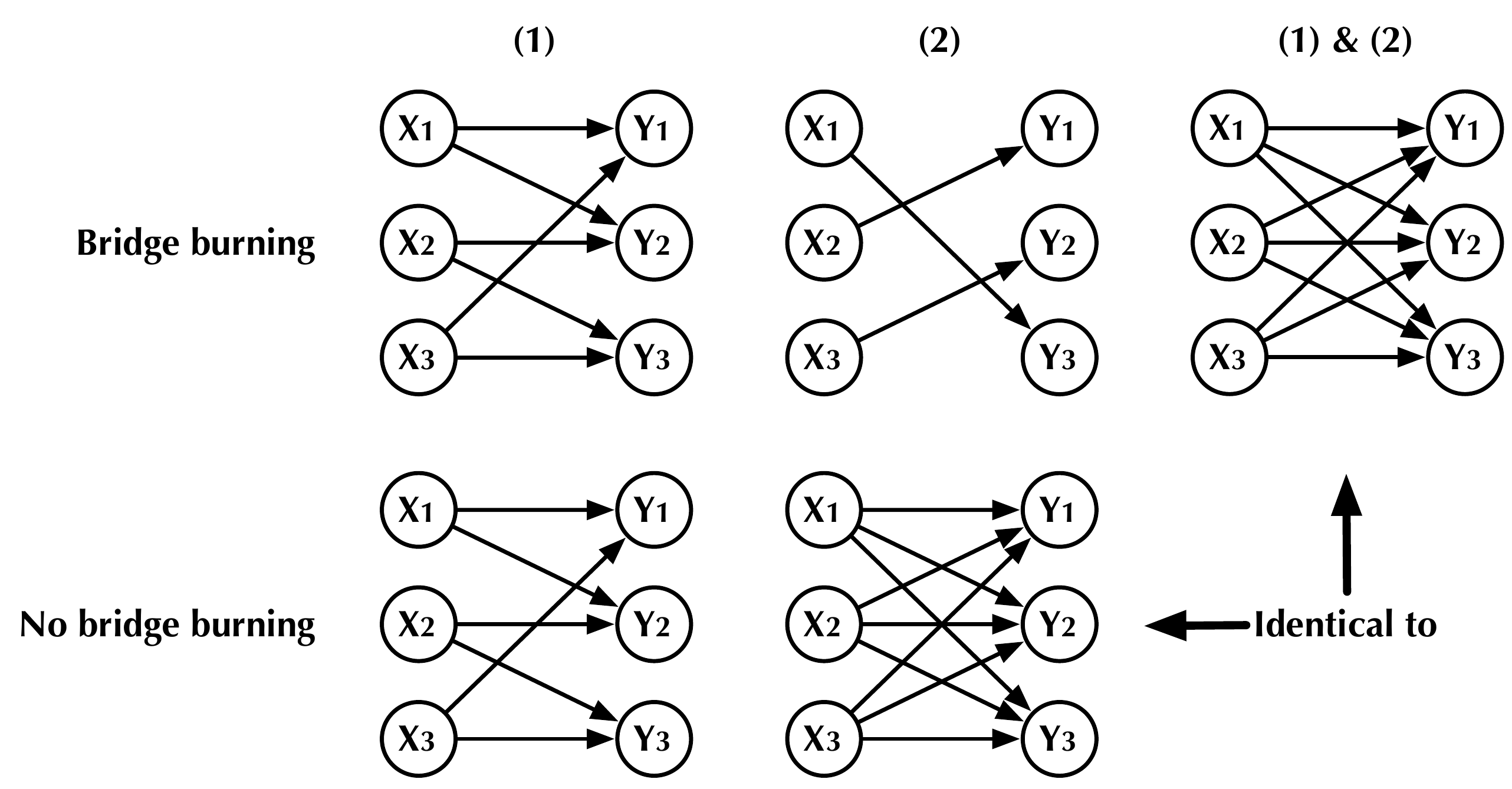}
	\label{fig:bridge-burning-ground-graphs}}
	\caption{Example demonstrating that bridge burning semantics yields a more expressive class of models than semantics without bridge burning. (a) Relational model over a schema with two entity classes and two attributes with two possible relational dependencies (relationship class omitted for simplicity). (b) Simple relational skeleton with three $A$ and three $B$ instances. (c) Bridge burning semantics yields three possible ground graphs with combinations of dependencies (1) and (2), whereas no bridge burning yields two possible ground graphs.  The bridge burning ground graphs subsume the ground graphs without bridge burning.}
	\label{fig:bridge-burning}
	\end{figure}

Recall the definition of a terminal set for a relational path:

\begin{repeatdefinition}[\ref{def:terminal-set}]{Terminal set}
For skeleton $\sigma\in\Sigma_\mathcal{S}$ and $i_j\in \sigma(I_j)$, the \textit{terminal set} $P\vert_{i_j}$ for relational path $P = [I_j,\dots, I_k]$ of length $n$ is defined inductively as
\begin{center}
$P^1\vert_{i_j} = [I_j]\vert_{i_j} = \lbrace i_j \rbrace$

$\vdots$
\end{center}

$P^n\vert_{i_j} = [I_j,\dots, I_k]\vert_{i_j} = \displaystyle\bigcup_{i_{m}\in P^{n-1}\vert_{i_j}} \Big\lbrace  i_k\ \vert\ \big((i_m\in i_k\mbox{ if }I_k\in\mathcal{R})\ \lor\ (i_k\in i_m\mbox{ if }I_k\in\mathcal{E})\big)$
\vspace{-5mm}

\hspace{70mm} $\land\ i_k\notin \displaystyle\bigcup_{l=1}^{n-1} P^{l}\vert_{i_j}\Big\rbrace$
\end{repeatdefinition}

The final condition in the inductive definition ($i_k\notin [I_1,\dots, I_j]\vert_{i_1}$ for $j= 1$ to $k-1$) encodes bridge burning.
The item $i_k$ is only added to the terminal set if it is not a member of the terminal set of any previous subpath.
For example, let $P$ be the relational path $[$\textsc{Employee}, \textsc{Develops}, \textsc{Product}, \textsc{Develops}, \textsc{Employee}$]$.
This relational path produces terminal sets that include the employees that work on the same products (that is, co-workers).
Instantiating this path with the employee Quinn, $P\vert_{\text{Quinn}}$, produces the terminal set $\{$Paul, Roger, Sally$\}$.
Since Quinn $\in [$\textsc{Employee}$]\vert_{\text{Quinn}}$, the bridge burning semantics excludes Quinn from this set.
This makes intuitive sense as well---Quinn should not be considered her own colleague.

A relational model is simply a collection of relational dependencies.
Each relational dependency is primarily described by the relational path of the parent relational variable (because, for canonically specified dependencies, the relational path of the child consists of a single item class).
The relational path specification is used in the construction of ground graphs, connecting variable instances that appear in the terminal sets of the parent and child relational variables.

To characterize the expressiveness of relational models, we can inspect the space of representable ground graphs by choosing an arbitrary relational skeleton and a small set of relational dependencies.
We show with a simple example that the bridge burning semantics for a model over a two-entity, bivariate schema yields more possible ground graphs than without bridge burning.
(We omit the relationship class for simplicity.)
In Figure~\ref{fig:bridge-burning}\subref{fig:bridge-burning-model}, we present such a model with two possible relational dependencies labeled (1) and (2).
Figure~\ref{fig:bridge-burning}\subref{fig:bridge-burning-skeleton} provides a simple relational skeleton involving three $A$ and three $B$ instances (relationship instances are represented as dashed lines for simplicity).
As shown in Figure~\ref{fig:bridge-burning}\subref{fig:bridge-burning-ground-graphs}, the bridge burning semantics leads to three possible ground graphs, one for each combination of the dependencies (1), (2), and both (1) and (2) together.
Without bridge burning, only two ground graphs are possible because dependency (2) completely subsumes dependency (1) with those semantics.

This example generalizes to arbitrary dependencies.  The terminal sets of relational paths that repeat item classes subsume subpaths under the semantics without bridge burning.  This leads to fewer possible relational models, which justifies our choice of semantics for terminal sets of relational paths.

\section{Soundness and Completeness of Relational Paths}
\label{sec:appendix_valid_relational_paths}
In this appendix, we prove that the definition of relational paths (repeated below) is sound and complete with respect to producing non-empty terminal sets for at least one relational skeleton.

\begin{repeatdefinition}[\ref{def:rel-path}]{Relational path}
A \textit{relational path} $[I_j,\dots, I_k]$ for relational schema $\mathcal{S}$ is an alternating sequence of entity and relationship classes $I_j,\dots, I_k\in\mathcal{E}\cup\mathcal{R}$ such that:
\begin{compactenum}
\item[(1)] For every pair of consecutive item classes $[E, R]$ or $[R, E]$ in the path, $E\in R$.
\item[(2)] For every triple of consecutive item classes $[E, R, E']$, $E\neq E'$.
\item[(3)] For every triple of consecutive item classes $[R, E, R']$, if $R = R'$, then $\mathit{card}(R, E)= \textsc{many}$.
\end{compactenum}
\end{repeatdefinition}

\begin{mylemma}
Let $\mathcal{S}$ be a relational schema and $[I_j, \ldots, I_k]$ be a sequence of alternating entity and relationship classes of $\mathcal{S}$ that satisfy participation constraints (condition (1) of Definition~\ref{def:rel-path}). 
The relational path $[I_j, \ldots, I_k]$ satisfies conditions (2) and (3) of Definition ~\ref{def:rel-path} if and only if there exists a relational skeleton $\sigma\in\Sigma_\mathcal{S}$ and an item instance $i_j\in\sigma(I_j)$ such that $[I_j, \ldots, I_k]|_{i_j}\not=\emptyset$. More formally, 
\begin{center}
$\begin{array}{lll}
\exists\sigma\in\Sigma_\mathcal{S}\ \exists i_j\!\in\!\sigma(I_j)\ \big( [I_j, \ldots, I_k]|_{i_j}\not=\emptyset\big) \Leftrightarrow \!\!\!\!&\big( [ERE]\!\! &\not\in [I_j, \ldots, I_k]\big) \\
\!\!\!\!&&\wedge\\
\!\!\!\!&\big( [RER]\!\! &\in [I_j, \ldots, I_k] \rightarrow \mathit{card}(R, E) = \textsc{many}\big)
\end{array}$
\end{center}
\end{mylemma}

\begin{proof}
\textbf{Left-to-right $\Rightarrow$:}
Assume that there exists a skeleton $\sigma\in\Sigma_\mathcal{S}$ and item instance $i_j\in\sigma(I_j)$ such that $[I_j, \ldots, I_k]|_{i_j}\not=\emptyset$.
We must show that $[I_j, \ldots, I_k]$ obeys conditions (2) and (3), i.e., $[I_j, \ldots, I_k]$ does not contain any $[ERE]$ patterns, and if it contains an $[RER]$ pattern, then $\mathit{card}(R, E)=\textsc{many}$.
\begin{itemize}
\item Assume for contradiction that $[I_j, \ldots, I_k]$ contains a pattern of the form $[ERE]$.
From Definition~\ref{def:terminal-set} for terminal sets, it follows that if the terminal set of a path is not empty, then the terminal set of every prefix of that path is not empty:
\begin{align*}
[I_j, \ldots, I_k]|_{i_j}\not=\emptyset \Rightarrow [I_j, \ldots, I_m]|_{i_j}\not=\emptyset \text{ for all } [I_j, \ldots, I_m]\le[I_j, \ldots, I_k]
\end{align*}
By assumption, $[I_j, \ldots, I_k]|_{i_j}\not=\emptyset$; therefore, the prefix $[I_j, \ldots, I_m]$ that ends in the $ERE$ pattern also has a non-empty terminal set:
\begin{align*}
[I_j, \ldots, I_k]|_{i_j}\not=\emptyset &\Rightarrow [I_j, \ldots, E, R, E]|_{i_j}\not=\emptyset\\
[I_j, \ldots, I_k]|_{i_j}\not=\emptyset &\Rightarrow [I_j, \ldots, E, R]|_{i_j}\not=\emptyset\\
[I_j, \ldots, I_k]|_{i_j}\not=\emptyset &\Rightarrow [I_j, \ldots, E]|_{i_j}\not=\emptyset
\end{align*}
Let $e\in\sigma(E)$ be an entity instance in the terminal set $[I_j, \ldots, E]|_{i_j}$.
Since the terminal set $[I_j, \ldots, E, R]|_{i_j}$ is not empty, it follows that there exists a relationship instance $r = \langle \ldots, e, \ldots\rangle$ such that  $r\in[I_j, \ldots, E, R]|_{i_j}$.
However, $[I_j, \ldots, E, R, E]|_{i_j}$ is also not empty; thus, there exists some $e'\in\sigma(E)$ such that $e'\in [I_j, \ldots, E, R, E]|_{i_j}$, where $e'\not=e$, and $e'\in r$.
It follows that both $e$ and $e'$ participate in the relationship instance $r$, which is a contradiction.

\item Assume for contradiction that $[I_j, \ldots, I_k]$ contains a pattern of the form $[R, E, R]$ and $\mathit{card}(R, E) =$ \textsc{one}.
\begin{align}
[I_j, \ldots, R]|_{i_j} \not=\emptyset &\Rightarrow \exists r=\langle e,\ldots\rangle\in [I_j, \ldots, R]|_{i_j}\\
[I_j, \ldots, R, E]|_{i_j} \not=\emptyset &\Rightarrow \exists e\in [I_j, \ldots, R, E]|_{i_j} \text{ and } e\in r\nonumber \\
[I_j, \ldots, R, E, R]|_{i_j} \not=\emptyset &\Rightarrow \exists r'=\langle e, \ldots\rangle\text{ such that } r'\in [I_j, \ldots, R, E, R]|_{i_j}\\
&\hspace{3cm}\text{ and } r'\not=r \text{ (bridge burning semantics)}\nonumber
\end{align}
From (1) and (2) it follows that $e$ participates in two instances of $R$; therefore, $\mathit{card}(R, E)$ must be \textsc{many}, which is a contradiction.
\end{itemize}

\textbf{Right-to-left $\Leftarrow$:} Assume that $[I_j, \ldots, I_k]$ adheres to Definition~\ref{def:rel-path} for relational paths. 
We must show that $\exists\sigma\!\in\!\Sigma_\mathcal{S}\ \exists i_j\!\in\!\sigma(I_j)\ \big( [I_j, \ldots, I_k]|_{i_j}\not=\emptyset\big)$.
We can construct such a skeleton $\sigma$ according to the following procedure:
For each entity class $E$ on the path, add a unique entity instance $e$ to $\sigma(E)$.
Then, for each relationship class $R$ on the path, add a unique relationship instance $r$ connecting the previously created unique entity instances that participate in $R$, and add unique entity instances for classes $E\in R$ not appearing on the path.
This process constructs an admissible skeleton---all instances are unique and this process assumes no cardinality constraints aside from those required by Definition~\ref{def:rel-path}.
By construction, there exists an item instance $i_j\in\sigma(I_j)$ such that $ [I_j, \ldots, I_k]|_{i_j}\not=\emptyset$. $\blacksquare$
\end{proof}

\section{Background on Propositional Data and Models}
\label{sec:appendix_prop_background}
In this appendix, we provide a brief review of Bayesian networks, traditional \textit{d}-separation, and their connection to causality.
We also describe why the class of Bayesian networks is a special case of relational models.
Finally, we give an example of how to propositionalize a data set drawn from a relational domain.

A common assumption in classical statistics, machine learning, and causal discovery is that data instances are independent and identically distributed (IID).
The first condition assumes that the variables on any given data instance are marginally independent of the variables of any other data instance.
The second condition assumes that every data instance is drawn from the same underlying joint probability distribution.
IID data (also referred to as propositional data\footnote{IID data are typically referred to as \textit{propositional} because the data can be equivalently expressed under propositional logic.}) are effectively represented as a single table, where rows correspond to the independent instances and columns are attributes of those instances.
	
A Bayesian network is a widely used probabilistic graphical model of propositional data \citep{pearl-probreas88}.
A Bayesian network is represented as a directed acyclic graph $G = (\mathbf{V}, \mathbf{E})$, where $\mathbf{V}$ is a set of vertices  corresponding to random variables in the data and $\mathbf{E}\subset \mathbf{V}\times \mathbf{V}$ is a set of edges encoding the probabilistic dependencies among the variables.
Each random variable $V\in\mathbf{V}$ is associated with a conditional probability distribution $P\big(V\ \vert\ \mathit{parents}(V)\big)$, where $\mathit{parents}(V)\subseteq \mathbf{V}\setminus \{ V \}$ is the set of parent variables for $V\!$.

If the joint probability distribution $P(\mathbf{V})$ satisfies the Markov condition for $G$, then $P(\mathbf{V})$ can be factored as $\prod_{V\in\mathbf{V}} P\big(V\ \vert\ \mathit{parents}(V)\big)$ using the conditional distributions.
The Markov condition states that every variable $V\in\mathbf{V}$ is conditionally independent of its non-descendants given its parents, where the descendants of $V$ are all variables reachable by a directed path from $V\!$.
Deriving the set of conditional independencies from $G$ based on the Markov condition is cumbersome, requiring complex combinations of probability axioms.
Fortunately, \textit{d}-separation, a set of graphical criteria, provides the foundation for algorithmic derivation of all conditional independencies in $G$ and entails the exact same set of conditional independencies as the Markov condition \citep{verma1988cnse, geiger1988lcm, neapolitan2004learning}.

In the following definition, a path is a sequence of vertices following edges in either direction.  We say that a variable $V$ is a collider on a path $p$ if the two arrowheads point at each other (collide) at $V\!$; otherwise, $V$ is a non-collider on $p$.

\begin{definition}[\textit{d}-separation]
Let $\mathbf{X}$, $\mathbf{Y}$, and $\mathbf{Z}$ be disjoint sets of variables in directed acyclic graph $G$.
A path from some $X\in\mathbf{X}$ to some $Y\in\mathbf{Y}$ is \textit{d-connected} given $\mathbf{Z}$ if and only if every collider $W$ on the path, or a descendant of $W\!$, is a member of $\mathbf{Z}$, and there are no non-colliders in $\mathbf{Z}$.
Then, say that $\mathbf{X}$ and $\mathbf{Y}$ are \textit{d-separated} by $\mathbf{Z}$ if and only if there are no \textit{d}-connecting paths between $\mathbf{X}$ and $\mathbf{Y}$ given $\mathbf{Z}$.
\label{def:d_sep}
\end{definition}
	
Figure~\ref{fig:d_separation}\subref{fig:d_separation_path_elements} depicts the graphical patterns found along paths that lead to \textit{d}-separation or \textit{d}-connection based on Definition~\ref{def:d_sep}, and Figure~\ref{fig:d_separation}\subref{fig:d_separation_examples} provides example \textit{d}-separated and \textit{d}-connected paths.
At first glance, identifying conditional independence facts using the rules of \textit{d}-separation appears computationally intensive, testing a potentially exponential number of paths.
However, \citet{geiger1990iibn} provide a linear-time algorithm based on breadth-first search and reachability on $G$.

\begin{figure}[t]
\centering
\subfloat[Graphical patterns of \textit{d}-separating and \textit{d}-connecting path elements among disjoint sets of variables $\mathbf{X}$ and $\mathbf{Y}$ given $\mathbf{Z}$.  Paths for which there exists a non-collider in $\mathbf{Z}$ or a collider not in $\mathbf{Z}$ are \textit{d}-separating.  Paths for which all non-colliders are not in $\mathbf{Z}$ and all colliders (or a descendant of colliders) are in $\mathbf{Z}$ are \textit{d}-connecting.]{
\includegraphics[width=120mm]{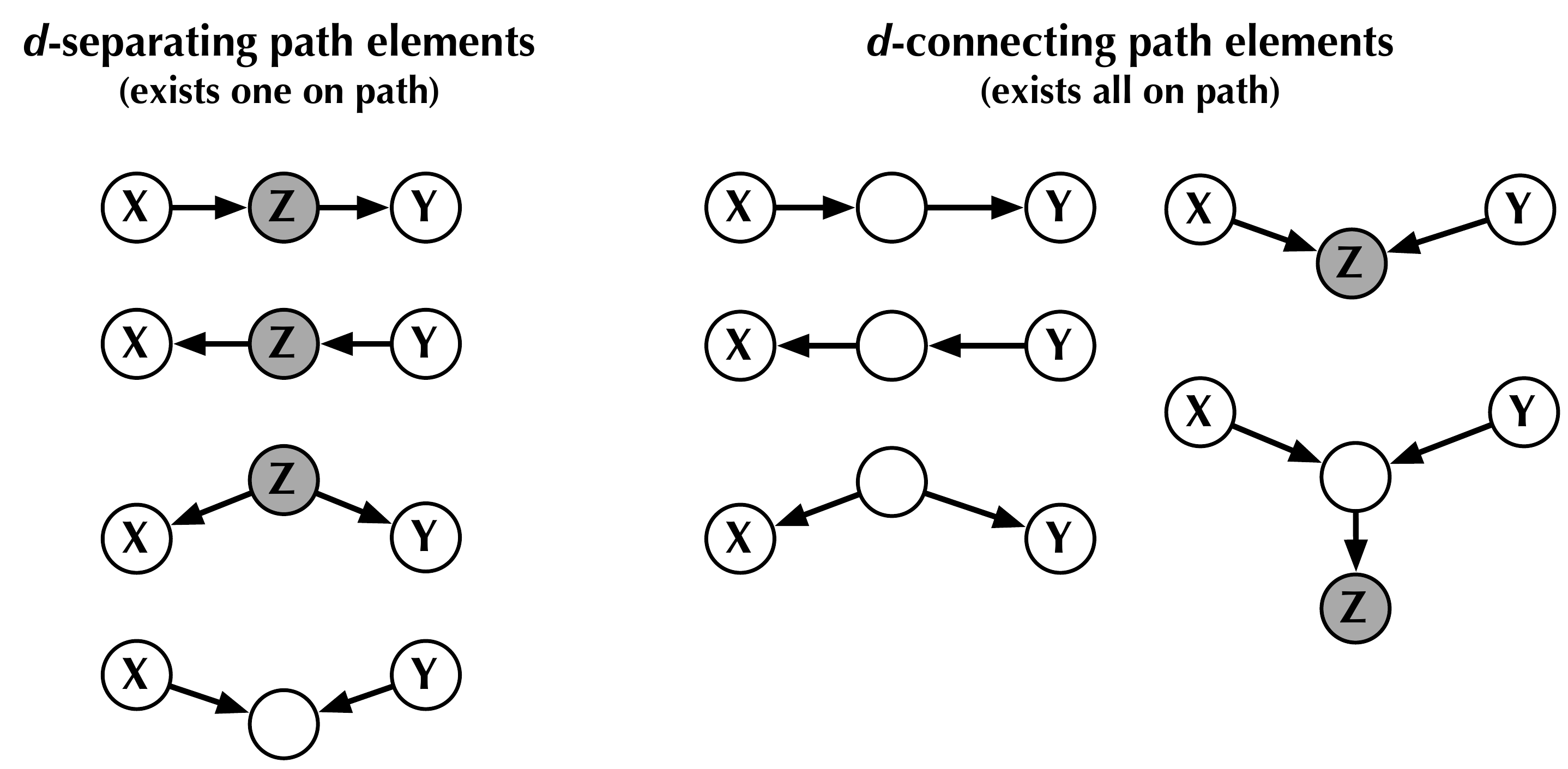}
\label{fig:d_separation_path_elements}}\\

\subfloat[Several example \textit{d}-separated and \textit{d}-connected paths that illustrate the composition of path elements.]{
\includegraphics[width=120mm]{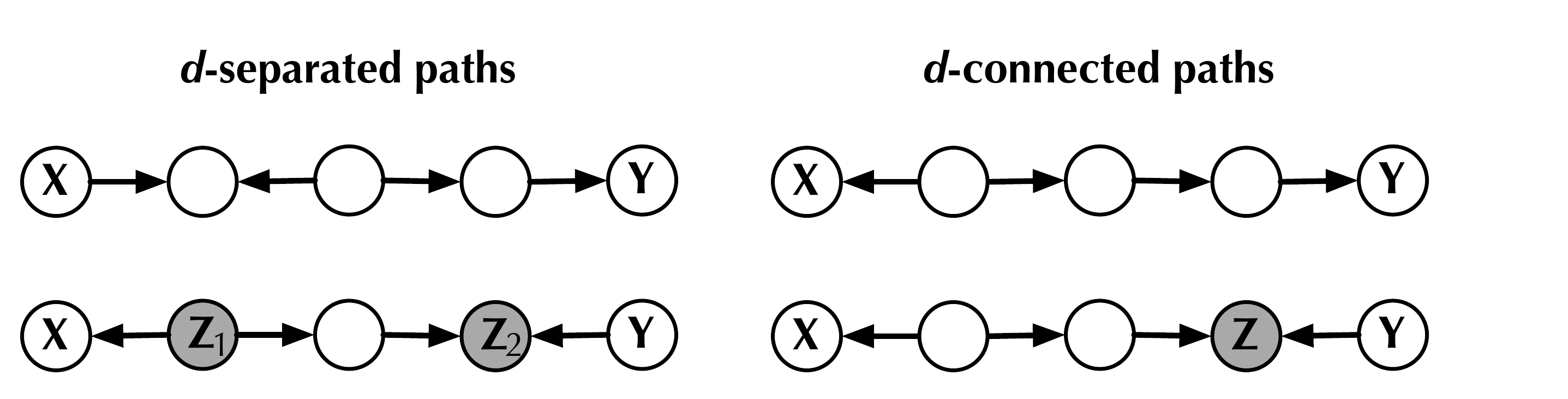}
\label{fig:d_separation_examples}}
\caption{Patterns of \textit{d}-separating and \textit{d}-connecting path elements and example \textit{d}-separating and \textit{d}-connecting paths.}
\label{fig:d_separation}
\end{figure}

Under a few assumptions, Bayesian networks can be interpreted causally, with edges corresponding to direct causal dependencies.
If $X\rightarrow Y$ is an edge in the causal model $G$, then manipulating or changing the value of $X$ will alter the conditional distribution of $Y$---denoted as $P\big(Y\ \vert\ \mathit{do}(X)\big)$ using Pearl's \textit{do}-calculus notation for interventions \citep{pearl-causality00}.
The causal interpretation of $G$ assumes the \textit{causal} Markov condition, which is identical to the Markov condition, replacing parents with direct causes and non-descendants with non-effects.
In order for the causal Markov condition to hold, the variables $\mathbf{V}$ must also be \textit{causally sufficient}: There are no latent common causes for any pair of variables in $\mathbf{V}$.
The causal Markov condition is also equivalent to \textit{d}-separation; therefore, both provide the connection between causal structures and probability distributions.

The conditional independencies entailed by both the causal Markov condition and \textit{d}-separation hold in all distributions that $G$ represents.
A distribution $P$ is \textit{faithful} to $G$ if all conditional independencies in $P$ are entailed by the causal Markov condition on $G$. 
If $P$ is assumed to be faithful to $G$, then there are algorithms that can learn the Markov, or likelihood, equivalent set of causal models.  
These algorithms assume causal sufficiency, faithfulness, and model acyclicity 
to identify the edges in $G$ that are consistent with observed conditional independencies and to determine the direction of causality \citep{spirtes-etal-book00}.

\begin{figure}[t]
\centering
\texttt{
\begin{tabular}{ll}
SELECT & t$_1$.id, t$_1$.salary, t$_2$.success, t$_3$.revenue  \\ 
FROM & \begin{tabular}[t]{lll}
( & SELECT & E.id, E.salary \\ 
  & FROM & Employee E) t$_1$, \\ 
( & SELECT & E.id, P.success \\ 
  & FROM & Employee E, Develops D, Product P \\ 
  & WHERE & E.id = D.e\_id AND D.p\_id = P.id) t$_2$, \\ 
( & SELECT & E.id, B.revenue \\ 
  & FROM & Employee E, Develops D, Product P \\ 
  & & Funds F, Business-Unit B \\ 
  & WHERE & E.id = D.e\_id AND D.p\_id = P.id AND \\ 
  & & P.id = F.p\_id AND F.b\_id = B.id) t$_3$
\end{tabular} \\
WHERE & t$_1$.id = t$_2$.id AND t$_2$.id = t$_3$.id 
\end{tabular}
}
\caption{Sketch of a relational database query that joins the instances of three relational variables having the common perspective \textsc{Employee} used to produce the data instances shown in Table~\ref{tab:employee-propositional}. The three relational variables are (1) $[$\textsc{Employee}$]$.\textit{Salary}, (2) $[$\textsc{Employee}, \textsc{Develops}, \textsc{Product}$]$.\textit{Success}, and (3) $[$\textsc{Employee}, \textsc{Develops}, \textsc{Product}, \textsc{Funds}, \textsc{Business-Unit}$]$.\textit{Revenue}.}
\label{fig:employee-query}
\end{figure}

The relational representation presented in Section~\ref{sec:rel_concepts} is strictly more expressive than the propositional representation used in Bayesian network modeling.
Propositional representations describe domains with a single entity class; thus, they produce schemas with $\vert \mathcal{E}\vert = 1$ (one entity class) and $\vert \mathcal{R}\vert = 0$ (no relationship classes).
For the organization domain example, consider data about only employees ($\mathcal{E} = \{$\textsc{Employee}$\}$).
Variables would include intrinsic attributes, such as salary, but could also include variables describing other related entities, all from the employee perspective.
This technique of translating a relational database down to a single, propositional representation is often referred to as \textit{propositionalization} \citep{kramer-etal-rdmchapter01}.
That is, we could construct a single table for employees that includes columns for the success of developed products, the revenue of all business units they work under, etc.
In Figure~\ref{fig:employee-query}, we show an example SQL-like query that would produce such data, and the resulting data set applied to the example in Figure~\ref{fig:org_example_data}\subref{fig:org_skeleton} is shown in Table~\ref{tab:employee-propositional}.\footnote{Note that modeling propositionalized data with Bayesian networks still requires the IID assumption, which is often violated since variables of one instance can influence variables of another.
For example, according to the model in Figure~\ref{fig:org_example_model}\subref{fig:org_model} the competence of collaborating employees influences the success of products, which affects the revenue of business units, which affects its budget, thereby influencing an employee's salary.
As a result, modeling relational data with a propositional representation may unnecessarily lose valuable information, especially in the context of causal reasoning and accurate estimation of causal effects.}

The relational skeleton of a Bayesian network consists of a set of disconnected entity instances, all drawn from the same entity class.
Consequently, the skeleton has a simple one-to-one mapping with the representation as a table: Each entity instance corresponds to a single row, and each variable is a column.
In this example, each employee would be an entity instance, and no instances of other entity types or relationships would appear in the skeleton.
Because all variables in a Bayesian network are defined for a single entity class and no relationships, the relational path specification becomes trivial and, hence, implicit.
All relational paths, relational variables, and relational dependencies are defined from a single perspective with singleton paths (e.g., $[$\textsc{Employee}$]$).
The ground graph of a Bayesian network, similar to the skeleton, has a very regular structure.
The ground graph consists of a set of identical copies of the model structure, one for each instance in the skeleton.
For a Bayesian network, \textit{d}-separation can be applied directly to the model structure because there is no variability in its ground graphs.

\begin{table}[t]
\centering
\begin{tabular}{c|c|p{40mm}|p{40mm}}
\textsc{Employee} & [\textsc{Employee}].\textit{Salary} & \parbox[t]{5cm}{[\textsc{Employee}, \textsc{Develops}, \\ \textsc{Product}].\textit{Success}} & \parbox[t]{5cm}{[\textsc{Employee}, \textsc{Develops}, \\ \textsc{Product}, \textsc{Funds}, \\ \textsc{Business-Unit}].\textit{Revenue}} \\ 
\hline
Paul & \{Paul.\textit{Salary}\} & \{{\proda}.\textit{Success}\} & \{{\businessa}.\textit{Revenue}\}\\

Quinn & \{Quinn.\textit{Salary}\} & \parbox[t]{5cm}{\{{\proda}.\textit{Success}, \\{\prodb}.\textit{Success}, \\{\prodc}.\textit{Success}\}} & \parbox[t]{5cm}{\{{\businessa}.\textit{Revenue}, \\{\businessb}.\textit{Revenue}\}}\\

Roger & \{Roger.\textit{Salary}\} & \{{\prodc}.\textit{Success}\} & \{{\businessb}.\textit{Revenue}\}\\

Sally & \{Sally.\textit{Salary}\} & \parbox[t]{5cm}{\{{\prodc}.\textit{Success}, \\{\prodd}.\textit{Success}\}} & \{{\businessb}.\textit{Revenue}\}\\

Thomas & \{Thomas.\textit{Salary}\} & \parbox[t]{5cm}{\{{\prodd}.\textit{Success}, \\{\prode}.\textit{Success}\}} & \{{\businessb}.\textit{Revenue}\}
\end{tabular}
\caption{Propositional table consisting of employees, their salary, the success of products they develop, and the revenue of the business units they operate under.  Producing this table requires joining the instances of three relational variables, all from a common perspective---\textsc{Employee}.}
\label{tab:employee-propositional}
\end{table}

%\begin{figure}[t]
%\centering
%\subfloat[]{
%\includegraphics[width=75mm]{prop-org-schema.pdf}
%\label{fig:prop-org-schema}}\\
%\subfloat[]{
%\includegraphics[width=120mm]{prop-org-skeleton.pdf}
%\label{fig:prop-org-skeleton}}
%\caption{Relational schema (a) and skeleton (b) for a propositional representation of employees from the organization domain example.  The skeleton consists of individual, disconnected entity instances.  Each RV$_i$ is a relational variable from the \textsc{Employee} perspective, such as the proportion of products developed successfully or the total revenue of involved business units.}
%\label{fig:prop-org-data}
%\end{figure}

%\begin{figure}[t]
%\centering
%\subfloat[]{
%\includegraphics[width=75mm]{prop-org-model.pdf}
%\label{fig:prop-org-model}}\\
%\subfloat[]{
%\includegraphics[width=120mm]{prop-org-ground-graph.pdf}
%\label{fig:prop-org-ground-graph}}
%\caption{Example Bayesian network (a) and ground graph (b) for the propositional representation of employees from the organization domain.  The model consists of simple dependencies involving employee variables.  The ground graph is a set of identical copies of the model structure, one for each instance in the data (skeleton).}
%\label{fig:prop-org-model-all}
%\end{figure}

\section{Hop Thresholds}
\label{sec:hop_thresholds}
For practical implementations, the size of the abstract ground graphs should be limited by a domain-specific threshold.
In this work, we choose to apply a singular hop threshold to the relational paths that are represented in an abstract ground graph.
In this appendix, we examine the effect of choosing a particular hop threshold. % $h_a$.
%We denote abstract ground graphs for perspective $B$, limited by a hop threshold $h_a$ as $\mathit{AGG}_{\mathcal{M}Bh_a}$.

%For practical implementations, the length of relational paths should be limited by a domain-specific hop threshold.
%In this work, we choose to apply a singular threshold to all relational paths, which also constrains the space of relational dependencies and models.
%In this appendix, we examine the effect of choosing a particular hop threshold $h$ to the soundness and completeness results of Section~\ref{sec:rds}.
%Specifically, we guarantee completeness for abstract ground graphs with a threshold of $3h-2$.

First, we introduce the notion of $(B,h)$-reachability, which describes the conditions under which an edge in a ground graph is represented in an abstract ground graph.

\begin{definition}[$(B,h)$-reachability]
\label{def:reachability}
Let $\mathit{GG}_{\mathcal{M}\sigma}$ be the ground graph for some relational model structure $\mathcal{M}$ and skeleton $\sigma\in\Sigma_\mathcal{S}$.
Then, $i_k.X\rightarrow i_j.Y\in \mathit{GG}_{\mathcal{M}\sigma}$ is $(B,h)$\textit{-reachable} for perspective $B$ and hop threshold $h$ if there exist relational variables $P_k.X=[B,\dots, I_k].X$ and $P_j.Y = [B,\dots, I_j].Y$ such that $\mathit{length}(P_k)\leq h+1$, $\mathit{length}(P_j)\leq h+1$, and there exists an instance $b\in\sigma(B)$ with $i_k\in P_k\vert_b$ and $i_j\in P_j\vert_b$.
\end{definition}

In other words, the edge $i_k.X\rightarrow i_j.Y$ in the ground graph is $(B,h)$-reachable if an instance of the base item $b\in\sigma(B)$ can reach $i_k$ and $i_j$ in at most $h$ hops.

%\begin{example}
%\label{ex:reachability}
%Consider the ground graph shown in Figure~\ref{fig:org_example_model}\subref{fig:org_ground_graph}.
%Let perspective $B$ be \textsc{Employee}, and let the hop threshold $h=6$.
%Let $i_k.X\rightarrow i_j.Y$ be the edge {\prodc}$.\mathit{Success}\rightarrow$ {\businessb}$.\mathit{Revenue}$ in the ground graph.
%This edge is $(B,h)$-reachable because of the following:
%Set $P_k.X = [$\textsc{Employee}, \textsc{Develops}, \textsc{Product}$].\mathit{Success}$, $P_j.Y = [$\textsc{Employee}, \textsc{Develops}, \textsc{Product}, \textsc{Funds}, \textsc{Business-Unit}$].\mathit{Revenue}$, and let $b = $ Sally $\in\sigma($\textsc{Employee}$)$.
%We have $\mathit{length}(P_k) = 3 < 7$, $\mathit{length}(P_j) = 5 < 7$, {\prodc}$\in P_k\vert_{\text{Sally}}$, and {\businessb}$\in P_j\vert_{\text{Sally}}$.
%$\square$
%\end{example}

Since Definition~\ref{def:reachability} pertains to edges reachable via a particular perspective $B$ and hop threshold $h$, it relates to the reachability of edges in abstract ground graphs.
We denote abstract ground graphs for perspective $B$, limited by a hop threshold $h$ as $\mathit{AGG}_{\mathcal{M}Bh}$.
Definition~\ref{def:reachability} implies that (1) for every edge in ground graph $\mathit{GG}_{\mathcal{M}\sigma}$, we can derive a set of abstract ground graphs for which that edge is $(B,h)$-reachable, and (2) for every abstract ground graph $\mathit{AGG}_{\mathcal{M}Bh}$, we can derive the set of $(B,h)$-reachable edges for a given ground graph.
Given $(B,h)$-reachability, we can now express the soundness and completeness of abstract ground graphs.

\begin{mytheorem}
\label{thm:agg-soundness-bh}
For every acyclic relational model structure $\mathcal{M}$, perspective $B\in\mathcal{E}\cup\mathcal{R}$, and hop threshold $h_a\in\mathbb{N}^0$, the abstract ground graph $\mathit{AGG}_{\mathcal{M}Bh_a}$ is sound up to hop threshold $h_a$ for all ground graphs $\mathit{GG}_{\mathcal{M}\sigma}$ with skeleton $\sigma\in\Sigma_\mathcal{S}$.
\end{mytheorem}

\begin{proof}
Soundness means that for every edge $[B,\ldots, I_j].X \rightarrow [B,\ldots, I_k].Y$ in the abstract ground graph $\mathit{AGG}_{\mathcal{M}Bh_a}$, there exists a skeleton $\sigma\in\Sigma_\mathcal{S}$, a base item instance $b\in\sigma(B)$, an instance $i_j\in[B,\ldots, I_j]|b$, and an instance $i_k\in[B,\ldots, I_k]|b$ such that 
$i_j.X \rightarrow i_k.Y$ is a $(B,h_a)$-reachable edge in $\mathit{GG}_{\mathcal{M}\sigma}$.
The proof is identical to the proof of soundness for Theorem~\ref{thm:agg-abstracts} (see Appendix~\ref{sec:appendix_proofs}). $\blacksquare$
\end{proof}

\begin{mytheorem}
\label{thm:agg-completeness-bh}
For every acyclic relational model structure $\mathcal{M}$, perspective $B\in\mathcal{E}\cup\mathcal{R}$, and hop threshold $h_r\in\mathbb{N}^0$, the abstract ground graph $\mathit{AGG}_{\mathcal{M}Bh_a}$ is complete up to hop threshold $h_r$ for all ground graphs $\mathit{GG}_{\mathcal{M}\sigma}$ with skeleton $\sigma\in\Sigma_\mathcal{S}$, where $h_a = \max(h_r + h_m, h_r+2h_m-2)$ and $h_m$ is the maximum number of hops for a dependency in $\mathcal{M}$.
\end{mytheorem}

\begin{proof}
Let $\mathcal{M}=(\mathcal{S}, \mathcal{D})$ be an arbitrary acyclic relational model structure, let $B\in\mathcal{E}\cup\mathcal{R}$ be an arbitrary perspective, and let $h_r\in\mathbb{N}^0$ be an arbitrary hop threshold.

To prove that the abstract ground graph $AGG_{\mathcal{M}Bh_a}$ is complete up to hop threshold $h_r$, we show that for every $(B,h_r)$-reachable edge $i_k.X\rightarrow i_j.Y$ in every ground graph $GG_{\mathcal{M}\sigma}$ with $\sigma\in\Sigma_{\mathcal{S}}$, there is a set of corresponding edges in $AGG_{\mathcal{M}Bh_a}$.
Specifically, the $(B,h_r)$-reachable edge $i_k.X\rightarrow i_j.Y$ yields two sets of relational variables for some $b\in\sigma(B)$, namely $\mathbf{P_k.X} = \{ P_k.X\ \vert\ i_k\in P_k\vert_b\ \land\ \mathit{length}(P_k)\leq h_r+1 \}$ and $\mathbf{P_j.Y} = \{ P_j.Y\ \vert\ i_j\in P_j\vert_b\ \land\ \mathit{length}(P_j)\leq h_r+1 \}$ by Definition~\ref{def:reachability}.
Note that all relational variables in both $\mathbf{P_k.X}$ and $\mathbf{P_j.Y}$ are nodes in $AGG_{\mathcal{M}Bh_a}$.
We show that for all $P_k.X\!\in\!\mathbf{P_k.X}$ and for all $P_j.Y\!\in\!\mathbf{P_j.Y}$ either 
(a) $P_k.X\rightarrow P_j.Y\in \mathit{AGG}_{\mathcal{M}Bh_a}$, 
(b) $P_k.X\cap P_k'.X\rightarrow P_j.Y\in \mathit{AGG}_{\mathcal{M}Bh_a}$
or
$P_k.X\cap P_k'.X\rightarrow P_j'.Y\in \mathit{AGG}_{\mathcal{M}Bh_a}$,
where $i_k\in P_k'\vert_b$ and $i_j\in P_j'\vert_b$,
or
(c) $P_k.X\rightarrow P_j.Y\cap P_j'.Y\in \mathit{AGG}_{\mathcal{M}Bh_a}$
or
$P_k'.X\rightarrow P_j.Y\cap P_j'.Y\in \mathit{AGG}_{\mathcal{M}Bh_a}$,
where $i_k\in P_k'\vert_b$ and $i_j\in P_j'\vert_b$.

Let $\sigma\in\Sigma_\mathcal{S}$ be an arbitrary skeleton, let $i_k.X\rightarrow i_j.Y\in \mathit{GG}_{\mathcal{M}\sigma}$ be an arbitrary $(B,h_r)$-reachable edge drawn from $[I_j,\dots, I_k].X\rightarrow [I_j].Y\in\mathcal{D}$ where $\mathit{length}([I_j,\dots,I_k])\le h_m+1$, and let $P_k.X\in\mathbf{P_k.X}, P_j.Y\in\mathbf{P_j.Y}$ be an arbitrary pair of relational variables.
There are three cases:

(a) $P_k\in \mathit{extend}(P_j, [I_j,\dots, I_k])$.
Then, $\mathit{length}(P_k)\le (h_r+1)+(h_m+1)-1 = h_r + h_m + 1 \le h_a +1$. 
Therefore, $P_k.X$ is a node in the abstract ground graph, and $P_k.X\rightarrow P_j.Y\in AGG_{\mathcal{M}Bh_a}$ by Definition~\ref{def:abstract-gg}.

(b) $P_k\notin \mathit{extend}(P_j, [I_j,\dots, I_k])$, but $\exists P_k'\in \mathit{extend}(P_j, [I_j,\dots, I_k])$ such that $i_k\in P_k'\vert_b$.
Then, $\mathit{length}(P_k')\le (h_r+1) + (h_m+1) - 1 = h_r + h_m + 1 \le h_a + 1$.
Therefore, $P_k'$ is a node in the abstract ground graph, $P_k'.X\rightarrow P_j.Y\in AGG_{\mathcal{M}Bh_a}$, and $P_k.X\cap P_k'.X\rightarrow P_j.Y\in AGG_{\mathcal{M}Bh_a}$ by Definition~\ref{def:abstract-gg}.

(c) For all $P_k\in \mathit{extend}(P_j, [I_j,\dots, I_k])$, it is the case that $i_k\notin P_k.X\vert_b$.
Then by Lemma~\ref{lemma:extend-reverse}, there exists a $P_j'$ such that $i_j\in P_j'\vert_b$ and there exists a $P_k''\in \mathit{extend}(P_j', [I_j,\dots, I_k])$.
Given the way $P_j'$ is constructed, its length is bounded by: 
\[\mathit{length}(P'_j) \le \mathit{length}(P_j) + \mathit{length}([I_j, \ldots, I_k]) - 3\le (h_r+1)+(h_m+1)-3 = h_r+h_m-1\]
$P_k''$ intersects with $P_k$ since they both reach $i_k$, and the length of $P_k''$ is bounded by:
\[\mathit{length}(P_k'') \le \mathit{length}(P_j') + \mathit{length}([I_j, \ldots, I_k]) - 1 \le (h_r+h_m-1)+(h_m+1)-1 = h_r+2h_m-1 \]
Also by Lemma~\ref{lemma:extend-reverse}, we know that $P_j$ and $P_j'$ intersect.  
Since $\mathit{length}(P_k'')\le h_r+2h_m-1 \le h_a+1$, $P_k''$ is a node in the abstract ground graph, $P_k''.X\rightarrow P_j'.Y\in AGG_{\mathcal{M}Bh_a}$ $P_k''.X\rightarrow P_j'.Y\cap P_j.Y\in AGG_{\mathcal{M}Bh_a}$, and $P_k.X\cap P_k''.X\rightarrow P_j'.Y\in AGG_{\mathcal{M}Bh_a}$ by Definition~\ref{def:abstract-gg}. 

From the above three cases, it follows that to guarantee completeness up to $h_r$, the abstract ground graph must contain nodes up to the hop threshold $h_a = \max(h_r+h_m, h_r+2h_m-2)$.
$\blacksquare$
\end{proof}

Theorems~\ref{thm:agg-soundness-bh} and \ref{thm:agg-completeness-bh} guarantee that if an abstract ground graph is constructed with a hop threshold of $h_a$ from perspective $B$, it captures all paths of dependence in all ground graphs, where (1) the variables along those paths are reachable in $h_r$ hops from instances of $B$ and (2) the underlying dependencies are bounded by a threshold of $h_m$.

In the following, we say that \textit{d}-separation holds up to a specified hop threshold $h$ if there are no \textit{d}-connecting paths involving relational variables of length greater than $h+1$.

\begin{mytheorem}
\label{thm:relational-d-sep-bh}
Relational \textit{d}-separation is sound and complete for abstract ground graphs up to a specified hop threshold.
Let $\mathcal{M}$ be an acyclic relational model structure, and let $h_m$ be the maximum number of hops for a dependency in $\mathcal{M}$. 
Let $\mathbf{X}$, $\mathbf{Y}$, and $\mathbf{Z}$ be three distinct sets of relational variables for perspective $B\in\mathcal{E}\cup\mathcal{R}$ defined over relational schema $\mathcal{S}$, and let $h_r$ be the maximum number of hops of relational variables in $\mathbf{X}, \mathbf{Y}$, and $\mathbf{Z}$.
Then, $\mathbf{\bar{X}}$ and $\mathbf{\bar{Y}}$ are \textit{d}-separated by $\mathbf{\bar{Z}}$ on the abstract ground graph $\mathit{AGG}_{\mathcal{M}Bh_a}$ if and only if for all skeletons $\sigma\in\Sigma_\mathcal{S}$ and for all $b\in \sigma(B)$, $\mathbf{X}\vert_b$ and $\mathbf{Y}\vert_b$ are \textit{d}-separated by $\mathbf{Z}\vert_b$ up to hop threshold $h_r$ in ground graph $\mathit{GG}_{\mathcal{M}\sigma}$, where $h_a = \max(h_r + h_m, h_r+2h_m-2)$.
\end{mytheorem}

\begin{proof}
We must show that \textit{d}-separation on an abstract ground graph implies \textit{d}-separation on all ground graphs it represents (soundness) and that \textit{d}-separation facts that hold across all ground graphs are also entailed by \textit{d}-separation on the abstract ground graph (completeness).

\textbf{Soundness}: Assume that $\mathbf{\bar{X}}$ and $\mathbf{\bar{Y}}$ are \textit{d}-separated by $\mathbf{\bar{Z}}$ on $\mathit{AGG}_{\mathcal{M}Bh_a}$.
Assume for contradiction that there exists a skeleton $\sigma\in\Sigma_\mathcal{S}$ and an item instance $b\in\sigma(B)$ such that $\mathbf{X}\vert_b$ and $\mathbf{Y}\vert_b$ are \textit{not d}-separated by $\mathbf{Z}\vert_b$ in the ground graph $\mathit{GG}_{\mathcal{M}\sigma}$.
Then, there must exist a \textit{d}-connecting path $p$ from some $x\in\mathbf{X}\vert_b$ to some $y\in\mathbf{Y}\vert_b$ given all $z\in\mathbf{Z}\vert_b$ such that every edge of $p$ is $(B, h_r)$-reachable.
By Theorem~\ref{thm:agg-completeness-bh}, $\mathit{AGG}_{\mathcal{M}Bh_a}$ is $(B,h_r)$-reachably complete, so all $(B,h_r)$-reachable edges in $\mathit{GG}_{\mathcal{M}\sigma}$ are captured by edges in $\mathit{AGG}_{\mathcal{M}Bh_a}$.
Thus, path $p$ must be represented from some node in $\{N_x\ \vert\ x\in N_x\vert_b\}$ to some node in $\{N_y\ \vert\ y\in N_y\vert_b\}$, where $N_x, N_y$ are nodes in $\mathit{AGG}_{\mathcal{M}Bh_a}$.
If $p$ is \textit{d}-connecting in $\mathit{GG}_{\mathcal{M}\sigma}$, then it is \textit{d}-connecting in $\mathit{AGG}_{\mathcal{M}Bh_a}$, which implies that $\mathbf{\bar{X}}$ and $\mathbf{\bar{Y}}$ are \textit{not d}-separated by $\mathbf{\bar{Z}}$.
Therefore, $\mathbf{X}\vert_b$ and $\mathbf{Y}\vert_b$ must be \textit{d}-separated by $\mathbf{Z}\vert_b$.

\textbf{Completeness}: Assume that $\mathbf{X}\vert_b$ and $\mathbf{Y}\vert_b$ are \textit{d}-separated by $\mathbf{Z}\vert_b$ in the ground graph $\mathit{GG}_{\mathcal{M}\sigma}$ for all skeletons $\sigma\in\Sigma_\mathcal{S}$ and for all $b\in\sigma(B)$.
Assume for contradiction that $\mathbf{\bar{X}}$ and $\mathbf{\bar{Y}}$ are \textit{not d}-separated by $\mathbf{\bar{Z}}$ on $\mathit{AGG}_{\mathcal{M}Bh_a}$. %for some hop threshold $h\in\mathbb{N}^0$.
Then, there must exist a \textit{d}-connecting path $p$ for some relational variable $X\in\mathbf{\bar{X}}$ to some $Y\in\mathbf{\bar{Y}}$ given all $Z\in\mathbf{\bar{Z}}$.
By Theorem~\ref{thm:agg-soundness-bh}, $\mathit{AGG}_{\mathcal{M}Bh_a}$ is $(B,h_a)$-reachably sound, so every edge in $\mathit{AGG}_{\mathcal{M}Bh}$ must correspond to some pair of variables in some ground graph.
Thus, if $p$ is \textit{d}-connecting in $\mathit{AGG}_{\mathcal{M}Bh_a}$, then there must exist some skeleton $\sigma$ such that $p$ is \textit{d}-connecting in $\mathit{GG}_{\mathcal{M}\sigma}$ for some $b\in\sigma(B)$, which implies that \textit{d}-separation does not hold for that ground graph.
Therefore, $\mathbf{\bar{X}}$ and $\mathbf{\bar{Y}}$ must be \textit{d}-separated by $\mathbf{\bar{Z}}$ on $\mathit{AGG}_{\mathcal{M}Bh_a}$.
$\blacksquare$
\end{proof}

%Theorem~\ref{thm:relational-d-sep-bh} proves that \textit{d}-separation on abstract ground graphs is a sound and complete solution to identifying independence in relational models.
%Given Theorem~\ref{thm:d-sep}, we can also say that the set of conditional independence facts derived from \textit{d}-separation on abstract ground graphs is exactly the same (up to a specified hop threshold) as the set of conditional independencies in common with all faithful distributions represented by all possible ground graphs.

\section{Experimental Details---Equivalence of a Na\"{i}ve Approach}
\label{sec:appendix_condindequiv_prediction}
In this appendix, we provide additional details for the experiment in Section~\ref{sec:naive_rds}.
The main goal of this experiment is to quantify how often traditional \textit{d}-separation applied directly to relational model structures produces incorrect conditional independence facts.
This provides a rough measurement for the additional representational power of relational \textit{d}-separation on abstract ground graphs.
Here, we present an analysis of which factors influence the number of equivalent and non-equivalent conditional independence judgments between both approaches (na\"{i}vely applying traditional \textit{d}-separation versus relational \textit{d}-separation).

Specifically, we show here the results of running log-linear regression to predict the number of equivalent and non-equivalent judgments for varying schemas and models.  We first applied lasso for feature selection \citep{tibshirani1996lasso} to minimize the number of predictors while maximizing model fit.  We also standardized the input variables by dividing by two standard deviations, as recommended by \citet{gelman2008scaling}.  Since the predictor for the number of dependencies is log-transformed, the standardization for that variable occurs after taking the logarithm.

In predicting the (log of the) number of equivalent conditional independencies, the following variables were significantly and substantively predictive (in order of decreasing predictive power):

\begin{table}
\centering
\begin{tabular}{r|c|c|c}
 Predictor & Coefficient & Partial & Semipartial \\ 
\hline
log(\# dependencies) $\times$ \# entities & 1.38 & 0.232 & 0.085 \\
log(\# dependencies) & 1.14 & 0.135 & 0.044 \\
log(\# dependencies) $\times$ \# \textsc{many} cardinalities & -0.71 & 0.092 & 0.028 \\
\# entities $\times$ \# relational variables & -0.32 & 0.044 & 0.013 \\
\end{tabular}
\caption{Number of equivalent conditional independence judgments: estimated standardized coefficient, squared partial correlation coefficient, and squared semipartial correlation coefficient for each predictor.}
\label{tab:cond-ind-equiv-model-stats}
\end{table}

\begin{itemize}
\item Interaction between the log of the number of dependencies and the number of entities (positive)
\item Log of the number of dependencies (positive)
\item Interaction between the log of the number of dependencies and the number of \textsc{many} cardinalities (negative)
\item Number of entities (negative)
\item Interaction between the number of entities and the number of relational variables in the AGG (negative)
\end{itemize}

The fit for the equivalent model has an $R^2 = 0.721$ for $n=4,000$, and Table~\ref{tab:cond-ind-equiv-model-stats} contains the standardized coefficients as well as the squared partial and semipartial correlation coefficients for each predictor.  
For lasso, $\lambda = 0.0076$ offered the fewest predictors while increasing the model fit by at least 0.01.

In predicting the (log of the) number of non-equivalent conditional independencies, the following variables were significantly and substantively predictive (in order of decreasing predictive power):

\begin{table}
\centering
\begin{tabular}{r|c|c|c}
 Predictor & Coefficient & Partial & Semipartial \\ 
\hline
\# \textsc{many} cardinalities $\times$ \# entities & -2.22 & 0.207 & 0.064 \\ 
log(\# dependencies) $\times$ \# entities & 0.90 & 0.165 & 0.048 \\
\# \textsc{many} cardinalities & 3.24 & 0.128 & 0.036 \\
log(\# dependencies) $\times$ \# \textsc{many} cardinalities & 1.47 & 0.127 & 0.036
\end{tabular}
\caption{Number of non-equivalent conditional independence judgments: estimated standardized coefficient, squared partial correlation coefficient, and squared semipartial correlation coefficient for each predictor.}
\label{tab:cond-ind-nonequiv-model-stats}
\end{table}

\begin{itemize}
\item Interaction between the number of \textsc{many} cardinalities and the number of entities (negative)
\item Interaction between the log of the number of dependencies and the number of entities (positive)
\item Number of \textsc{many} cardinalities (positive)
\item Interaction between the log of the number of dependencies and the number of \textsc{many} cardinalities (positive)
\end{itemize}

The fit for the non-equivalent model has an $R^2 = 0.755$ for $n=4,000$, and Table~\ref{tab:cond-ind-nonequiv-model-stats} contains the standardized coefficients and the squared partial and semipartial correlation coefficients for each predictor.
For lasso, $\lambda = 0.0155$ offered the fewest predictors while increasing the model fit by at least 0.01.

\section{Experimental Details---Abstract Ground Graph Size}
\label{sec:appendix_aggsize_prediction}
In this appendix, we provide additional details for the experiment in Section~\ref{sec:agg-size}.
The goal of this experiment is to determine which factors influence the size of abstract ground graphs because the computational complexity of relational \textit{d}-separation depends on their size.
Specifically, we show here the results of running log-linear regression to predict the size of abstract ground graphs for varying schemas and models.  We first applied lasso for feature selection \citep{tibshirani1996lasso} to minimize the number of predictors while maximizing model fit.  We also standardized the input variables by dividing by two standard deviations, as recommended by \citet{gelman2008scaling}.  Since the predictor for the number of dependencies is log-transformed, the standardization for that variable occurs after taking the logarithm.

In predicting the (log of the) number of nodes, the following variables were significantly and substantively predictive (in order of decreasing predictive power):

\begin{table}
\centering
\begin{tabular}{r|c|c|c}
 Predictor & Coefficient & Partial & Semipartial \\ 
\hline
\# relationships & 3.24 & 0.452 & 0.150 \\
\# \textsc{many} cardinalities $\times$ isEntity=F & 3.09 & 0.349 & 0.109 \\
\# entities & -2.11 & 0.359 & 0.102 \\
\# \textsc{many} cardinalities $\times$ isEntity=T & 2.51 & 0.216 & 0.053 \\
\# \textsc{many} cardinalities $\times$ \# relationships & -0.88 & 0.100 & 0.020 \\
\# attributes & 0.23 & 0.024 & 0.004
\end{tabular}
\caption{Number of nodes in an abstract ground graph: estimated standardized coefficient, squared partial correlation coefficient, and squared semipartial correlation coefficient for each predictor.}
\label{tab:agg-node-model-stats}
\end{table}

\begin{itemize}
\item Number of relationships (positive)
\item Interaction between \textsc{many} cardinalities and an indicator variable for whether the abstract ground graph is from an entity or relationship perspective (positive)
\item Number of entities (negative)
\item Interaction between the number of \textsc{many} cardinalities and relationships (negative)
\item Total number of attributes (positive)
\end{itemize}

The fit for the nodes model has an $R^2 = 0.818$ for $n=450,000$, and Table~\ref{tab:agg-node-model-stats} contains the standardized coefficients as well as the squared partial and semipartial correlation coefficients for each predictor.  
For lasso, $\lambda = 0.0095$ offered the fewest predictors while increasing the model fit by at least 0.01.

In predicting the (log of the) number of edges, the following variables were significantly and substantively predictive (in order of decreasing predictive power):

\begin{table}
\centering
\begin{tabular}{r|c|c|c}
 Predictor & Coefficient & Partial & Semipartial \\ 
\hline
log(\# dependencies) & 1.44 & 0.440 & 0.165 \\
\# relationships & 3.86 & 0.395 & 0.138 \\
\# \textsc{many} cardinalities $\times$ isEntity=F & 4.27 & 0.356 & 0.123 \\
\# entities & -2.78 & 0.353 & 0.115 \\
\# \textsc{many} cardinalities $\times$ isEntity=T & 3.52 & 0.231 & 0.067 \\ 
\# \textsc{many} cardinalities $\times$ \# relationships & -1.35 & 0.127 & 0.031
\end{tabular}
\caption{Number of edges in an abstract ground graph: estimated standardized coefficient, squared partial correlation coefficient, and squared semipartial correlation coefficient for each predictor.}
\label{tab:agg-edge-model-stats}
\end{table}

\begin{itemize}
\item Log of the number of dependencies (positive)
\item Number of relationships (positive)
\item Interaction between \textsc{many} cardinalities and an indicator variable for whether the abstract ground graph is from an entity or relationship perspective (positive)
\item Number of entities (negative)
\item Interaction between the number of \textsc{many} cardinalities and relationships (negative)
\end{itemize}

The fit for the edges model has an $R^2 = 0.789$ for $n=450,000$, and Table~\ref{tab:agg-edge-model-stats} contains the standardized coefficients and the squared partial and semipartial correlation coefficients for each predictor.
For lasso, $\lambda = 0.0164$ offered the fewest predictors while increasing the model fit by at least 0.01.

\bibliography{jmlr2012}

\end{document}